\documentclass{article} 
\usepackage[final]{neurips_2024}

\usepackage{amsmath,amsfonts,bm}

\def\eqref#1{equation~\ref{#1}}

\def\1{\bm{1}}

\DeclareMathAlphabet{\mathsfit}{\encodingdefault}{\sfdefault}{m}{sl}
\SetMathAlphabet{\mathsfit}{bold}{\encodingdefault}{\sfdefault}{bx}{n}

\newcommand{\R}{\mathbb{R}}

\DeclareMathOperator*{\argmax}{arg\,max}
\DeclareMathOperator*{\argmin}{arg\,min}

\usepackage{hyperref}
\usepackage{url}
\usepackage{amsmath,amsthm,amssymb}
\usepackage{color}
\usepackage{cleveref}
\usepackage[noend]{algpseudocode}
\usepackage{algorithm}
\usepackage{xspace}
\usepackage{graphicx}
\usepackage{tabularx}
\usepackage{booktabs}
\usepackage{wrapfig}
\usepackage{subcaption}
\usepackage{multirow}
\usepackage{natbib}
\usepackage{enumitem}
\usepackage{comment}

\theoremstyle{plain}
\newtheorem{lemma}{Lemma}
\newtheorem{proposition}{Proposition}
\newtheorem{theorem}{Theorem}

\theoremstyle{definition}
\newtheorem{definition}{Definition}[section] 

\renewcommand{\eqref}[1]{\text{(\ref{#1})}}
\newcommand{\loose}{\looseness=-1}
\newcommand{\edit}[1]{#1}
\newtoggle{arxiv}
\toggletrue{arxiv}

\newcommand{\dppo}{\textsc{Dppo}\xspace}
\newcommand{\dsrl}{\textsc{Dsrl}\xspace}

\newcommand{\nbc}{$\sigma$-\textsc{Bc}\xspace}
\newcommand{\pbc}{\textsc{PostBc}\xspace}
\newcommand{\bc}{\textsc{Bc}\xspace}
\newcommand{\iql}{\textsc{Iql}\xspace}
\newcommand{\dice}{\textsc{DICE}\xspace}

\title{Posterior Behavioral Cloning: Pretraining BC Policies for Efficient RL Finetuning}

\author{Andrew Wagenmaker\thanks{Correspondance to: \texttt{ajwagen@berkeley.edu}.}\\
UC Berkeley
\And
Perry Dong \\
Stanford
\And
Raymond Tsao \\
UC Berkeley
\And
Chelsea Finn \\
Stanford
\And
Sergey Levine \\
UC Berkeley
}

\begin{document}

\maketitle

\begin{abstract}
Standard practice across domains from robotics to language is to first pretrain a policy on a large-scale demonstration dataset, and then finetune this policy, typically with reinforcement learning (RL), in order to improve performance on deployment domains. This finetuning step has proved critical in achieving human or super-human performance, yet while much attention has been given to developing more effective finetuning algorithms, little attention has been given to ensuring the pretrained policy is an effective initialization for RL finetuning. In this work we seek to understand how the pretrained policy affects finetuning performance, and how to pretrain policies in order to ensure they are effective initializations for finetuning. We first show theoretically that standard behavioral cloning (\bc)---which trains a policy to directly match the actions played by the demonstrator---can fail to ensure coverage over the demonstrator's actions, a minimal condition necessary for effective RL finetuning.
We then show that if, instead of exactly fitting the observed demonstrations, we train a policy to model the \emph{posterior} distribution of the demonstrator's behavior given the demonstration dataset, we \emph{do} obtain a policy that ensures coverage over the demonstrator's actions, enabling more effective finetuning. Furthermore, this policy---which we refer to as the \emph{posterior behavioral cloning} (\pbc) policy---achieves this while ensuring pretrained performance is no worse than that of the \bc policy.
We then show that \pbc is practically implementable with modern generative models in robotic control domains---relying only on standard supervised learning---and leads to significantly improved RL finetuning performance on both realistic robotic control benchmarks and real-world robotic manipulation tasks, as compared to standard behavioral cloning.\loose
\end{abstract}


\vspace{-0.5em}
\section{Introduction}
\vspace{-0.5em}

Across domains---from language, to vision, to robotics---a common paradigm has emerged for training highly effective ``policies'': collect a large set of demonstrations, ``pretrain'' a policy via behavioral cloning (\bc) to mimic these demonstrations, then ``finetune'' the pretrained policy on a deployment domain of interest. While pretraining can endow the policy with generally useful abilities, the finetuning step has proved critical in obtaining effective performance, enabling human value alignment and reasoning capabilities in language domains \citep{ouyang2022training,bai2022training,team2025kimi,guo2025deepseek}, and improving task solving precision and generalization to unseen tasks in robotic domains \citep{nakamoto2024steering,chen2025conrft,kim2025fine,wagenmaker2025steering}. In particular, reinforcement learning (RL)-based finetuning---where the pretrained policy is deployed in a setting of interest and its behavior updated based on the outcomes of these online rollouts---is especially crucial in improving the performance of a pretrained policy.

Critical to achieving successful RL-based finetuning performance in many domains---particularly in settings when policy deployment is costly and time-consuming, such as robotic control---is sample efficiency; effectively modifying the behavior of the pretrained model using as few deployment rollouts as possible.
While significant attention has been given to developing more efficient finetuning algorithms, this ignores a primary ingredient in the RL finetuning process: the pretrained policy itself. Though generally more effective pretrained policies are the preferred initialization for finetuning \citep{guo2025deepseek,yue2025does},
it is not well understood how pretraining impacts finetuning performance beyond this, and how we might pretrain policies to enable more efficient RL finetuning.\loose

\begin{figure}[t]
  \centering
  \includegraphics[width=.85\textwidth]{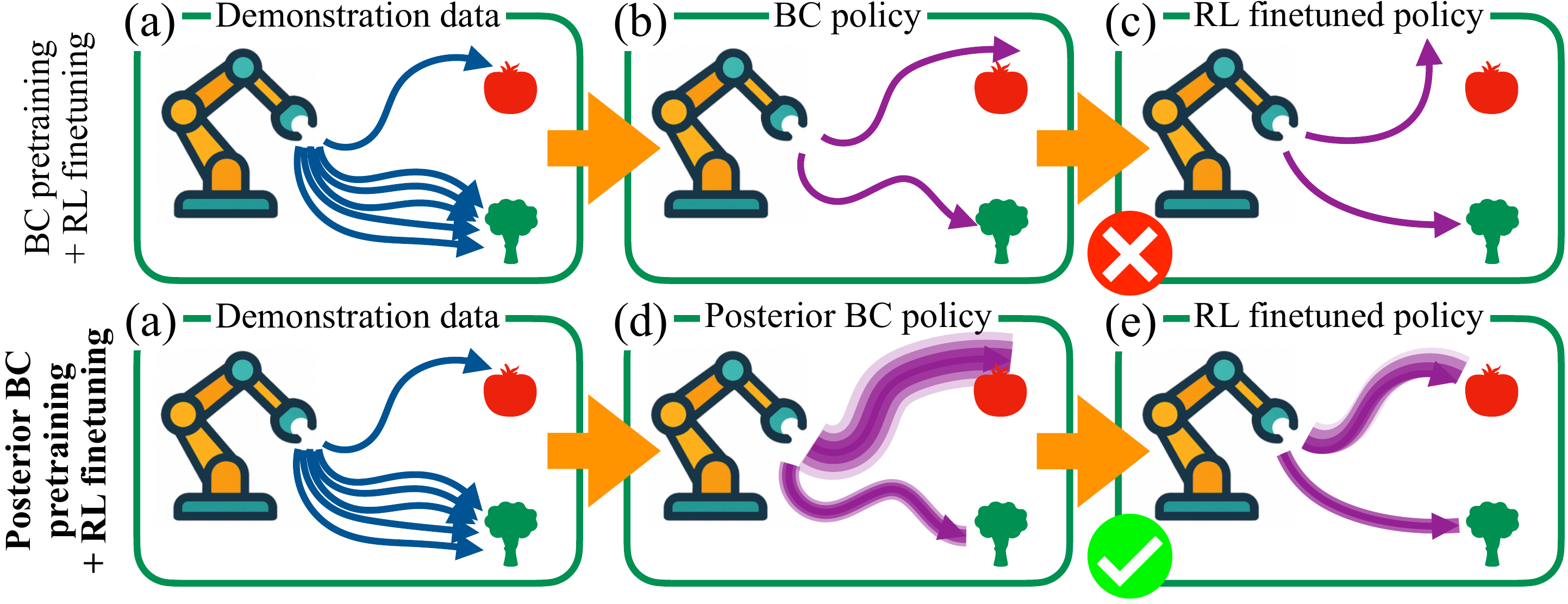}
  \vspace{-0.25em}
  \caption{(a) We consider the setting where we are given demonstration data for some tasks of interest. (b) Standard \bc pretraining fits the behaviors in the demonstrations, leading to effective performance in regions with high demonstration data density, yet can overcommit to the observed behaviors in regions with low data density. (c) This leads to ineffective RL finetuning, since rollouts from the \bc policy provide little meaningful reward signal in such low data density regions, which is typically necessary to enable effective improvement. (d) In contrast, we propose \emph{posterior behavioral cloning} (\pbc), which instead of directly mimicking the demonstrations, trains a generative policy to fit the \emph{posterior distribution} of the demonstrator's behavior. This endows the pretrained policy with a wider distribution of actions in regions of low demonstrator data density, while in regions of high data density it reduces to approximately the standard \bc policy. (e) This wider action distribution in low data density regions allows for collection of diverse observations with more informative reward signal, enabling more effective RL finetuning, while in regions of high data density performance converges to that of the demonstrator.\loose}
  \label{fig:paper_fig}
  \vspace{-1.5em}
\end{figure}

In this work we seek to understand the role of the pretrained policy in RL finetuning, and how we might pretrain policies that (a) enable efficient RL finetuning, and (b) before finetuning, perform no worse than the policy pretrained with standard \bc.
We propose a novel pretraining approach---\emph{posterior behavioral cloning} (\pbc)---which, rather than fitting the empirical distribution of demonstrations as standard \bc does, instead fits the \emph{posterior} distribution over the demonstrator's behavior. 
That is, assuming a uniform prior over the demonstrator's behavior and viewing the demonstration data as samples from the demonstrator's behavioral distribution, we seek to train a policy that models the posterior distribution of the demonstrator's behavior given these observations.
This enables the pretrained policy to take into account its potential uncertainty about the demonstrator's behavior, and adjust the entropy of its action distribution based on this uncertainty. In states where it is uncertain about the demonstrator's actions, \pbc samples from a high-entropy distribution, allowing for a more diverse set of actions that may enable further policy improvement, while in states where it is certain about the demonstrator's actions, it samples from a low-entropy distribution, simply mimicking what it knows to be the (correct) demonstrator behavior (see \Cref{fig:paper_fig}).\loose

Theoretically, we show that \pbc leads to provable improvements over standard \bc in terms of the potential for downstream RL performance. In particular, we focus on the ability of the pretrained policy to cover the demonstrator policy's actions---whether it samples all actions the demonstrator policy might sample---which, for finetuning approaches that rely on rolling out the pretrained policy, is a prerequisite to ensure finetuning can even match the performance of the demonstrator.
We show that standard \bc can provably fail to cover the demonstrator's distribution, while \pbc \emph{does} cover the demonstrator's distribution, incurs no suboptimality in the performance of the pretrained policy as compared to the standard \bc policy, and achieves a near-optimal sampling cost out of all policy estimators which have pretrained performance no worse than the \bc policy's. 

Inspired by this, we develop a practical approach to approximating the posterior of the demonstrator in continuous action domains, and instantiate \pbc with modern generative models---diffusion models---on robotic control tasks. Our instantiation relies only on pretraining with scalable supervised learning objectives---no RL is required in pretraining---and can be incorporated into existing \bc training pipelines with minimal modification.
We demonstrate experimentally that \pbc pretraining can lead to significant performance gains in terms of the efficiency and effectiveness of RL finetuning, as compared to running RL finetuning on a policy pretrained with standard \bc, and achieves these gains without decreasing the performance of the pretrained policy itself. 
We show that this holds for a variety of finetuning algorithms---both policy-gradient-style algorithms, and algorithms which explicitly refine or filter the distribution of the pretrained policy---enabling effective RL finetuning across a variety of challenging robotic tasks in both simulation and the real world.

\iftoggle{arxiv}{}{\vspace{-0.5em}}
\vspace{-0.5em}
\section{Related Work}
\vspace{-0.5em}
\iftoggle{arxiv}{}{\vspace{-0.5em}}

\textbf{\bc pretraining.}
\bc training of expressive generative models---where the model is trained to predict the next ``action'' of the demonstrator---forms the backbone of pretraining for LLMs 
\citep{radford2018improving} and robotic control \citep{bojarski2016end,zhang2018deep,rahmatizadeh2018vision,stepputtis2020language,shafiullah2022behavior,gu2023rt,team2024octo,zhao2024aloha,black2024pi_0,kim2024openvla}. Experimentally, we focus in particular on policies parameterized as diffusion models \citep{sohl2015deep,ho2020denoising,song2020denoising}, which have seen much attention in the robotics community \citep{chi2023diffusion,ankile2024juicer,zhao2024aloha,ze20243d,sridhar2024nomad,dasari2024ingredients,team2024octo,black2024pi_0,bjorck2025gr00t}, yet our approach can extend to other generative model classes as well. These works, however, simply pretrain with standard \bc, and do not consider how the pretraining may affect RL finetuning performance.
Please see \Cref{sec:additional_related} for a discussion of other approaches to pretraining (imitation learning, meta-learning, and reinforcement learning).\loose

\textbf{Pretraining for downstream finetuning.}
Several recent works in the language domain aim to understand the relationship between pretraining and downstream finetuning \citep{springer2025overtrained,zeng2025can,chen2025rethinking,jin2025rl,chen2025coverage}.
A common thread through these works is that cross entropy loss is not predictive of downstream finetuning performance, and, in fact, low cross entropy loss can be anti-correlated with finetuning performance as the model can become \emph{overconfident}. 
Most related to our work is the concurrent work of \cite{chen2025coverage}, which consider a notion of  \emph{coverage} closely related to our notion of \emph{demonstrator action coverage}, and show that coverage generalizes faster than cross-entropy, and is an effective predictor of the downstream success of Best-of-$N$ sampling.
While both our work and \cite{chen2025coverage} focus on notions of coverage to enable downstream improvement, we see this work as complementary.
While \cite{chen2025coverage} show coverage generalizes faster than cross-entropy, our results show \bc pretraining can \emph{still} fail to ensure meaningful coverage, especially in the small sample regime. 
Furthermore, \cite{chen2025coverage} does not consider the tradeoff between policy performance and coverage that is a primary focus of our work, and their proposed pretraining intervention---gradient normalization---would not resolve the shortcomings of \bc in our setting, while our proposed intervention, \pbc, does.
Finally, \cite{chen2025coverage} is primarily a theoretical work and focuses on discrete next-token prediction (indeed, all works cited above consider only discrete next-token prediction); in contrast, a primary focus of our work is on continuous control, and we demonstrate our approach scales to real-world robotic settings.\loose

\textbf{RL finetuning of pretrained policies.}
RL finetuning of pretrained policies is a critical step in both language and robotic domains. In language domains, RL finetuning has proved crucial in aligning LLMs to human values \citep{ziegler2019fine,ouyang2022training,bai2022training,ramamurthy2022reinforcement,touvron2023llama}, and enabling reasoning abilities \citep{shao2024deepseekmath,team2025kimi,guo2025deepseek}. A host of finetuning algorithms have been developed, both online 
\citep{bai2022constitutional,bakker2022fine,dumoulin2023density,lee2023rlaif,munos2023nash,swamy2024minimaximalist,chakraborty2024maxmin,chang2024dataset} and offline 
\citep{rafailov2023direct,azar2024general,rosset2024direct,tang2024generalized,yin2024relative}. In robotic control domains, RL finetuning methods include directly modifying the weights of the base pretrained policy \citep{zhang2024grape,xu2024rldg,mark2024policy,ren2024diffusion,hu2025flare,guo2025improving, lu2025vla,chen2025conrft, liu2025can}, Best-of-$N$ sampling-style approaches that filter the output of the pretrained policy with a learned value function \citep{chen2022offline,hansen2023idql,he2024aligniql,nakamoto2024steering,dong2025matters}, ``steering'' the pretrained policy by altering its sampling process \citep{wagenmaker2025steering}, and learning smaller residual policies to augment the pretrained policy's actions \citep{ankile2024imitation, yuan2024policy, julg2025refined, dong2025expo}. Our work is tangential to this line of work: rather than improving the finetuning algorithm, we aim to ensure the pretrained policy is amenable to RL finetuning.

\textbf{Posterior sampling and exploration.}
Our proposed approach relies on modeling the posterior distribution of the demonstrator's behavior. While this is, to the best of our knowledge, the first example of applying posterior sampling to \bc, posterior methods have a long history in RL, going back to the work of \cite{thompson1933likelihood}. This works spans applied \citep{osband2016deep,osband2016generalization,osband2018randomized,zintgraf2019varibad} and theoretical \citep{agrawal2012analysis,russo2014learning,russo2018tutorial,janz2024exploration,kveton2020randomized,russo2019worst} settings. More generally, our approach can be seen as enabling \bc-trained policies to \emph{explore} more effectively. Exploration is a well-studied problem in the RL community \citep{stadie2015incentivizing,bellemare2016unifying,burda2018exploration,choi2018contingency,ecoffet2019go,shyam2019model,lee2021sunrise,henaff2022exploration}, with several works considering learning exploration strategies from offline data \citep{hu2023unsupervised,li2023accelerating,wilcoxson2024leveraging,wagenmakerbehavioral}. These works, however, either consider RL-based pretraining (while we focus on \bc) or do not consider the question of online finetuning.\loose


\newcommand{\Dtv}{D_{\mathrm{TV}}}
\newcommand{\pihat}{\widehat{\pi}}
\newcommand{\cS}{\mathcal{S}}
\newcommand{\cA}{\mathcal{A}}
\newcommand{\Qhat}{\widehat{Q}}
\newcommand{\Qbeta}{Q^\beta}
\newcommand{\regbeta}{\mathrm{Regret}^\beta}
\newcommand{\Exp}{\mathbb{E}}
\newcommand{\cF}{\mathcal{F}}
\newcommand{\Pihat}{\widehat{\Pi}}

\newcommand{\pihatbeta}{\widehat{\pi}^{\mathrm{bc}}}
\newcommand{\pibeta}{\pi^\beta}
\newcommand{\cJ}{\mathcal{J}}
\newcommand{\pipost}{\widehat{\pi}^{\mathrm{post}}}
\newcommand{\unif}{\mathrm{unif}}
\newcommand{\bbI}{\mathbb{I}}
\newcommand{\var}{\mathrm{Var}}
\newcommand{\cE}{\mathcal{E}}
\newcommand{\cM}{\mathcal{M}}
\newcommand{\piunif}{\widehat{\pi}^{\mathrm{u,\alpha}}}
\newcommand{\pitilunif}{\widetilde{\pi}^{\mathrm{u,\alpha}}}
\newcommand{\betaprior}{P^\beta_{\mathrm{prior}}}
\newcommand{\betapost}{P^\beta_{\mathrm{post}}}
\newcommand{\simplex}{\triangle}
\newcommand{\frakD}{\mathfrak{D}}
\newcommand{\pipostbeta}{\widehat{\pi}^{\beta,\mathrm{post}}}
\newcommand{\cO}{\mathcal{O}}
\newcommand{\cN}{\mathcal{N}}
\newcommand{\muhat}{\widehat{\mu}}
\newcommand{\mubar}{\bar{\mu}}

\newcommand{\Sigprior}{\Lambda_0}
\newcommand{\Qprior}{Q_{\mathrm{prior}}}
\newcommand{\Qpost}{Q_{\mathrm{post}}}
\newcommand{\must}{\mu^\star}
\newcommand{\Sigpost}{\Lambda_{\mathrm{post}}}
\newcommand{\mupost}{v}
\newcommand{\mutil}{\widetilde{\mu}}
\newcommand{\xtil}{\widetilde{x}}
\newcommand{\pihatpt}{\widehat{\pi}^{\mathrm{pt}}}
\newcommand{\pipostreg}{\widehat{\pi}^{\mathrm{post},\lambda}}
\newcommand{\algcomment}[1]{{\color{blue} \texttt{// #1}\xspace}}
\newcommand{\frakDhat}{\widehat{\frakD}}
\newcommand{\pibar}{\bar{\pi}}

\iftoggle{arxiv}{}{\vspace{-0.5em}}
\vspace{-0.5em}
\section{Preliminaries}
\vspace{-0.5em}
\iftoggle{arxiv}{}{\vspace{-0.5em}}
\textbf{Mathematical notation.} Let $\lesssim$ denote inequality up to absolute constants, $\simplex_{\mathcal{X}}$ the simplex over $\mathcal{X}$, and $\unif(\mathcal{X})$ the uniform distribution over $\mathcal{X}$. $\bbI \{ \cdot \}$ denotes the indicator function, $\Exp^{\pi}[\cdot]$ the expectation under policy $\pi$ and, unless otherwise noted, $\Exp[\cdot]$ the expectation over the demonstrator dataset.\loose

\textbf{Markov decision processes.}
We consider decision-making in the context of episodic, fixed-horizon Markov decision processes (MDPs). An MDP $\cM$ is denoted by a tuple $(\cS, \cA, \{ P_h \}_{h=1}^H, P_0, r, H)$, where $\cS$ is the set of states, $\cA$ the set of actions, $P_h : \cS \times \cA \rightarrow \simplex_{\cS}$ the next-state distribution at step $h$, $P_0 \in \simplex_{\cS}$ the initial state distribution, $r_h : \cS \times \cA \rightarrow \simplex_{[0,1]}$ the reward distribution, and $H$ the horizon. 
Interaction with $\cM$ proceeds in episodes of length $H$. At step $1$, we sample a state $s_1 \sim P_0$, take an action $a_1 \in \cA$, receive reward $r_1(s_1,a_1)$, and transition to state $s_2 \sim P_1(\cdot \mid s_1, a_1)$. This continues for $H$ steps until the MDP resets.
We let $\cJ(\pi) := \Exp^{\pi}[\sum_{h=1}^H r_h(s_h,a_h)]$ denote the expected reward for policy $\pi$. In general, our goal is to find a policy that maximizes $\cJ(\pi)$.\loose

\textbf{Behavioral cloning.}
We assume we are given some dataset $\frakD = \{ (s_1^t, a_1^t, \ldots , s_H^t, a_H^t) \}_{t=1}^T$ collected by running a \emph{demonstrator} policy $\pibeta$ on $\cM$, so that $ (s_1^t, a_1^t, \ldots , s_H^t, a_H^t) $ denotes a trajectory rollout of $\pibeta$ on $\cM$, with $a_h^t \sim \pibeta_h(\cdot \mid s_h^t)$. We assume that $\pibeta$ is Markovian but otherwise make no further assumptions on it (so, in particular, $\pibeta$ may be stochastic and suboptimal). Our demonstrator dataset does not include reward labels---preventing standard offline RL approaches from applying---but we assume that we have access to reward labels during online interactions.

\emph{Behavioral cloning} (\bc) attempts to fit a policy $\pihatbeta$ to match the distribution of actions observed in $\frakD$.
Typically this is achieved via supervised learning, where $\pihatbeta$ is trained to predict $a$ given $s$ for $(s,a) \in \frakD$. 
In the tabular setting, which we consider in \Cref{sec:theory}, the natural choice for $\pihatbeta$ models the empirical action distribution in $\frakD$:
\iftoggle{arxiv}{
\begin{align}\label{eq:tab_bc_policy}
\pihatbeta_h(a \mid s) := \begin{cases}
\frac{T_h(s,a)}{T_h(s)} & T_h(s) > 0 \\
\unif(\cA) & T_h(s) = 0,
\end{cases}
\end{align}
}{
\begin{align}\label{eq:tab_bc_policy}
\pihatbeta_h(a \mid s) := 
\tfrac{T_h(s,a)}{T_h(s)} \cdot \bbI \{ T_h(s) > 0 \} + \unif(\cA) \cdot \bbI \{ T_h(s) = 0 \}
\end{align}}
where $T_h(s,a) = \sum_{t=1}^T \bbI \{ (s_h^t, a_h^t) = (s,a) \}$ and $T_h(s)  = \sum_{t=1}^T \bbI \{ s_h^t = s \}$. The following result bounds the suboptimality of this estimator, and shows that it is optimal, up to log factors. 
\iftoggle{arxiv}{
\begin{proposition}[\cite{rajaraman2020toward}]
If $\frakD$ contains $T$ demonstrator trajectories, we have
\begin{align*}
\textstyle \cJ(\pibeta) - \Exp[\cJ(\pihatbeta)] \lesssim \frac{H^2 S \log T}{T}.
\end{align*}
Furthermore, for any estimator $\pihat$, there exists some MDP $\cM$ and demonstrator $\pibeta$ such that
\begin{align*}
\textstyle \cJ(\pibeta) - \Exp[\cJ(\pihat)] \gtrsim \min \left \{ H, \frac{H^2 S}{T} \right \}.
\end{align*}
\end{proposition}
}{
\begin{proposition}[\cite{rajaraman2020toward}]
If $\frakD$ contains $T$ demonstrator trajectories, we have $\cJ(\pibeta) - \Exp[\cJ(\pihatbeta)] \lesssim \frac{H^2 S \log T}{T}$. Furthermore, for any estimator $\pihat$, there exists some MDP $\cM$ and demonstrator $\pibeta$ such that $\cJ(\pibeta) - \Exp[\cJ(\pihat)] \gtrsim \min \{ H, \frac{H^2 S}{T} \}$.
\end{proposition}}
In other words, without additional reward information, we cannot in general hope to obtain a policy from $\frakD$ that does better than \eqref{eq:tab_bc_policy}, if our goal is to maximize the performance of the pretrained policy. Note that the \bc estimator in \eqref{eq:tab_bc_policy} is, under the uniform demonstrator prior (i.e. the prior under which the demonstrator is equally likely to play each action in each state), the \emph{maximum a posterior} (MAP) estimate of the demonstrator's behavior.

\iftoggle{arxiv}{}{\vspace{-0.5em}}
\vspace{-0.5em}
\section{Achieving Demonstrator Action Coverage via Posterior Sampling}\label{sec:theory}
\vspace{-0.5em}
\iftoggle{arxiv}{}{\vspace{-0.5em}}
In this section we seek to understand how pretraining affects the ability to further improve the downstream policy with RL finetuning, and how we might pretrain to enable downstream improvement. For simplicity, here we assume that our MDP $\cM$ is tabular, and let $S$ and $A$ denote the cardinalities of the state and action spaces, respectively; we will show how our proposed approach can be extended to more general settings in the following section.

\iftoggle{arxiv}{}{\vspace{-0.5em}}
\vspace{-0.25em}
\subsection{Demonstrator Action Coverage}\label{sec:act_coverage}
\vspace{-0.25em}
\iftoggle{arxiv}{}{\vspace{-0.5em}}
The performance of RL finetuning depends significantly on the RL algorithm applied. Rather than limiting our results to a particular RL algorithm, we instead focus on what is often a prerequisite for effective application of any such approach---demonstrating that the \emph{support} of the pretrained policy is sufficient to enable improvement. In particular, we consider the following definition for the ``effective'' support of a policy, relative to the demonstrator policy $\pibeta$.
\begin{definition}[Demonstrator Action Coverage]
We say policy $\pi$ achieves demonstrator action coverage with parameter $\gamma > 0$ if, for all $(s,h) \in \cS \times [H]$ and $a \in \cA$, we have $\pi_h(a \mid s) \ge \gamma \cdot \pibeta_h(a \mid s)$.\loose
\end{definition}
The majority of RL finetuning approaches rely on rolling out the pretrained policy---which we denote as $\pihatpt$---online, and using the collected observations to finetune its behavior. If our pretrained policy achieves demonstrator action coverage with parameter $\gamma$,
then this ensures that any action sampled by $\pibeta$ will also be sampled by $\pihatpt$ in these rollouts (with some probability). While this is not a \emph{sufficient} condition for online improvement, it is a \emph{necessary} condition, in some cases, for performing as well as the demonstrator $\pibeta$ (as \Cref{prop:bc_fails} in the following shows), and is therefore also a necessary condition for improving over $\pibeta$.
Furthermore, the \emph{value} of $\gamma$ has impact on the cost of RL finetuning. 
A policy $\pi$ which achieves demonstrator action coverage with parameter $\gamma$
requires a factor of $1/\gamma$ more samples than $\pibeta$ to ensure it samples some action in the support of $\pibeta$. 
For approaches such as Best-of-$N$ sampling that rely on sampling many actions from the pretrained policy and then taking the best one, a large value of $\gamma$ therefore ensures we can efficiently sample actions likely to be sampled by the demonstrator policy $\pibeta$, while if $\gamma$ is small, it may take a significant number of samples to sample an action necessary for improvement.
In addition, a small value of $\gamma$ may impact the statistical cost of RL finetuning---for small $\gamma$ we may require a large number of online rollouts to observe the behavior of actions that $\pibeta$ plays, which is necessary to ensure we can match the performance of $\pibeta$ after RL finetuning.

\textbf{Problem Statement: Demonstrator Action Coverage with \bc-Pretrained Performance.}
In the following, we aim to understand how we can pretrain policies that achieve demonstrator action coverage with values of $\gamma$ as large as possible.
Furthermore, we aim to achieve this without incurring significant additional suboptimality as compared to $\pihatbeta$, the \bc-pretrained policy---we would like to ensure that $\pihatpt$ is an effective initialization for finetuning while still itself achieving performance comparable to the \bc policy, the optimal policy judged on pretrained performance alone.

\iftoggle{arxiv}{}{\vspace{-0.5em}}
\vspace{-0.25em}
\subsection{Behavioral Cloning Fails to Achieve Demonstrator Action Coverage}
\vspace{-0.25em}
\iftoggle{arxiv}{}{\vspace{-0.5em}}
We first consider standard \bc, i.e. \eqref{eq:tab_bc_policy}.
The following result shows that the estimator in \eqref{eq:tab_bc_policy}, despite achieving the best possible suboptimality rate,
can fail to achieve a meaningful guarantee on demonstrator action coverage, and that this fundamentally limits its ability to serve as an effective initialization for finetuning. 
\begin{proposition}\label{prop:bc_fails}
Fix $\epsilon \in (0, 1/8]$. Then there exist some MDPs $\cM^1,\cM^2$ and demonstrator policy $\pibeta$ such that, if $\cM \in \{ \cM^1, \cM^2 \}$, unless $T \ge \frac{1}{20\epsilon}$, we have that, with probability at least $1/2$:
\begin{align*}
\cJ(\pibeta) - \epsilon > \max_{\pi \in \Pihat} \cJ(\pi) \quad \text{for} \quad \Pihat := \{ \pi : \pi_h(a \mid s) = 0 \text{ if } \pihatbeta_h(a \mid s) = 0, \forall s , a, h \}.
\end{align*}
Furthermore, for any $T' > 0$,
\begin{align*}
    \min_{\pihat^{T'}} \max_{i \in \{1, 2 \}} \Exp^{\cM^i, \pihatbeta}[\max_\pi \cJ^{\cM^i}(\pi) - \cJ^{\cM^i}(\pihat^{T'})] \ge \frac{1}{2},
\end{align*}
where $\Exp^{\cM^i, \pihatbeta}[\cdot]$ denotes the expectation over trajectories generated by rolling out $\pihatbeta$ on $\cM^i$, and $\pihat^{T'}$ is a policy estimator obtained after $T'$ such rollouts.
\end{proposition}

\Cref{prop:bc_fails} shows that, unless we have a sufficiently large demonstration dataset, half of the time (i.e. half of the random draws of the demonstrator dataset) the policy returned by standard \bc will not contain a near-optimal policy in its support and, furthermore, that rolling out $\pihatbeta$ on $\cM$ will not allow us to learn a near-optimal policy on $\cM$. In other words, some fraction of the time standard \bc produces a policy which cannot be improved with RL finetuning approaches that rely on rolling out the pretrained policy. Furthermore, this shows that demonstrator action coverage is, in some cases, a necessary condition for successful RL improvement---without this, we simply will not sample actions played by the demonstrator, and will therefore be unable to determine which actions actually lead to the best performance.

The key failing of \bc in \Cref{prop:bc_fails} is that, if it has not yet observed the demonstrator play an action, it will simply not play this action---it overfits to the actions it has observed.
A straightforward solution to this is to simply add exploration noise to our pretrained policy---rather than playing $\pihatbeta$ at every step, with some probability play a random action.
While this will clearly address the shortcoming of \bc outlined above---the pretrained policy will now play \emph{every} action---as the following result shows, there is a fundamental tradeoff between the suboptimality of this policy and the number of samples from the policy required to achieve demonstrator action coverage.\loose

\begin{proposition}\label{prop:unif_fails}
Fix $T > 0$, $H \ge 2$, $S \ge \lceil \log_2 4T \rceil + 2$, $\xi \ge 0$, define $\epsilon := \tfrac{H^2 S \log T}{T} + \xi$,
and assume $\epsilon \le \frac{1}{2}$.
Define the policy $\piunif$ as $\piunif_h ( \cdot \mid s) := (1 - \alpha) \cdot \pihatbeta_h(\cdot \mid s) + \alpha \cdot \unif(\cA)$.
Then there exists some MDP $\cM$ with $S$ states, 2 actions, and horizon $H$ where, in order to ensure that:
\begin{enumerate}
\item $\cJ(\pibeta) - \Exp[\cJ(\piunif)] \le  \epsilon$,
\item $\piunif$ achieves demonstrator action coverage with parameter $\gamma$ and probability at least $1-\delta$, for $\delta \in (0, 1/4e)$,
\end{enumerate}
we must have $\alpha \le 32 \epsilon$ and $\gamma \le \frac{64}{A} \cdot \epsilon$. 
Furthermore, with probability at least $1/4e$, we have
\begin{align*}
\textstyle \cJ(\pibeta) - \tfrac{1}{T} \cdot \epsilon > \max_{\pi \in \Pihat} \cJ(\pi) \quad \text{for} \quad \Pihat := \{ \pi : \pi_h(a \mid s) = 0 \text{ if } \pihatbeta_h(a \mid s) = 0, \forall s , a, h \}.
\end{align*}
\end{proposition}
In order to achieve the $\frac{H^2 S \log T}{T}$ suboptimality rate achieved by standard \bc, \Cref{prop:unif_fails} then shows that we can only guarantee demonstrator action coverage with parameter $\gamma \lesssim \frac{1}{A} \cdot \frac{H^2 S \log T}{T}$. Or, in other words, to ensure we sample a particular action from $\piunif$ that is sampled by $\pibeta$,
it will require sampling a factor of $\frac{AT}{H^2 S \log T}$ \emph{more} samples 
from $\piunif$ than it would require from $\pibeta$.
While this does enable RL improvement from rolling out the pretrained policy,
in settings where $T$ is large it could require a significant number of samples from the pretrained policy to achieve this, greatly increasing the cost of such an approach. Furthermore, \Cref{prop:unif_fails} shows that this limitation is critical---if we seek to shortcut this exploration and set $\alpha \leftarrow 0$, we will fail to match the performance of $\pibeta$ on this instance completely.

\iftoggle{arxiv}{}{\vspace{-0.5em}}
\vspace{-0.25em}
\subsection{Posterior Demonstrator Policy Achieves Demonstrator Action Coverage}
\vspace{-0.25em}
\iftoggle{arxiv}{
Can we do better than \bc or \bc augmented with uniform noise? Here we show that a mixture of the standard \bc policy and the \emph{posterior} on the demonstrator's policy achieves a near optimal balance between policy suboptimality and demonstrator action coverage.
}{
Can we do better than this? Here we show that mixing the \bc policy with the \emph{posterior} on the demonstrator's policy achieves a near optimal balance between suboptimality and demonstrator action coverage.
}
\begin{definition}[Posterior Demonstrator Policy]
Given prior distribution $\betaprior \in \simplex_{\Pi}$ over demonstrator policies, let $\betapost(\cdot \mid \frakD)$ denote the posterior distribution given demonstration dataset $\frakD$.
We then define the \emph{posterior demonstrator policy} $\pipost$ as $\pipost_h(a \mid s) := \Exp_{\pi\sim \betapost(\cdot \mid \frakD)}[\pi_h(a \mid s)]$.
\end{definition}
\iftoggle{arxiv}{
$\pipost$ is therefore the expected posterior policy of the demonstrator under prior $\betaprior$ given observations $\frakD$. Critically, this takes into account the entire posterior distribution of the demonstrator's behavior, in contrast to the MAP estimate produced by standard \bc, which simply returns a point estimate of the behavior.
In the tabular setting, some algebra shows that
\begin{align*}
\pipost_h(a \mid s) = \begin{cases}
\frac{T_h(s, a) + 1}{T_h(s) + A} & T_h(s) > 0 \\
\unif(\cA) & T_h(s) = 0,
\end{cases}
\end{align*}
so that $\pipost_h(a \mid s)$ increases the weight on actions for which $T_h(s,a)$ is very small, as compared to the \bc policy.
In practice, we require a slightly regularized version of $\pipost$, $\pipostreg$, which is identical to $\pipost$ if $HT \lesssim e^A$, and otherwise adds a small amount of additional regularization (see \Cref{sec:post_analysis} for a precise definition). We have the following.
}{
$\pipost$ is the expected policy of the demonstrator under prior $\betaprior$ given observations $\frakD$.
In practice, we require a slightly regularized version of $\pipost$, $\pipostreg$, which is identical to $\pipost$ if $HT \lesssim e^A$, and otherwise adds a small amount of regularization (see \Cref{sec:post_analysis}). We have the following.}
\begin{theorem}\label{thm:main}
Let $\betaprior$ be the uniform distribution over Markovian policies, and set $\pihatpt$ to
\begin{align}\label{eq:pipt_post}
\pihatpt_h(a \mid s) = (1 - \alpha) \cdot \pihatbeta_h(a\mid s) + \alpha \cdot \pipostreg_h(a \mid s)
\end{align}
for $\alpha =  \frac{1}{\max \{ A, H, \log(HT) \}}$. Then
\begin{align*}
\textstyle \cJ(\pibeta) - \Exp[\cJ(\pihatpt)] \lesssim \frac{H^2 S \log T}{T} ,
\end{align*}
and with probability at least $1-\delta$, for all $(s,a,h)$,
\begin{align*}
\textstyle \pihatpt_h(a \mid s) \gtrsim \frac{1}{A + H + \log(HT)} \cdot \min \left \{ \frac{\pibeta_h(a \mid s)}{\log (SH/\delta)}, \frac{1}{A + \log (HT)} \right \}.
\end{align*}
\end{theorem}

\Cref{thm:main} shows that by setting $\pihatpt$ to a mixture of the \bc policy and the posterior demonstrator policy, we obtain the same suboptimality guarantee as standard \bc. 
Furthermore, this policy achieves demonstrator action coverage with $\gamma \approx 1/(A+H)$, only requiring
 a factor of $\approx A+ H$ more samples to ensure we sample a particular action from $\pibeta$ than $\pibeta$ itself does (for actions $a$ such that $\pibeta_h(a \mid s) \lesssim 1/A$, and otherwise requires at most a factor of $A(A+H)$ more). We refer to this approach as \emph{posterior behavioral cloning} (\pbc). The following result shows that the scaling in $\gamma$ \pbc achieves is nearly unimprovable.

\begin{theorem}\label{thm:main_lb}
Fix any $A > 1$ and $T > 1$. Then there exists a family of MDPs $\{ \cM^i \}_{i \in [A]}$ such that each $\cM^i$ has $A$ actions and $S=H=1$, and if any estimator $\pihat$ satisfies
$\cJ^{\cM^i}(\pi^{\beta,i}) - \Exp^{\cM^i}[\cJ(\pihat)] \le c \cdot \frac{H^2 S \log T}{T}$ for all  $i \in [A]$
and some constant $c > 0$, then for $\pihat$ to achieve demonstrator action coverage with respect to $\pi^{\beta,i}$ on each $\cM^i$ with probability at least $\delta \in (0, 1/4]$, we must have $\gamma \le c \cdot \frac{\log T}{A}$.\loose
\end{theorem} 

 \Cref{thm:main_lb} shows that, to match the suboptimality guarantee of the \bc policy, no estimator can achieve demonstrator action coverage with $\gamma$ larger than $\approx 1/A$.
Thus, the demonstrator action coverage achieved by \pbc is nearly unimprovable, matching the lower bound as long as $H \le A$.
In other words, if we want a policy that preserves the optimality of $\pihatbeta$ while playing a diverse enough action distribution to enable further online improvement, mixing the posterior demonstrator policy with the \bc policy achieves a near-optimal tradeoff, playing all actions taken by $\pibeta$ with minimal additional sampling and matching the pretrained performance of the \bc policy. This is in contrast to the \bc policy, which does not achieve a meaningful guarantee on demonstrator action coverage, as well as the \bc policy augmented with random exploration, which in order to match the suboptimality of the \bc policy achieves a very suboptimal guarantee on demonstrator action coverage.\loose

The key insight behind the performance of \pbc is that, if we add entropy to the action distribution at each state proportional to our uncertainty about the demonstrator's behavior at that state, this will not hurt the performance relative to the \bc policy. Intuitively, if we are not certain what the demonstrator's behavior is at a given state, the \bc policy may or may not be correct at this state, so adding additional entropy will not make it worse---we are simply selecting the maximum entropy distribution that \emph{might} explain the observations produced by the demonstrator. 
\pbc expands the action distribution enough to ensure that we cover the demonstrator's true action distribution at such states, while at states where we have enough observations to accurately estimate the demonstrator's action distribution, \pbc will decrease entropy to simply match this distribution. 
This is in contrast to the behavior induced by uniformly adding entropy as in \Cref{prop:unif_fails}---while we want to add significant entropy to states where we are uncertain about the demonstrator's behavior, if we add this entropy to states where we \emph{are} certain the performance could drop significantly below that of the demonstrator, greatly limiting the amount of entropy we can add uniformly.

\newcommand{\atil}{\widetilde{a}}
\newcommand{\cov}{\mathsf{cov}}
\newcommand{\fbar}{\bar{f}}
\newcommand{\wtil}{\widetilde{w}}

\iftoggle{arxiv}{}{\vspace{-0.5em}}
\vspace{-0.5em}
\section{Practical Posterior Behavioral Cloning}\label{sec:posterior_bc}
\vspace{-0.5em}
\iftoggle{arxiv}{}{\vspace{-0.5em}}

The previous section suggests a simple recipe to obtain a pretrained policy amenable to online improvement: compute the posterior demonstrator policy given the demonstration data, then mix the posterior demonstrator policy with the \bc-pretrained policy. In this section we show how this can be instantiated in continuous control settings using expressive generative policy classes.
To this end, in \Cref{sec:gauss_approx} we first consider a simplified Gaussian setting, and in \Cref{sec:practical_pbc} seek to generalize the insights from this simplified setting to more complex domains.

\newcommand{\ahat}{\widehat{a}}

\vspace{-0.25em}
\subsection{Sampling from the Posterior Demonstrator Policy for Gaussian Demonstrators}\label{sec:gauss_approx}
\vspace{-0.25em}
To motivate our practical instantiation, consider the setting where:
\begin{align*}
\pibeta_h(\cdot \mid s) = \cN(\mu_h(s), \sigma_h^2(s) \cdot I),
\end{align*}
for some (unknown) $\mu_h(s) \in \R^d$ and (known) $\sigma_h(s) \in \R$.
Assume we have observations $\frakD = \{ a_1, \ldots, a_T \} \sim \pibeta_h(\cdot \mid s)$, and a $\cN(0,I)$ prior on $\mu_h(s)$.
Our theory suggests that instead of fitting the \bc policy, we should fit the posterior demonstrator policy $\pipost$.
In this Gaussian setting, it is straightforward to show that $\pipost_h(\cdot \mid s)$ is the distribution:
$$\cN\Big (\tfrac{1}{\sigma_h^2(s) + T} \cdot {\textstyle \sum_{t=1}^T} a_t, \tfrac{\sigma_h^2(s)}{\sigma_h^2(s) + T} \cdot I + \sigma_h^2(s) \cdot I \Big ).$$
While in the Gaussian setting we can easily sample from this distribution, we wish to motivate a generalizable procedure that extends to settings where sampling is less straightforward. 
To this end, we first note that the \bc policy (the MAP estimator) is simply the distribution 
$$\cN\Big (\tfrac{1}{\sigma_h^2(s) + T} \cdot {\textstyle \sum_{t=1}^T} a_t,  \sigma_h^2(s) \cdot I \Big ).$$
To generate a sample from $\pipost_h(\cdot \mid s)$, it then suffices to sample from the \bc policy and perturb the sample by noise $w \sim \cN(0, \frac{\sigma_h^2(s)}{\sigma_h^2(s) + T} \cdot I)$. 
The following result, an extension of \cite{osband2018randomized}, shows that there is a close connection between this noise distribution and the posterior on $\mu_h(s)$, and that we can generate a sample from the posterior on $\mu_h(s)$ with a simple optimization procedure.\loose
\begin{proposition}\label{prop:policy_post_opt}   
    We have $\betapost(\cdot \mid \frakD) = \cN(\frac{1}{\sigma_h^2(s) + T} \cdot \sum_{t=1}^T a_t, \frac{\sigma_h^2(s)}{\sigma_h^2(s) + T} \cdot I)$ and, if we set
    \begin{align*}
\textstyle        \muhat_h(s) = \argmin_{\mu} \sum_{i=1}^T \| \mu - \atil_i \|_2^2 + \sigma_h^2(s) \cdot \| \mu - \mutil_h(s) \|_2^2, 
    \end{align*}
    for $\atil_t = a_t + w_t$, $w_t \sim \cN(0, \sigma_h^2(s) \cdot I)$, and $\mutil_h(s) \sim \cN(0,I)$,
    then $\muhat_h(s) \sim \betapost(\cdot \mid \frakD)$.
\end{proposition}
Thus, to generate a sample $w \sim \cN(0, \frac{\sigma_h^2(s)}{\sigma_h^2(s) + T} \cdot I)$, we can first generate samples $\muhat_h(s)$, compute their empirical variance, which we denote as $\cov_h(s)$, and sample from a Gaussian with mean 0 and variance $\cov_h(s)$. Altogether, then, we have the following procedure to sample $\ahat \sim \pipost_h(\cdot \mid s)$:
\begin{enumerate}
\item Compute the \bc policy $\pihatbeta$ from observations $\frakD$.
\item Compute samples from the posterior $\betapost$ using the optimization procedure of \Cref{prop:policy_post_opt}, and estimate the variance $\cov_h(s)$ of the posterior from these samples.
\item Generate samples $\atil \sim \pihatbeta_h(\cdot \mid s)$ and $w \sim \cN(0, \cov_h(s))$, and set $\ahat \leftarrow \atil + w$.
\end{enumerate}
While in the Gaussian setting simpler methods would suffice to generate a sample from $\pipost$, 
critically, each step in this procedure can be easily extended to more complex settings, suggesting a generalizable approach to approximate samples from $\pipost$.

\vspace{-0.25em}
\subsection{Practical Instantiation of Posterior Behavioral Cloning}\label{sec:practical_pbc}
\vspace{-0.25em}
In practice our data likely does not satisfy the above Gaussianity assumption, and we also wish to incorporate function approximation to handle settings where we do not have multiple samples from the same state. 
In this section we show how the above procedure motivates a general approach to sample from $\pipost$ in such settings. 
First, to generate approximate samples from $\betapost$, we generalize the optimization-based procedure of \Cref{prop:policy_post_opt} to the following.
\begin{algorithm}[H]
\begin{algorithmic}[1]
\State \textbf{input:} demonstration dataset $\frakD$, ensemble size $K$, posterior model class $\cF$
\For{$\ell = 1,2,\ldots, K$}
	\State Set $\frakD_\ell$ to ``noisy'' version of $\frakD$ 
    \State Fit $f_\ell$ by solving $f_\ell \leftarrow \argmin_{f \in \cF} \sum_{(s,\atil) \in \frakD_\ell} \| f_\ell(s) - \atil \|_2^2$
\EndFor 
$\cov(\cdot) \leftarrow \sum_{\ell = 1}^K (f_\ell(\cdot) - \fbar(\cdot)) (f_\ell(\cdot) - \fbar(\cdot))^\top$ for $\fbar(\cdot) \leftarrow \frac{1}{K} \sum_{\ell=1}^K f_\ell(\cdot)$
\State \textbf{return} $\cov(\cdot)$
\end{algorithmic}
\caption{Posterior Variance Approximation via Ensembled Prediction}
\label{alg:posterior_variance}
\end{algorithm}
\vspace{-1em}
\Cref{alg:posterior_variance} fits an ensemble of predictors to a perturbed version of $\frakD$ in order to approximate a posterior sample, and uses these samples to approximate the posterior covariance. While these samples may not correspond precisely to the posterior covariance in general settings, \Cref{prop:policy_post_opt} shows that in simple settings they do, suggesting that, at minimum, this is a principled approximation.
In the Gaussian setting we generate a noisy $\frakD$ by perturbing the actions in $\frakD$ with Gaussian noise. In practice, however, other methods to obtain a ``noisy'' version of $\frakD$ can be applied as well. In particular, we found that generating $\frakD_\ell$ by  \emph{bootstrapped sampling} \citep{fushiki2005nonparametric,osband2015bootstrapped,osband2016deep}---where we sample with replacement from $\frakD$---typically outperforms directly adding noise to the actions in $\frakD$. 

Given the approximate posterior covariance $\cov(\cdot)$,  \Cref{sec:gauss_approx} suggests that for any $s$ we can generate approximate samples from $\pipost(\cdot \mid s)$ by first sampling an action from the \bc policy at $s$, and then perturbing the resulting action by posterior noise $w \sim \cN(0, \cov(s))$.
In practice, to avoid this two-stage procedure, we can fit a single policy to $\frakD$ (which would be the \bc policy) but where we perturb each action in $\frakD$ by the posterior noise. Note that the distribution this policy will fit is equivalent to the distribution produced by the above two-stage procedure, as long as we utilize an expressive generative model able to represent the posterior demonstrator policy. We arrive at the following.\loose
\begin{algorithm}[H]
\begin{algorithmic}[1]
\State \textbf{input:} demonstration dataset $\frakD$, generative model class $\pihat^{\theta}$, posterior covariance $\cov(\cdot)$, posterior weight $\alpha$
\For{$i = 1,2,3, \ldots$}
	\State Sample batch $\mathfrak{B}_i \sim \unif(\frakD)$
    \State For all $(s,a) \in \mathfrak{B}_i$, sample $w_s \sim \cN(0,  \cov(s))$, and set $\widetilde{\mathfrak{B}}_i \leftarrow \{ (s, a + \alpha \cdot w_s) : s \in \mathfrak{B}_i \}$
    \State Take gradient step on $\pihat^{\theta}$ for loss computed on $\widetilde{\mathfrak{B}}_i$
\EndFor
\State \textbf{return} $\pihat^{\theta}$
\end{algorithmic}
\caption{Posterior Behavioral Cloning (\pbc)}
\label{alg:posterior_bc}
\end{algorithm}
\vspace{-1em}
Altogether, if $\pihat^{\theta}$ is an expressive generative model, \Cref{alg:posterior_bc} will produce a policy that, instead of fitting the empirical distribution of the demonstrator as \bc does, fits the full posterior demonstrator policy. This approximates the posterior mixture in \Cref{eq:pipt_post}, and, \Cref{thm:main} suggests, leads to a more effective initialization for RL finetuning, instantiating the behavior illustrated in \Cref{fig:paper_fig}. As \Cref{thm:main} suggests that instead of sampling just from $\pipost$ we should sample from a mixture of $\pipost$ and $\pihatbeta$, \Cref{alg:posterior_bc} modulates the weight of the posterior noise by some $\alpha$, allowing us to vary the weight of the mixture by varying $\alpha$.

Furthermore, \Cref{alg:posterior_variance} can be implemented with standard supervised learning training pipelines (it requires only standard regression training), and implementing \Cref{alg:posterior_bc} only requires minor modification to standard generative policy training (simply add noise to the action target for each batch sampled). \pbc then only requires training via standard supervised learning---no RL in pretraining is required---making it a scalable approach and simple modification to existing \bc training pipelines.
While any expressive generative model class can be used for $\pihat^{\theta}$, in practice, for all the following experiments, we utilize a diffusion model. Please see \Cref{sec:app_exp} for further details on the practical instantiation of \pbc.

\iftoggle{arxiv}{}{\vspace{-0.5em}}
\vspace{-0.5em}
\section{Experiments}
\vspace{-0.5em}
\iftoggle{arxiv}{}{\vspace{-0.5em}}

\begin{wrapfigure}{r}{0.38\textwidth}  
    \centering
    \vspace{-1em} 
    \includegraphics[width=\linewidth]{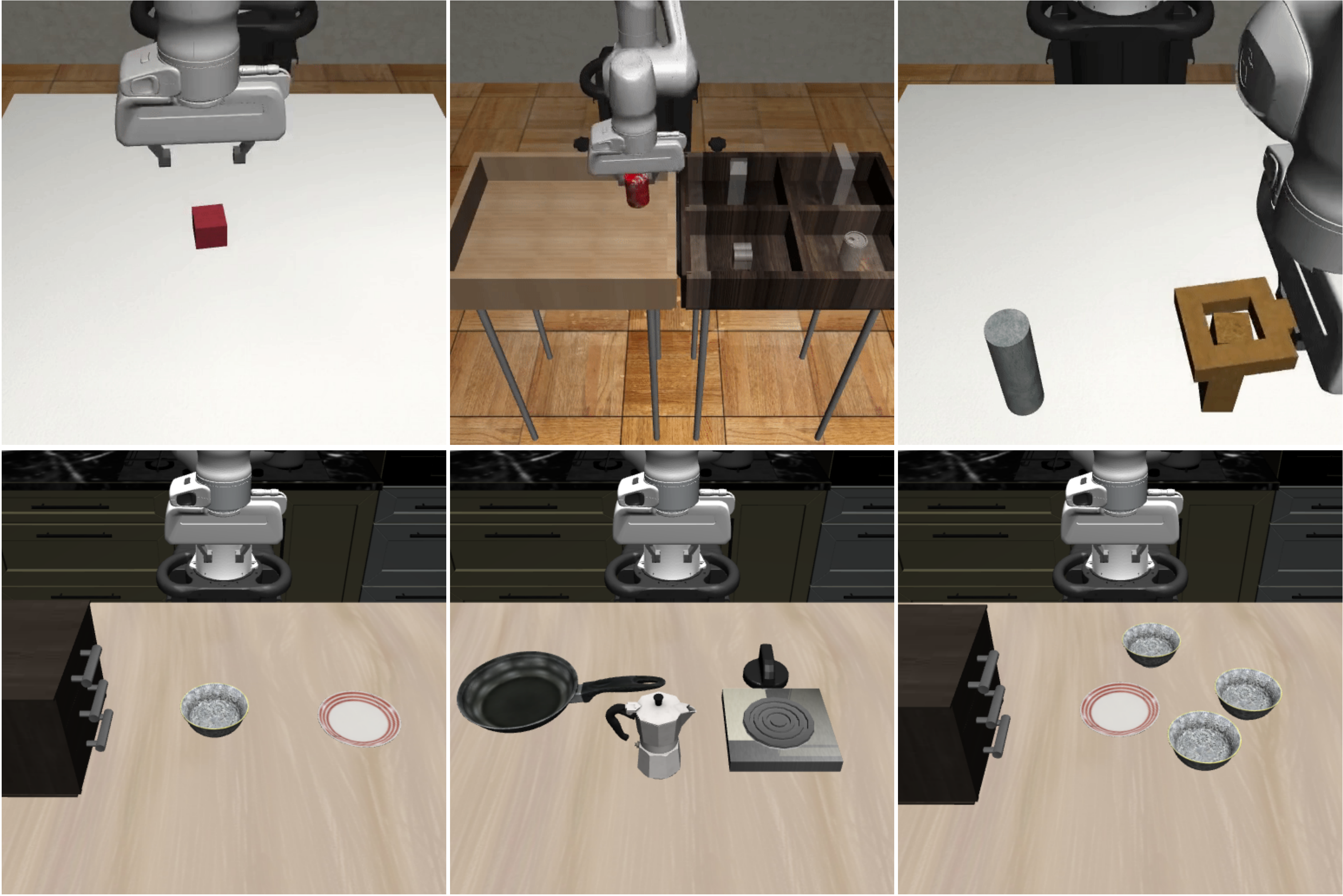}
    \caption{\texttt{Robomimic} and \texttt{Libero} settings}
    \label{fig:robot_visual}
    \vspace{-1em} 
\end{wrapfigure}

Finally, we seek to demonstrate that in practice \pbc (a) enables more efficient RL finetuning of pretrained policies, and (b) produces a pretrained policy that itself performs effectively, on par with the \bc pretrained policy.
We focus on continuous control domains, in particular robotic control. We test on both the \texttt{Robomimic} \citep{mandlekar2021matters} and \texttt{Libero} \citep{liu2023libero} simulators (\Cref{sec:sim_finetune_results,sec:sim_bc_results}), as well as a real-world WidowX 250 6-DoF robot arm (\Cref{sec:widowx_results}).
We consider the \texttt{Lift}, \texttt{Can}, and \texttt{Square} tasks on \texttt{Robomimic}.
\texttt{Robomimic} is comprised of several robotic manipulation tasks, provides a set of human demonstrations on each task, and enables training and finetuning of single-task \bc policies. 
\texttt{Libero} similarly contains a variety of robotic manipulation tasks with provided human demonstrations, but enables multi-task training, allowing for pretraining on large corpora of data and then finetuning on particular tasks of interest. In particular, we rely on a subset of the \texttt{Libero 90} suite of tasks, training and evaluating on the scenes \texttt{Kitchen Scene 1-3} containing a total of 16 tasks.
See \Cref{fig:robot_visual,fig:widowx_visual} for visualizations of our settings. \edit{Further details on all experiments can be found in \Cref{sec:app_exp}.}\loose

We instantiate $\pihatpt$ with a diffusion model, a standard parameterization for \bc policies in robotic control settings \citep{chi2023diffusion,dasari2024ingredients,team2024octo,black2024pi_0,bjorck2025gr00t}. For the \texttt{Robomimic} experiments, we use an MLP-based architecture, pretrain on a single-task demonstration dataset, and rely on state-based observations. For \texttt{Libero}, we utilize a diffusion transformer architecture due to \cite{dasari2024ingredients} and rely on image-based observations and language task conditioning. In \texttt{Libero}, we pretrain a single $\pihatpt$ policy on the demonstration data from all 16 tasks \citep{black2024pi_0,kim2024openvla,khazatsky2024droid}, and then run RL finetuning on each individual task. On the WidowX experiments we utilize a U-Net architecture with image observations. In all cases, we use a binary success reward for the RL finetuning.\loose

In principle, \pbc can be combined with any RL finetuning algorithm, and we seek to demonstrate that it improves performance on a representative set of approaches. In particular, we consider \dsrl \citep{wagenmaker2025steering}, which refines a pretrained diffusion policy's distribution by running RL over its latent-noise space, \dppo \citep{ren2024diffusion}, an on-policy policy-gradient-style algorithm for finetuning diffusion policies, and Best-of-$N$ sampling. 
\edit{For \dsrl and \dppo we utilize the publicly available implementations without modification.}
Best-of-$N$ can be instantiated in a variety of ways (see e.g. \cite{chen2022offline,hansen2023idql,he2024aligniql,nakamoto2024steering,dong2025matters}). \edit{Here we instantiate it by rolling out the pretrained policy on the task of interest $T_{\mathrm{on}}$ times (where $T_{\mathrm{on}}$ is specified in our results) to collect trajectories labeled with success and failure, and train a $Q$-function via \textsc{Iql} \citep{kostrikov2021offline} on these trajectories. At test time, we again roll out the pretrained policy but at each state sample $N$ actions from the policy, and play the action that has the largest value under this $Q$-function.}

\begin{figure*}[t]
  \centering

 \begin{minipage}[t]{0.51\textwidth}
    \centering
    \includegraphics[width=.3775\textwidth]{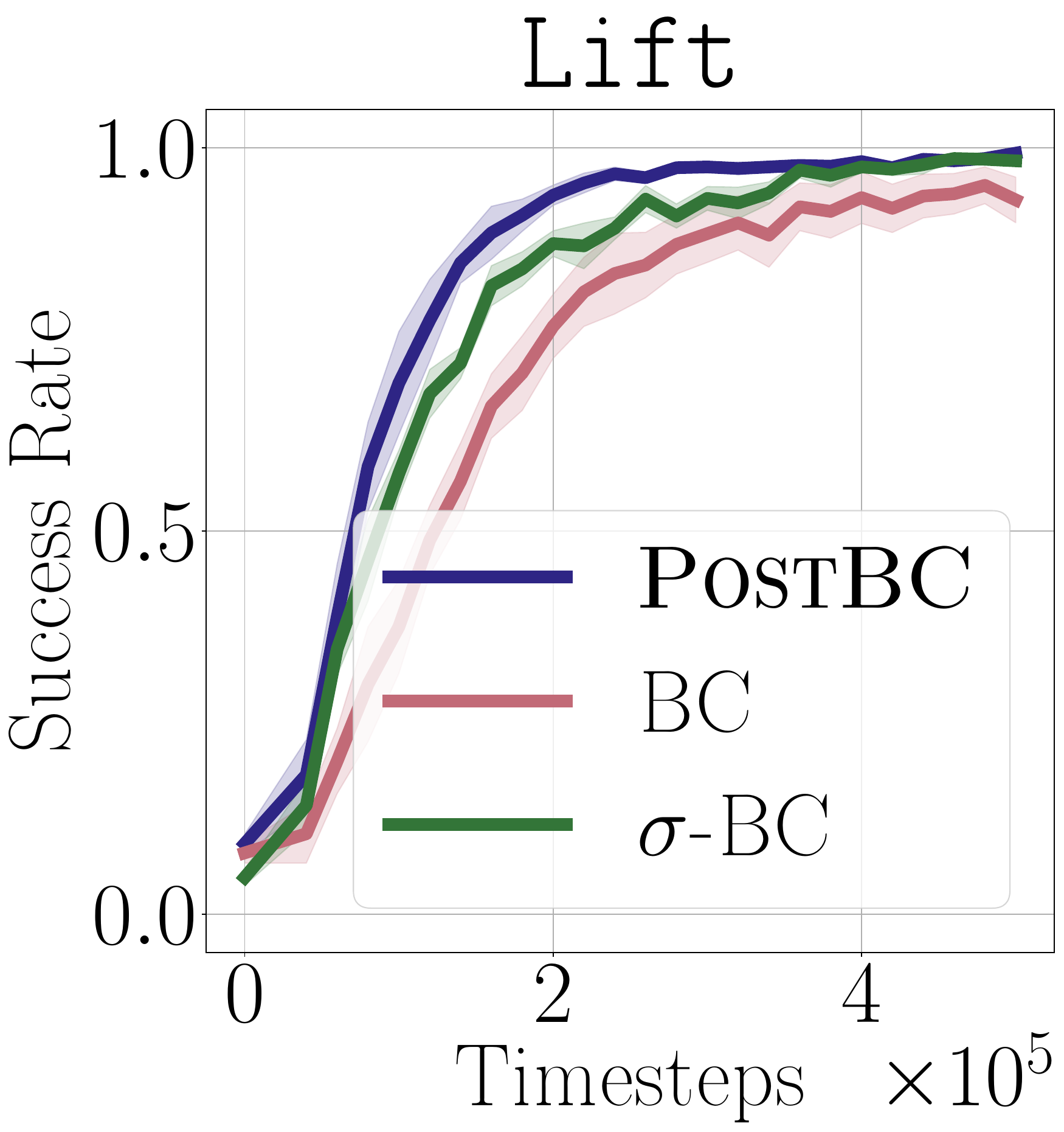}\hfill
  \includegraphics[width=.3125\textwidth]{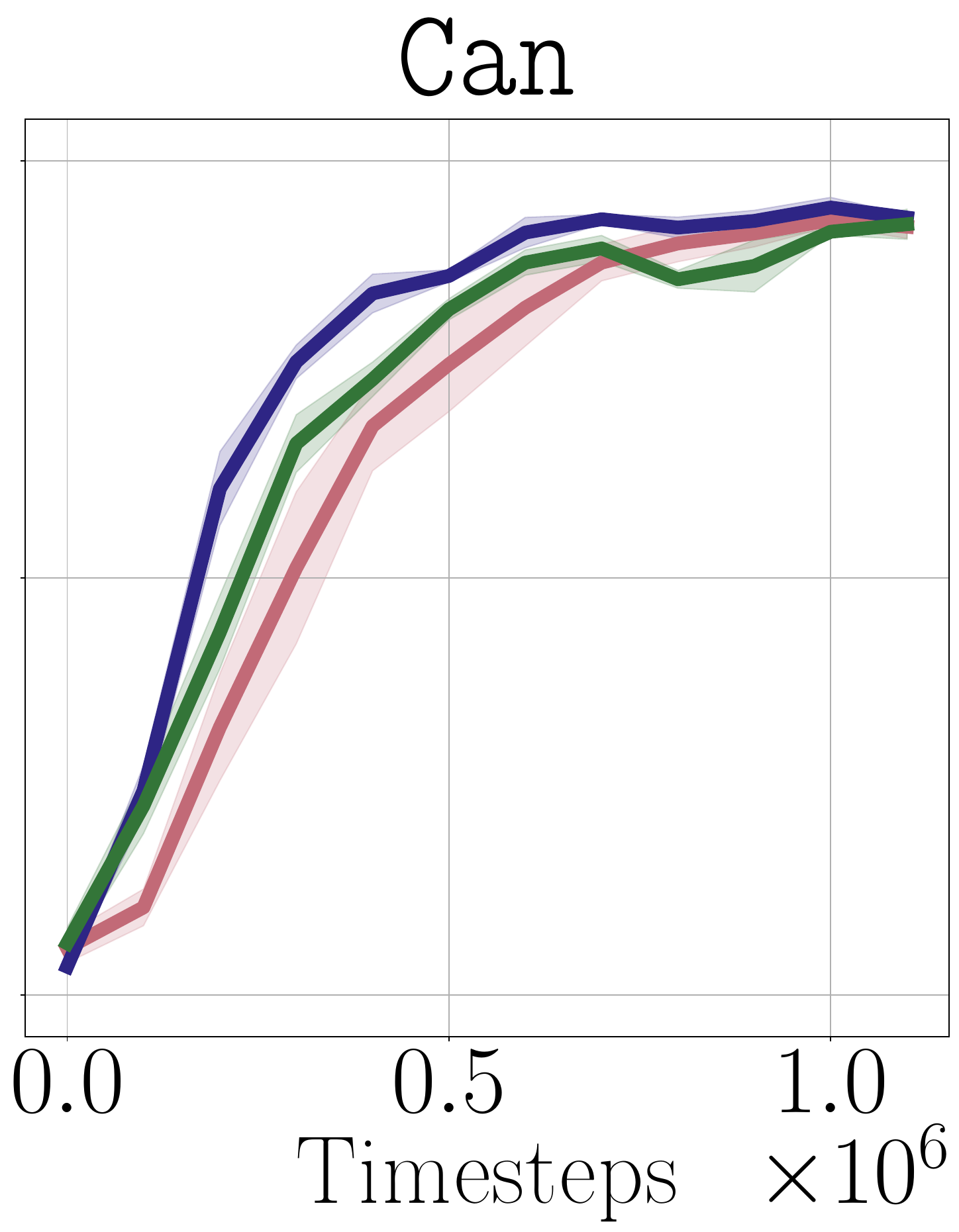}\hfill
  \includegraphics[width=.31\textwidth]{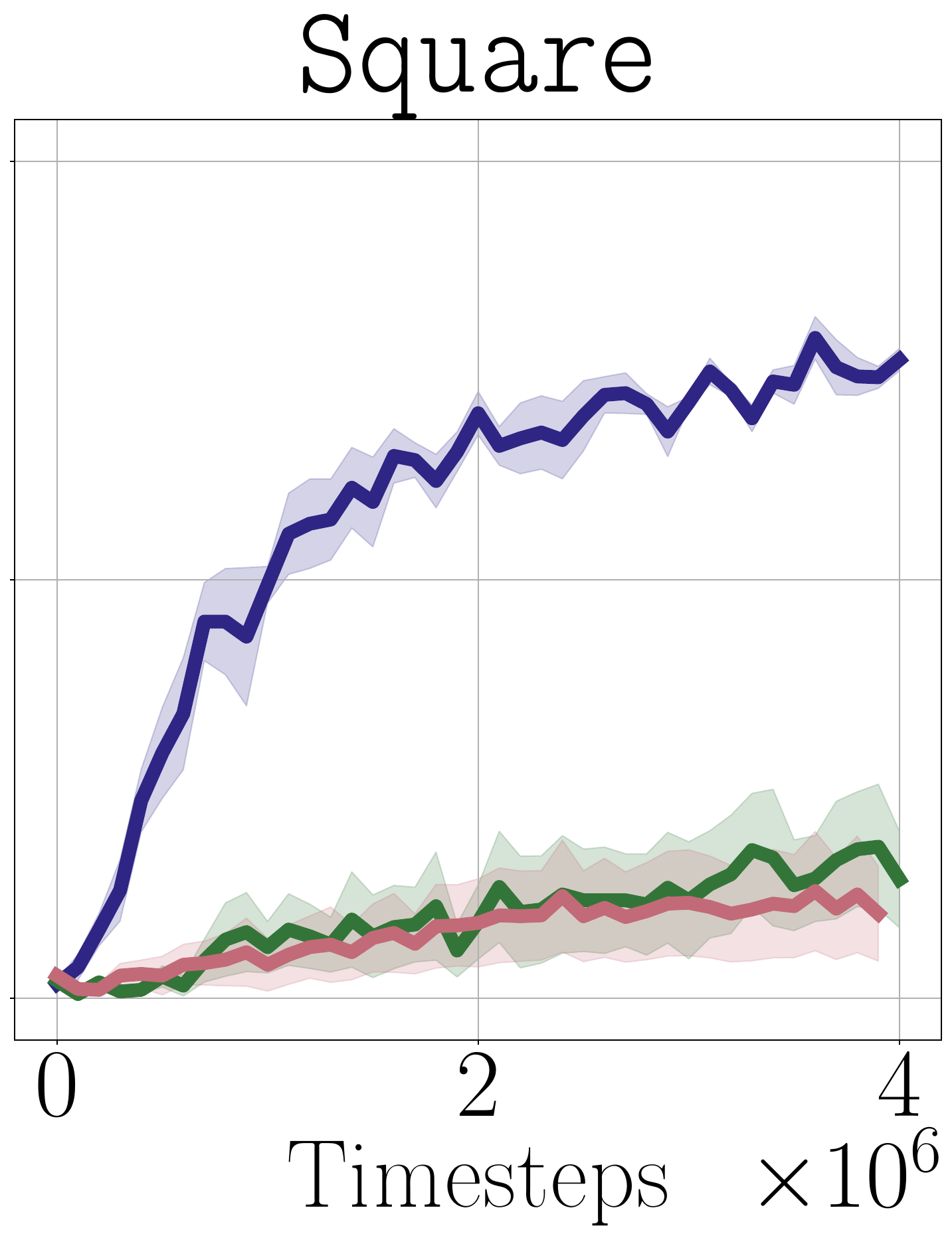}\hfill
  \vspace{-0.5em}
    \caption{\edit{Comparison of \dsrl finetuning performance combined with different \bc pretraining approaches on \texttt{Robomimic}.}}
    \label{fig:dsrl_robomimic}
  \end{minipage}
  \hfill
 \begin{minipage}[t]{0.48\textwidth}
    \centering
    \includegraphics[width=.33\textwidth]{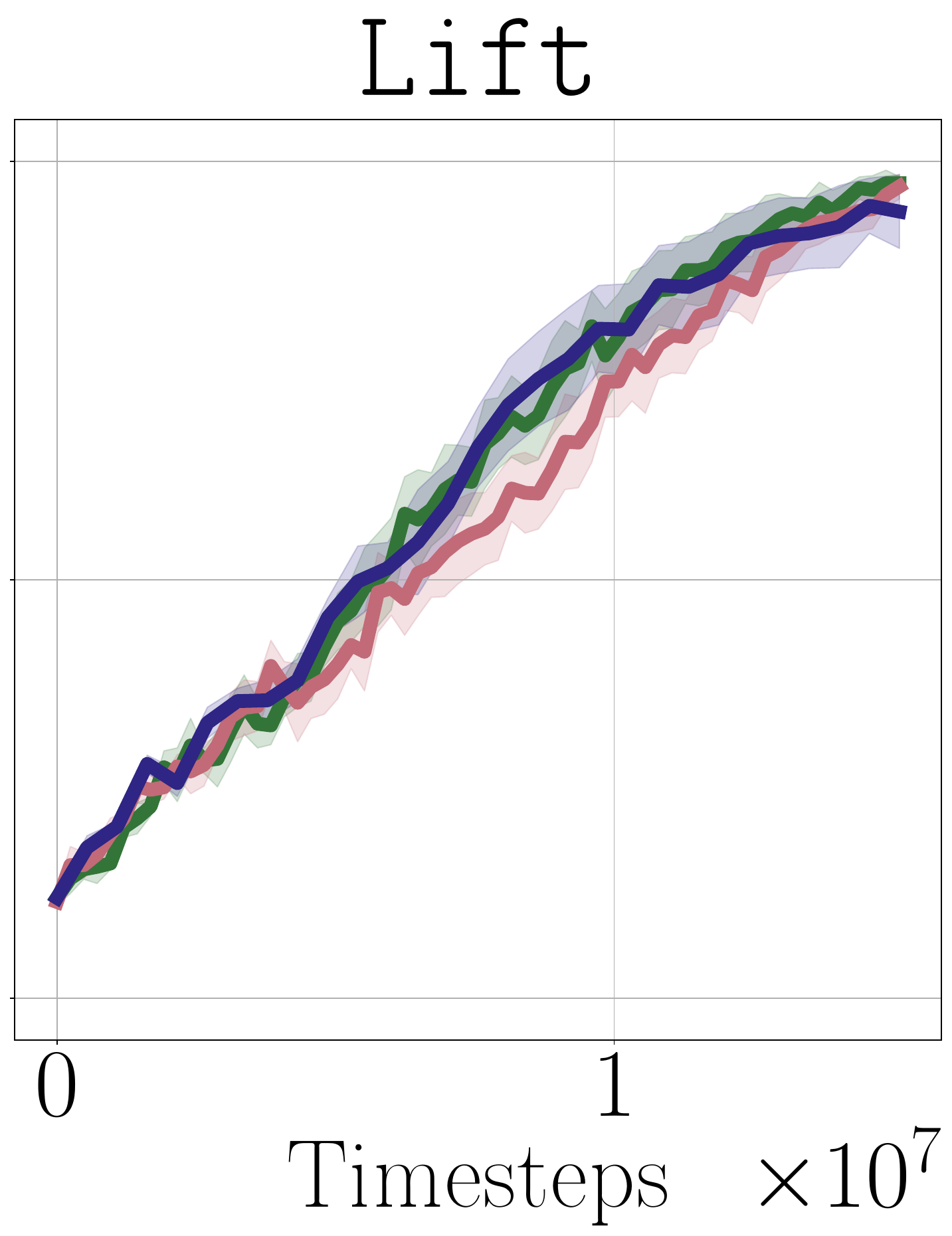}\hfill
  \includegraphics[width=.3375\textwidth]{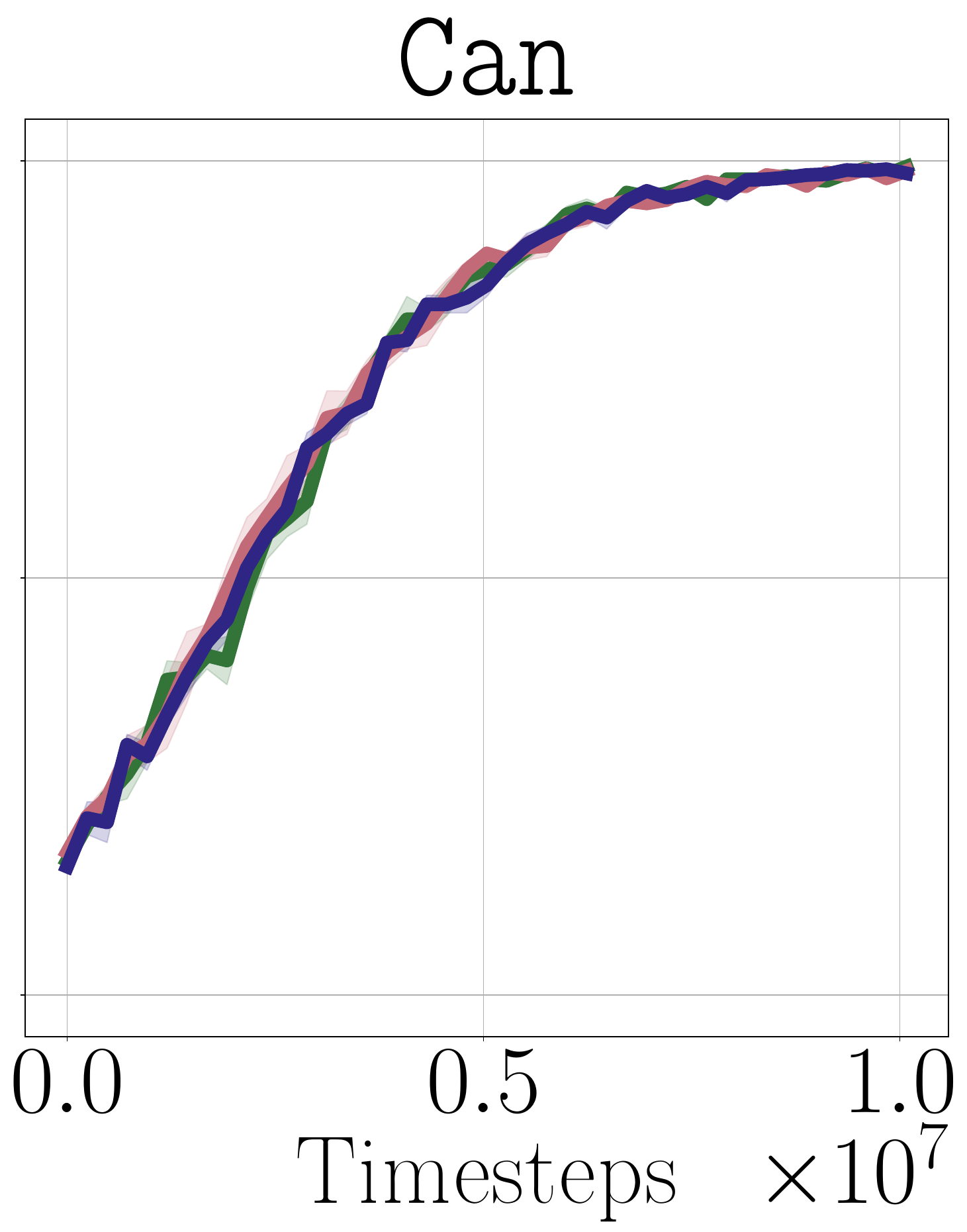}\hfill
  \includegraphics[width=.33\textwidth]{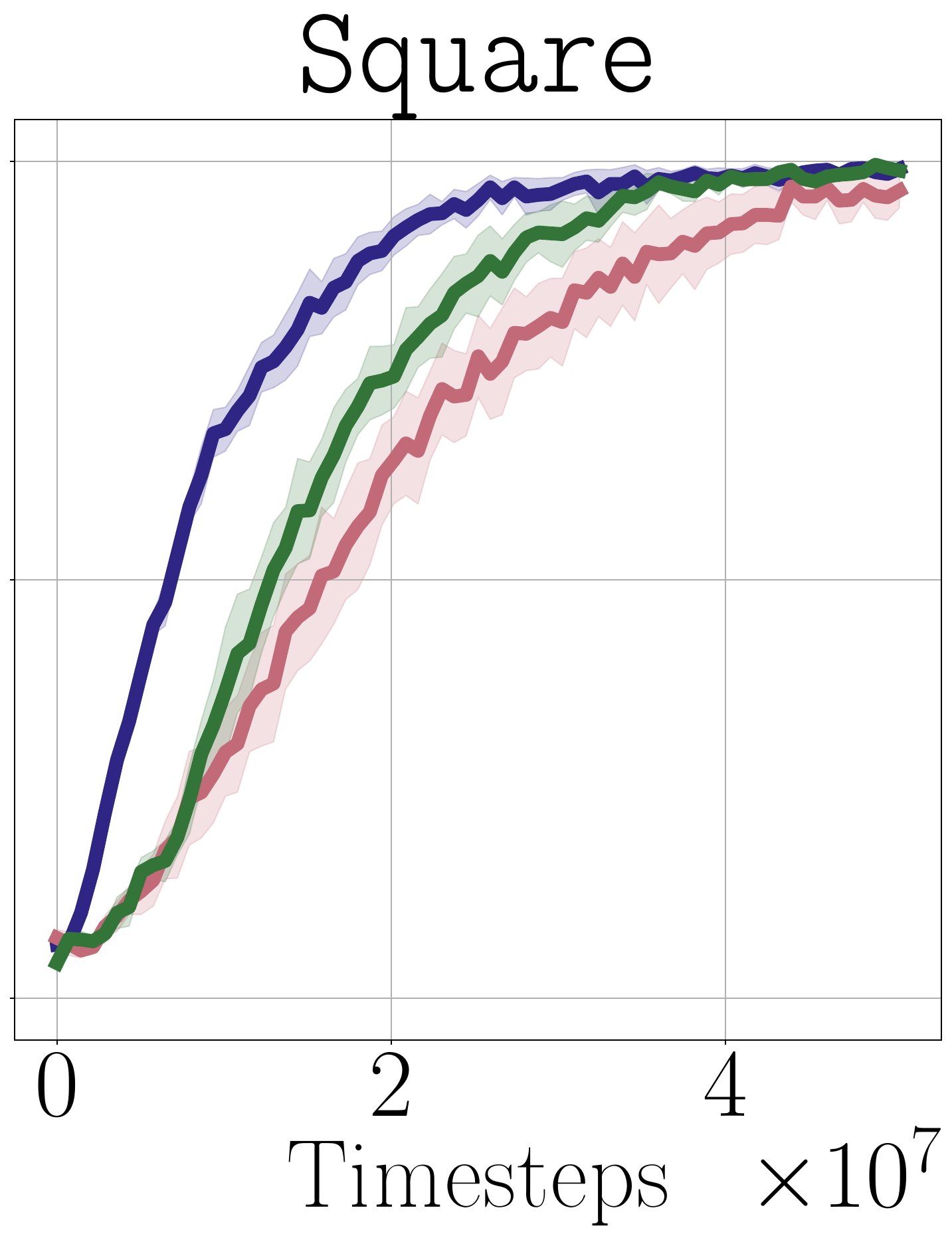}\hfill
  \vspace{-0.5em}
    \caption{\edit{Comparison of \dppo finetuning performance combined with different \bc pretraining approaches on \texttt{Robomimic}.}}
    \label{fig:dppo_robomimic}
  \end{minipage}\hfill
  \vspace{-1.0em}
\end{figure*}

\begin{figure*}[t!]
  \centering
    \includegraphics[width=.26\linewidth]{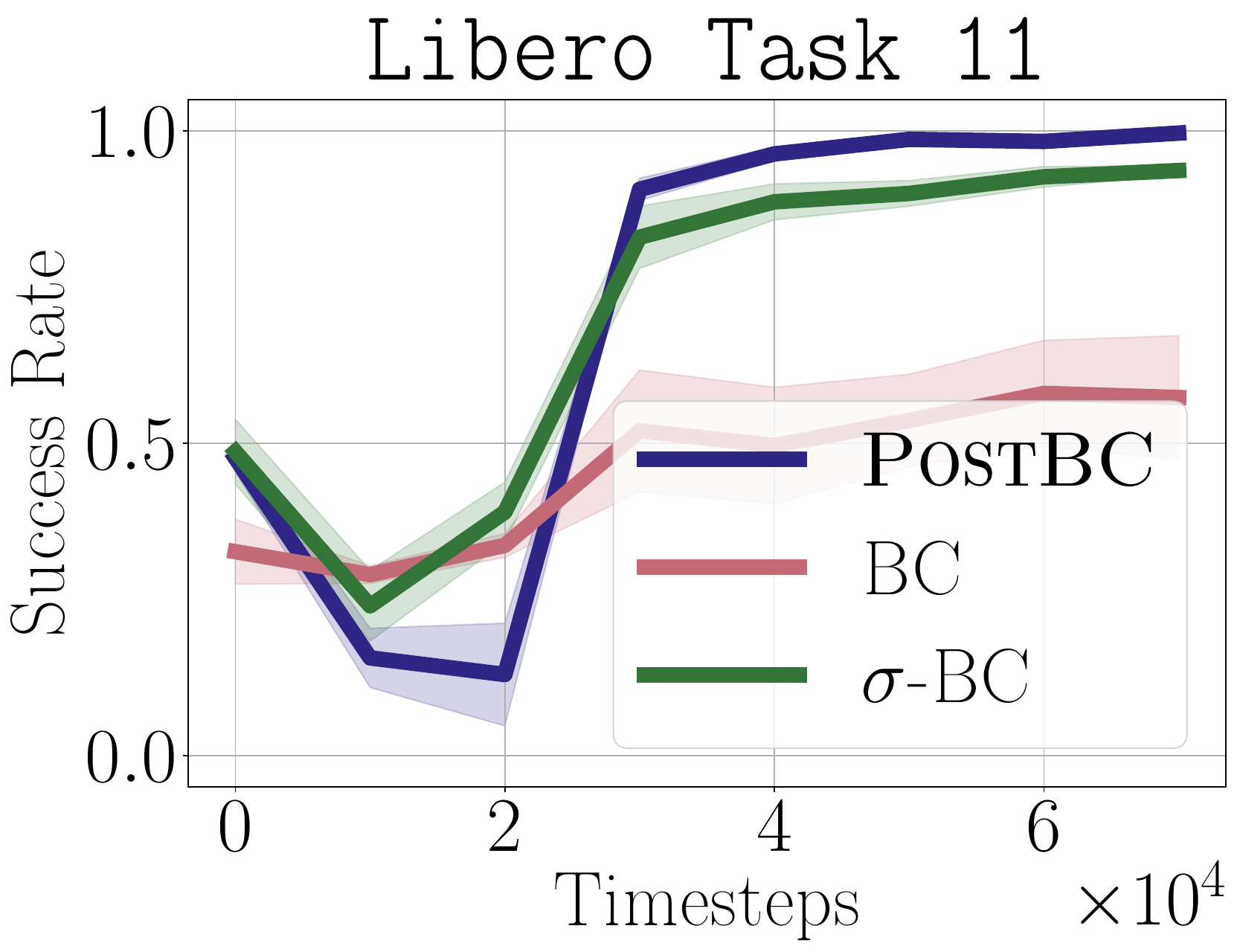}\vspace{0.5em}
    \includegraphics[width=.2239\linewidth]{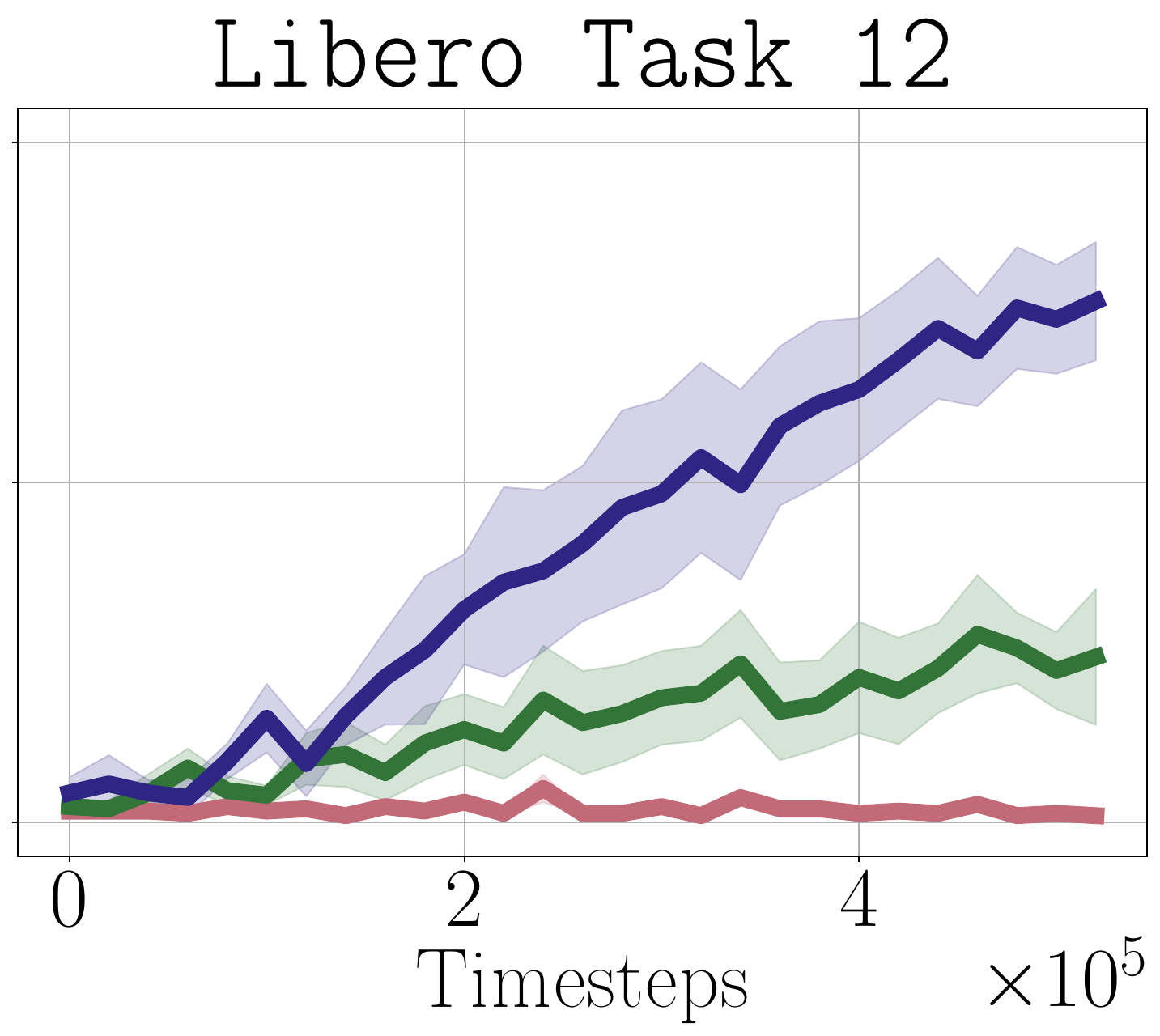}\vspace{0.5em}
    \includegraphics[width=.2239\linewidth]{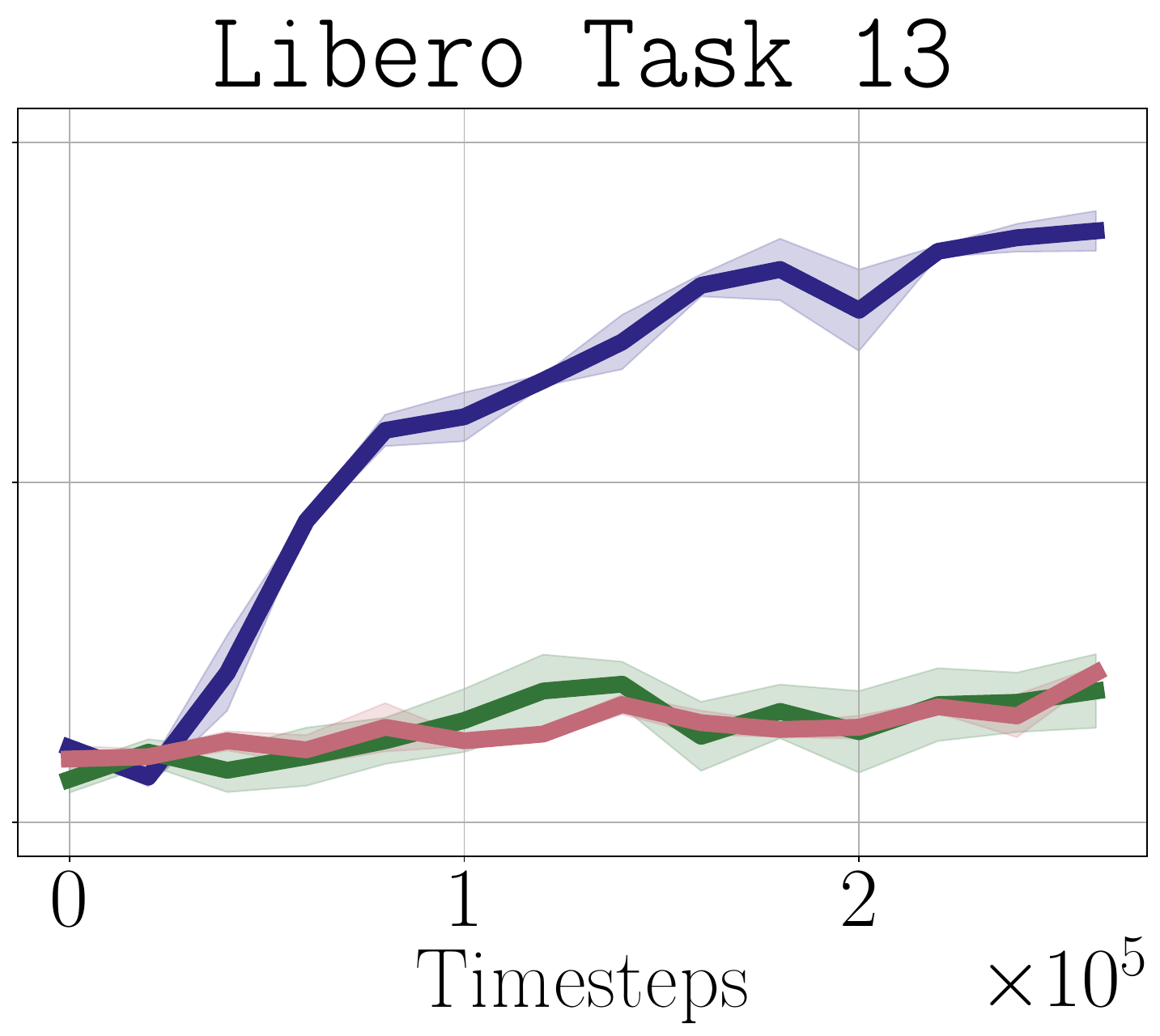}
    \includegraphics[width=.2239\linewidth]{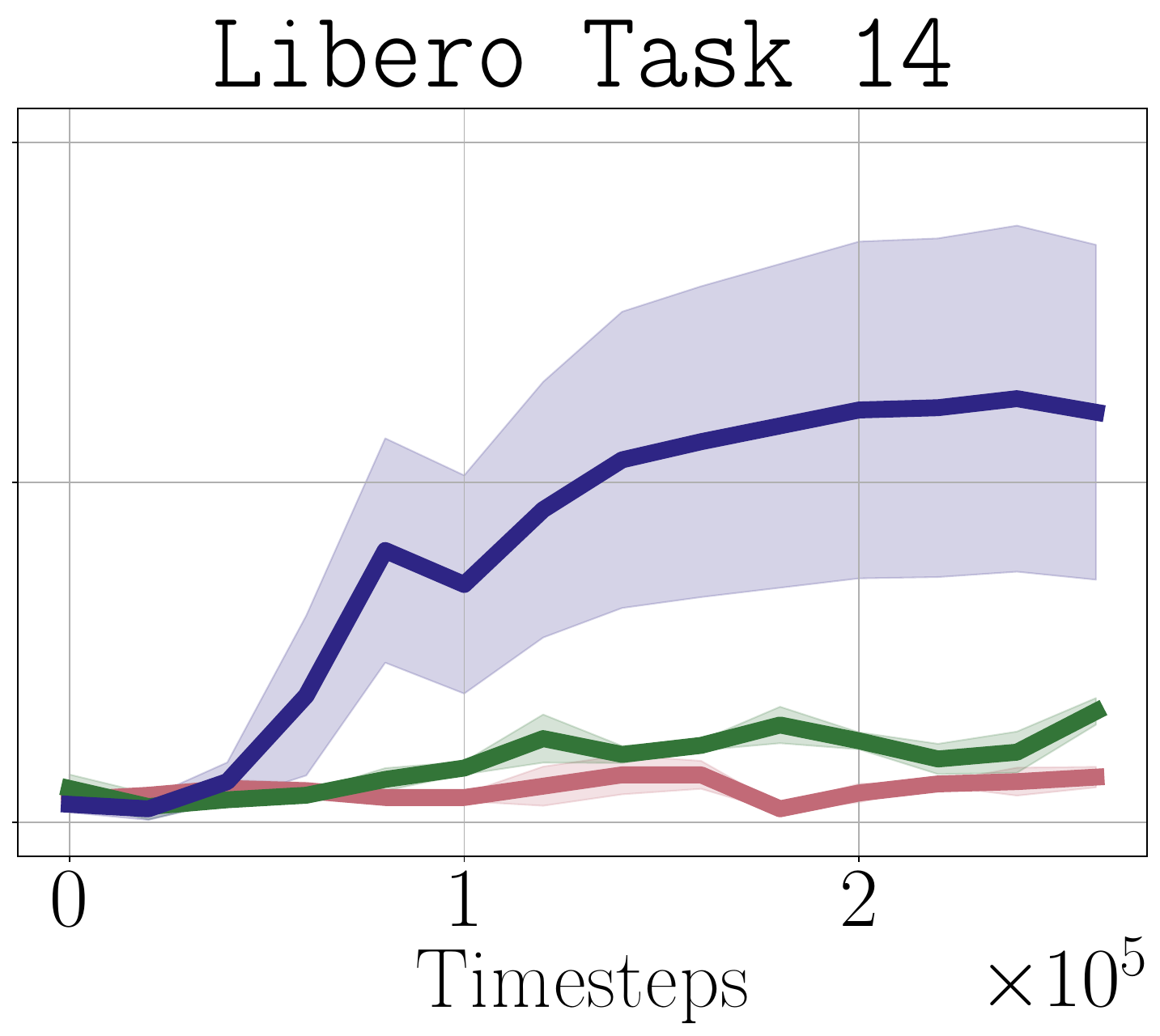}
    \includegraphics[width=.26\linewidth]{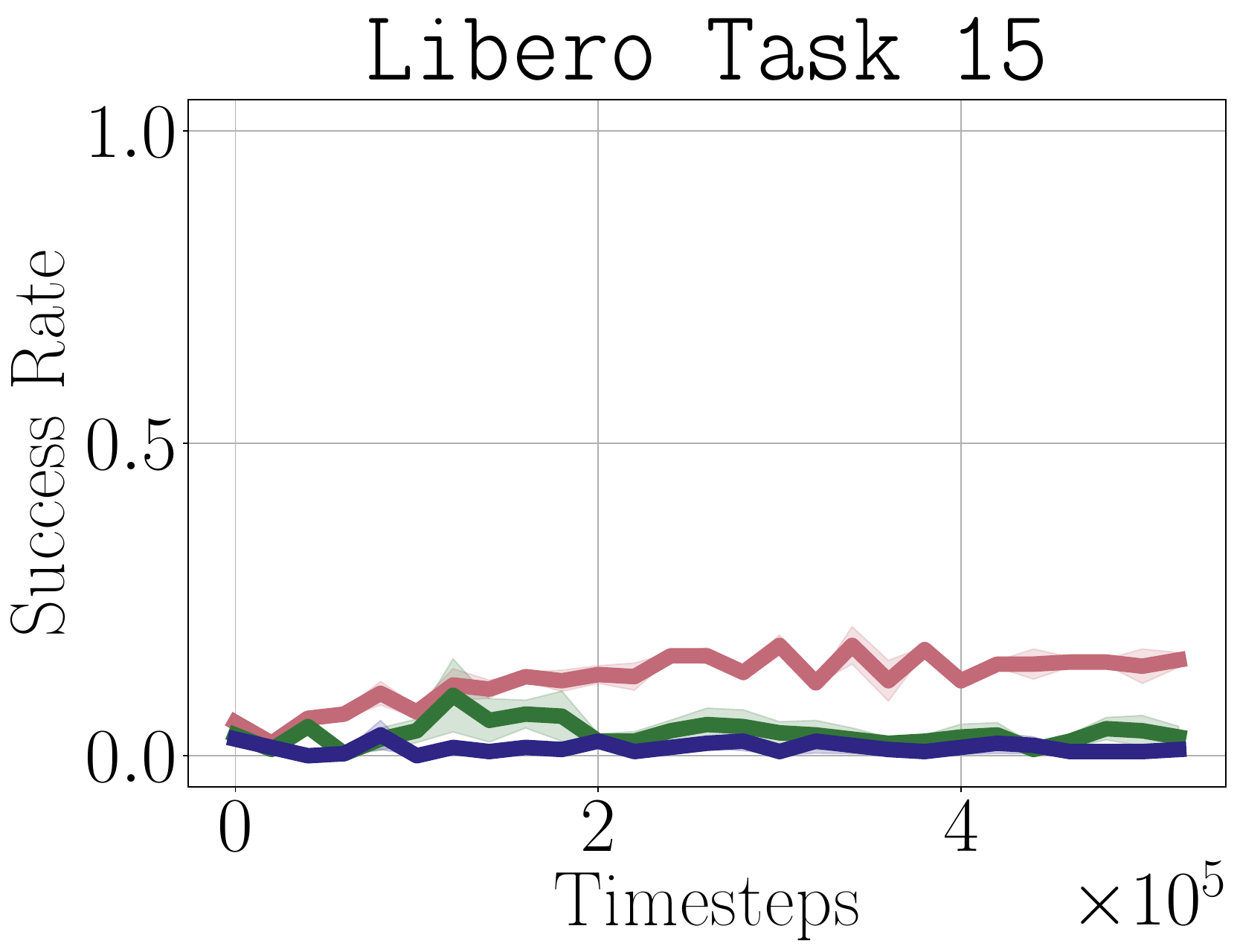}
    \includegraphics[width=.2239\linewidth]{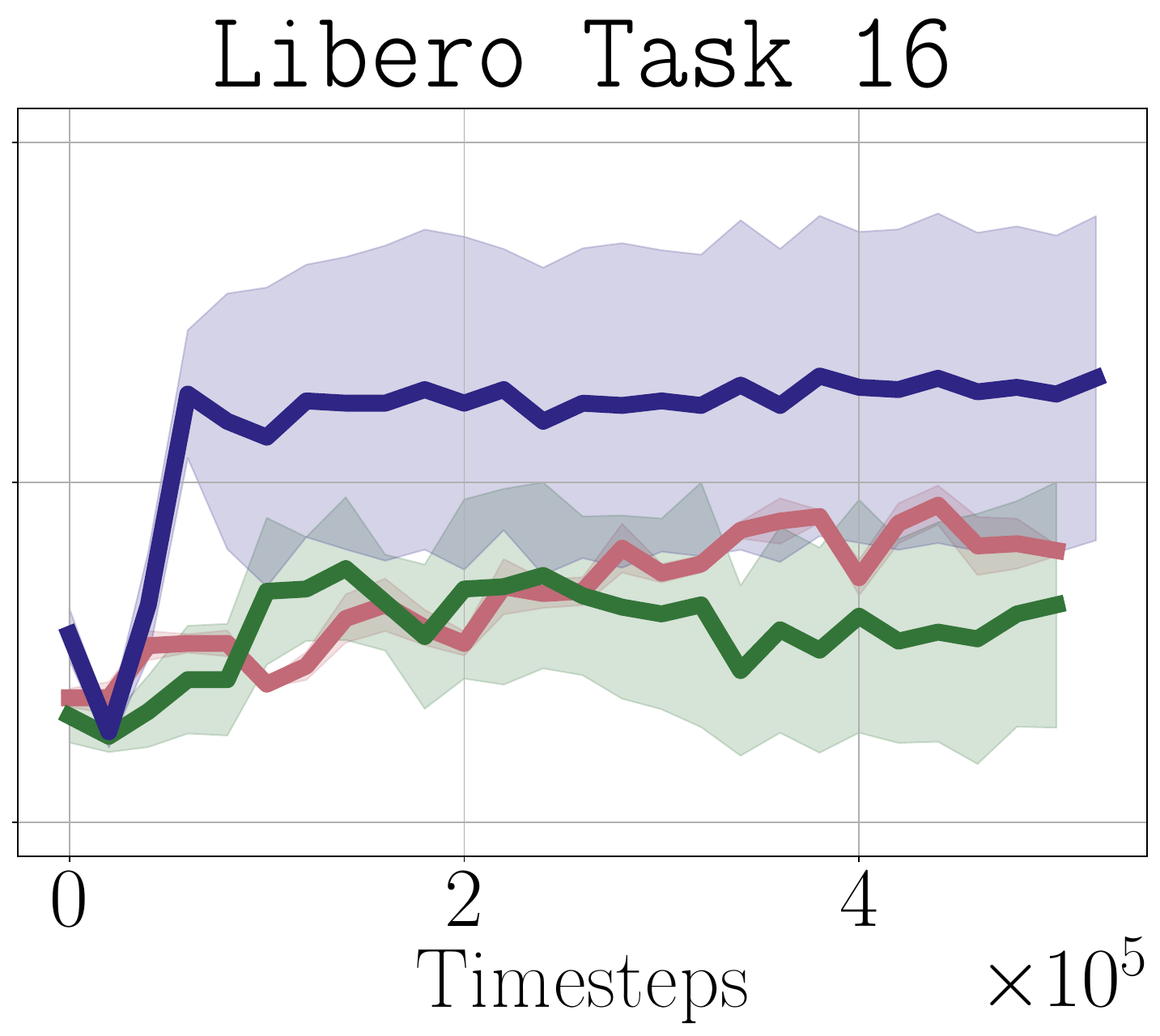}
    \includegraphics[width=.2239\linewidth]{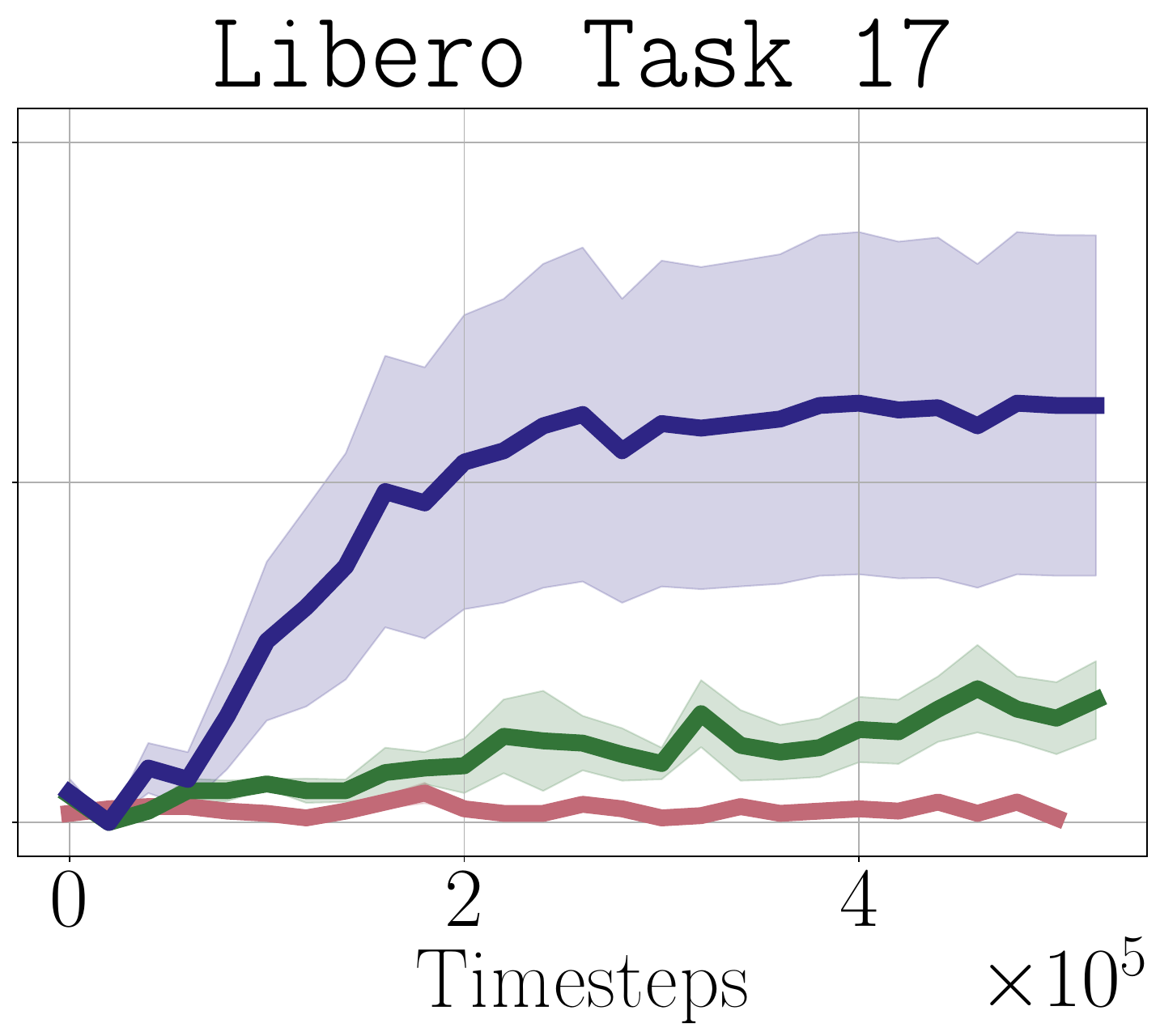}
    \vspace{-0.5em}
    \caption{Comparison of \dsrl finetuning performance combined with different \bc pretraining approaches on all tasks from \texttt{Libero 90}, \texttt{Kitchen Scene 2}.}
    \label{fig:dsrl_libero}
  \vspace{-1.5em}
\end{figure*}

As baselines, we consider running standard \bc pretraining on $\frakD$, as well as what we refer to as \nbc, where instead of perturbing the actions in $\frakD$ by the posterior variance as in \Cref{alg:posterior_bc}, we instead perturb them by uniform, state-independent noise with variance $\sigma^2$. This is then equivalent to \pbc, except we set $\cov(s) = \sigma^2 \cdot I$ for some fixed $\sigma > 0$ in \Cref{alg:posterior_bc} (note that this is a continuous analog to the approach considered in \Cref{prop:unif_fails}). This itself is a novel approach and our theory predicts it too may lead to improved performance over pretraining with standard \bc.
\edit{On \texttt{Robomimic}, we also compare against \textsc{ValueDICE} \citep{kostrikov2019imitation} (which we abbreviate as \dice), as a representative non-\bc imitation learning approach. Rather than training to match the demonstrator's actions via supervised learning, \textsc{ValueDICE} attempts to learn a policy with state distribution matching the state distribution of the demonstrations, and only requires access to offline demonstration data.}
All \texttt{Robomimic} results are averaged over 5 seeds, and \texttt{Libero} results are averaged over 3 seeds, 
and policies are evaluated with 200 rollouts for \texttt{Robomimic} and 100 for \texttt{Libero}. For all experiments, error bars denote 1 standard error.\loose

\iftoggle{arxiv}{}{\vspace{-0.75em}}
\vspace{-0.25em}
\subsection{Posterior Behavioral Cloning Enables Efficient RL Finetuning}\label{sec:sim_finetune_results}
\vspace{-0.25em}
\iftoggle{arxiv}{}{\vspace{-0.5em}}

\begin{table*}[t!]
\centering
\scalebox{0.75}
{
\begin{tabular}{l!{\vrule width 1pt}cccc|cccc}
\toprule
 & \multicolumn{4}{c|}{Best-of-$N$ (1000 Rollouts)} & \multicolumn{4}{c}{Best-of-$N$ (2000 Rollouts)} \\
\texttt{Task}  & \bc & \nbc & \edit{\dice} & \pbc & \bc & \nbc & \edit{\dice} & \pbc \\
\midrule

\texttt{Robomimic Lift}  & $59.4$ {\tiny $\pm {2.2}$} & $61.5$ {\tiny $\pm {3.9}$} & $44.4$ {\tiny $\pm {7.7}$} & $\mathbf{74.2}$ {\tiny $\pm {3.0}$} & $68.1$ {\tiny $\pm {2.2}$} & $\mathbf{76.1}$ {\tiny $\pm {3.5}$} & $47.6$ {\tiny $\pm {7.8}$} & $\mathbf{81.3}$ {\tiny $\pm {5.7}$} \\

\texttt{Robomimic Can}  & $70.3$ {\tiny $\pm {1.7}$} & $\mathbf{77.9}$ {\tiny $\pm {1.2}$} & $17.0$ {\tiny $\pm {5.7}$} & $\mathbf{75.0}$ {\tiny $\pm {2.5}$} & $76.9$ {\tiny $\pm {1.5}$} & $\mathbf{82.4}$ {\tiny $\pm {1.6}$} & $41.8$ {\tiny $\pm {8.4}$} & $\mathbf{84.5}$ {\tiny $\pm {1.3}$} \\

\texttt{Robomimic Square} & $44.8$ {\tiny $\pm {0.7}$} & $48.1$ {\tiny $\pm {2.2}$} & $6.9$ {\tiny $\pm {0.9}$} & $\mathbf{52.4}$ {\tiny $\pm {1.9}$} & $\mathbf{54.4}$ {\tiny $\pm {1.3}$} & $\mathbf{54.2}$ {\tiny $\pm {3.7}$} & $8.3$ {\tiny $\pm {1.3}$} & $\mathbf{56.8}$ {\tiny $\pm {3.2}$} \\

\texttt{Libero Scene 1} (5 tasks) 
&  $37.7$ {\tiny $\pm {2.5}$} & $57.9$ {\tiny $\pm {3.4}$} & - & $\mathbf{67.0}$ {\tiny $\pm {5.6}$}  
& \edit{$46.1$ {\tiny $\pm {2.6}$}} & \edit{$63.1$ {\tiny $\pm {4.1}$}} & - & \edit{$\mathbf{77.7}$ {\tiny $\pm {1.3}$}}  \\

\texttt{Libero Scene 2}  (7 tasks)  
& $21.5$ {\tiny $\pm {1.2}$} & $26.9$ {\tiny $\pm {1.0}$} & - & $\mathbf{42.0}$ {\tiny $\pm {2.2}$} 
& \edit{$23.9$ {\tiny $\pm {0.7}$}} & \edit{$29.0$ {\tiny $\pm {1.6}$}} & - & \edit{$\mathbf{49.5}$ {\tiny $\pm {2.9}$}}  \\

\texttt{Libero Scene 3} (4 tasks)  
& $47.7$ {\tiny $\pm {0.7}$} & $53.2$ {\tiny $\pm {2.3}$} & - & $\mathbf{63.3}$ {\tiny $\pm {5.1}$} 
&  \edit{$45.8$ {\tiny $\pm {3.3}$}} & \edit{$60.8$ {\tiny $\pm {1.9}$}} & - & \edit{$\mathbf{70.0}$ {\tiny $\pm {0.4}$}}  \\

\texttt{Libero All} { (16 tasks)}  
& $33.1$ {\tiny $\pm {0.5}$} & $43.1$ {\tiny $\pm {0.8}$} & - & $\mathbf{55.1}$ {\tiny $\pm {2.1}$} 
& \edit{$36.3$ {\tiny $\pm {0.3}$}} & \edit{$47.6$ {\tiny $\pm {2.8}$}} & - & \edit{$\mathbf{63.4}$ {\tiny $\pm {1.3}$}}  \\

\bottomrule
\end{tabular}
}
\caption{
Comparison of \edit{success rates} of pretrained policies and Best-of-$N$ sampling on \texttt{Robomimic} and \texttt{Libero}, for different pretraining approaches. \edit{Bolded text denotes best approach. 
}
}
\vspace{-1.0em}
\label{table:bon_robomimic}
\end{table*}

Our results from running \dsrl on \texttt{Robomimic} are given in \Cref{fig:dsrl_robomimic} and on \texttt{Libero} in \Cref{fig:dsrl_libero}. \edit{On \texttt{Robomimic}, \pbc significantly outperforms both baselines on \texttt{Square}, and achieves modest gains over \bc on \texttt{Lift} and \texttt{Can} (requiring roughly $2 \times$ fewer samples to achieve 75\% performance than \bc). For \texttt{Libero}, we run \dsrl on all tasks from \texttt{Kitchen Scene 2}. We see that \pbc pretraining leads to significant gains for \texttt{Libero}, enabling efficient RL finetuning in settings where both standard \bc pretraining and \nbc pretraining fail. Our results for \dppo are given in \Cref{fig:dsrl_robomimic} where we see that \pbc pretraining again leads to substantial gains on \texttt{Square} (again approximately $2 \times$ fewer samples to reach 75\% performance compared to \bc). This illustrates that \pbc still improves performance even for RL finetuning algorithms that modify the weights of the pretrained policy, and for which the resulting actions are therefore not explicitly constrained to the actions played by the pretrained policy.} 
Our Best-of-$N$ results are given in \Cref{table:bon_robomimic}. We see that across settings, \pbc-pretraining leads to consistent improvements over both \bc- and \nbc-pretrained policies for Best-of-$N$, \edit{and also consistently outperforms \textsc{ValueDICE}}. In particular, on \texttt{Libero}, \pbc improves by approximately 20-30\% over \bc, and 10-20\% over \nbc. 
\edit{We note as well that, even in the cases when \pbc does not yield substantial gains, it performs no worse than \bc.}
These results illustrate that \pbc enables more effective RL finetuning on both single-task settings (for each \texttt{Robomimic} task we pretrain a single policy) as well as multi-task settings (for \texttt{Libero} we pretrain a single policy across tasks), and across state- and image-based observations. \loose

\begin{wraptable}[12]{r}{0.6\textwidth}
\centering
\vspace{-1.0em}
\scalebox{0.75}
{
\begin{tabular}{l!{\vrule width 1pt}cccc}
\toprule
& \multicolumn{4}{c}{Pretrained Performance} \\
\texttt{Task} & BC & \nbc & \dice & \pbc  \\
\midrule

\texttt{Robomimic Lift} & $\textbf{71.0}$ {\tiny $\pm {0.5}$} & $68.0$ {\tiny $\pm {0.7}$} & $25.5$ {\tiny $\pm {4.4}$} & $\textbf{69.7}$ {\tiny $\pm {1.5}$} \\

\texttt{Robomimic Can} & $\mathbf{43.1}$ {\tiny $\pm {0.9}$} & $42.3$ {\tiny $\pm {0.8}$} & $14.2$ {\tiny $\pm {2.5}$} & $\mathbf{44.7}$ {\tiny $\pm {1.0}$} \\

\texttt{Robomimc Square} & $\mathbf{17.9}$ {\tiny $\pm {0.7}$} & $\mathbf{17.8}$ {\tiny $\pm {0.7}$} & $5.7$ {\tiny $\pm {0.3}$} & $\mathbf{18.1}$ {\tiny $\pm {0.8}$} \\

\texttt{Libero Scene 1 } (5 tasks) & $\mathbf{23.6}$ {\tiny $\pm {1.7}$} & $\mathbf{22.3}$ {\tiny $\pm {2.1}$} & - & $\mathbf{25.3}$ {\tiny $\pm {1.2}$}  \\
\texttt{Libero Scene 2 } (7 tasks)  & $\mathbf{11.4}$ {\tiny $\pm {0.2}$} & $\mathbf{10.6}$ {\tiny $\pm {0.8}$} & - & $\mathbf{12.3}$ {\tiny $\pm {1.5}$}  \\
\texttt{Libero Scene 3 }  (4 tasks) & $\mathbf{39.5}$ {\tiny $\pm {1.5}$} & $36.9$ {\tiny $\pm{1.8}$} & - & $\mathbf{40.8}$ {\tiny $\pm {1.3}$} \\
\texttt{Libero All }  (16 tasks) & $22.2$ {\tiny $\pm {0.3}$} & $20.9$ {\tiny $\pm {0.6}$} & - &$\mathbf{23.5}$ {\tiny $\pm {0.5}$}  \\

\bottomrule
\end{tabular}
}
\caption{
Comparison of \edit{success rates} of all pretrained policies on \texttt{Robomimic} and \texttt{Libero}, for different pretraining approaches. \edit{Bolded text denotes best approach.}
}
\vspace{-1.5em}
\label{table:pretrained_robomimic}
\end{wraptable}

\vspace{-0.25em}
\subsection{Posterior Behavioral Cloning Preserves Pretrained Performance}\label{sec:sim_bc_results}
\vspace{-0.25em}
We next show that \pbc pretraining produces a policy that has pretrained performance no worse than that of the \bc-pretrained policy.
 We provide results on pretrained performance for \texttt{Robomimic} and \texttt{Libero} in \Cref{table:pretrained_robomimic}. 
As these results illustrate, across both single-task and multi-task settings, the \pbc policy performs comparably to, or even better than, the \bc policy in terms of pretrained performance. 
In contrast, the \nbc policy often suffers from marginally lower performance than the \bc policy (while also underperforming the \pbc policy in terms of finetuned performance), and the \textsc{ValueDICE} policy significantly underperforms all \bc policies.
Combined with the results of \Cref{sec:sim_finetune_results}, this shows that \pbc not only enables more effective RL finetuning, it does this without hurting the performance of the pretrained policy.

\vspace{-0.25em}
\subsection{Posterior Behavioral Cloning Scales to Real-World Robotic Manipulation}\label{sec:widowx_results}
\vspace{-0.25em}

We next show that \pbc scales to real-world robotic settings, leading to  improvement in real-world RL finetuning over standard \bc pretraining. We evaluate on the WidowX 250 6-DoF robot arm and consider two tasks on the scene illustrated in \Cref{fig:widowx_visual}. We first collect 10 human teleoperation demonstrations for the task ``\texttt{Put corn in pot}'', where the objective is to pick up the corn and set it in the pot. We train diffusion policies with standard \bc as well as \pbc on these demonstrations. For RL finetuning we consider two tasks---the original ``\texttt{Put corn in pot}'' task the policy is trained on, and the task ``\texttt{Pick up banana}'' where the corn is replaced with a banana and the goal of the robot is to simply pick up the banana (see \Cref{fig:widowx_banana}).
For RL finetuning, we consider the Best-of-$N$ procedure outlined above, and utilize 100 rollouts per task.

Our results are given in \Cref{table:real}. We see that \pbc leads to significant improvements over \bc---in both tasks achieving significantly higher final success rate. In particular, for the ``\texttt{Put corn in pot}'' task, RL finetuning is only able to marginally improve the performance of the \bc policy---a 10\% improvement in success rate from the base policy---while the \pbc policy improves by 30\%. Furthermore, \pbc pretraining not only does not hurt the success rate of the pretrained policy, but actually leads to improved performance of the pretrained policy. This illustrates that \pbc scales to real robot settings, providing improved RL finetuning performance without decreasing pretrained policy performance.

\begin{figure*}[t]
\centering

\begin{minipage}[c]{0.28\textwidth}
  \centering
  \includegraphics[width=0.9\linewidth]{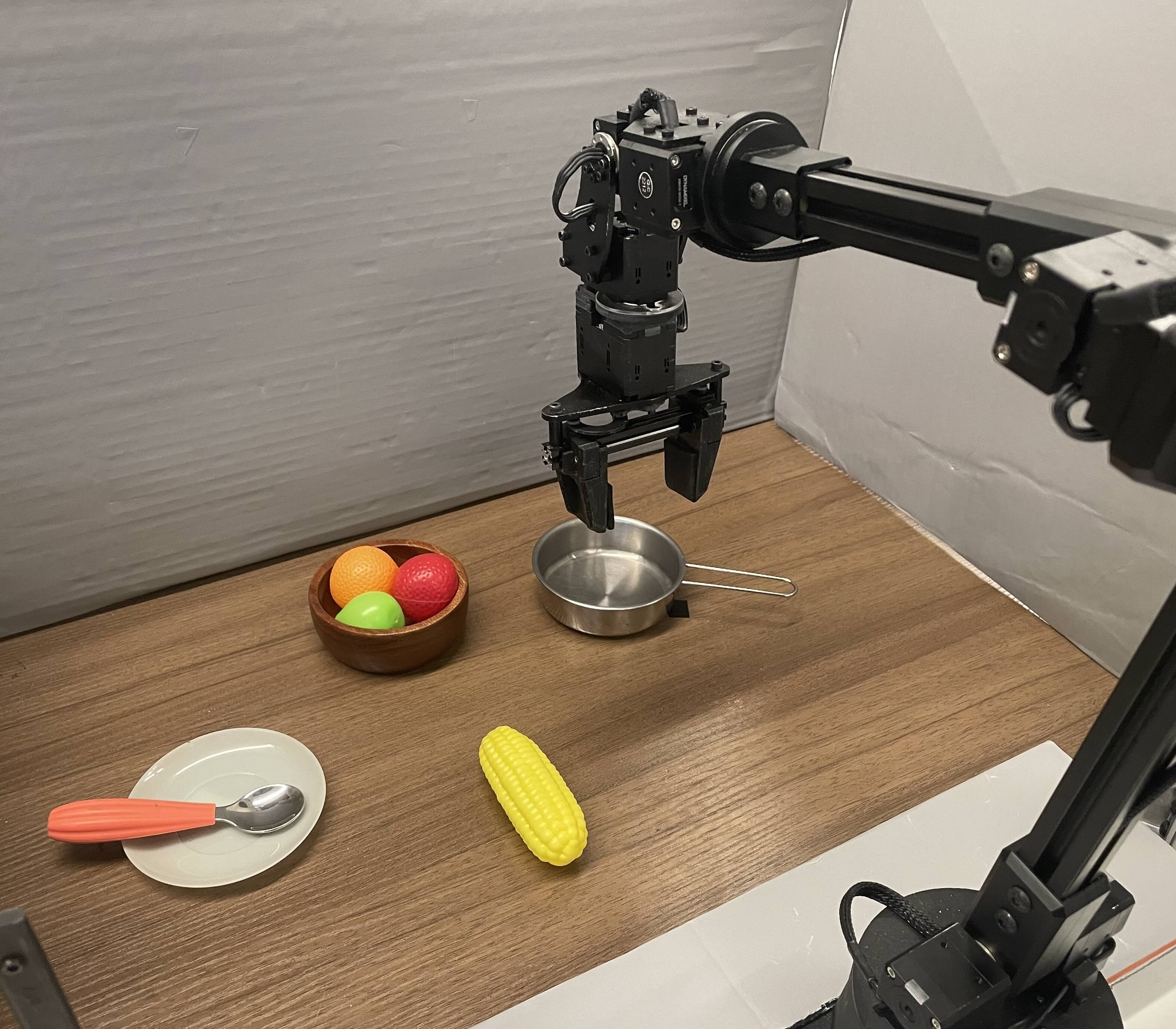}
  \caption{WidowX setup.}
  \label{fig:widowx_visual}
\end{minipage}
\hfill
\begin{minipage}[c]{0.68\textwidth}
  \centering
  \vspace*{\fill}
  \scalebox{0.8}{%
  \begin{tabular}{l!{\vrule width 1pt}cc|cc}
    \toprule
    & \multicolumn{2}{c|}{Pretrained Performance}
    & \multicolumn{2}{c}{Best-of-$N$ (100 Rollouts)} \\
    \texttt{Task} & \bc & \pbc & \bc & \pbc \\
    \midrule
    \texttt{Pick up banana}   & $2/20$ & $\mathbf{4/20}$ & $10/20$ & $\mathbf{16/20}$ \\
    \texttt{Put corn in pot}  & $3/20$ & $\mathbf{7/20}$ & $5/20$  & $\mathbf{13/20}$ \\
    \bottomrule
  \end{tabular}
  }
  \captionof{table}{Comparison of \bc and \pbc on real-world WidowX tasks using Best-of-$N$ sampling. We see that \pbc enables significantly larger improvement than \bc, while also improving the success rate of the pretrained policy.}
  \label{table:real}
  \vspace*{\fill}
\end{minipage}
\vspace{-1em}
\end{figure*}

\vspace{-0.25em}
\subsection{Understanding Posterior Behavioral Cloning}
\vspace{-0.25em}

Finally, we seek to provide insight into how \pbc improves RL finetuning performance. We first aim to disambiguate the role of the additional exploration a \pbc policy may provide over a \bc policy, versus the role that having access to a larger action distribution at test time might play. While these factors are intimately coupled for \dsrl and \dppo, for Best-of-$N$ sampling we can decouple them by selecting the rollout policy (the ``exploration'' policy) that collects data to learn the $Q$-function with \iql, and the policy whose actions we sample from and filter with the learned $Q$-function at test-time (the ``test-time'' policy).

\begin{wraptable}{r}{0.7\textwidth}
\centering
\vspace{-0.5em}
\scalebox{0.75}
{
\begin{tabular}{lcccc}
\toprule
\bc exploration  & \bc exploration   & \pbc exploration   & \pbc exploration  \\
  + \bc test-time & + \pbc test-time & + \bc test-time & + \pbc test-time \\
\midrule

 $68.1$ {\tiny $\pm {2.2}$} & $\mathbf{82.7}$ {\tiny $\pm {1.9}$} & $29.4$ {\tiny $\pm {3.6}$} & $\mathbf{81.3}$ {\tiny $\pm {5.7}$} \\

\bottomrule
\end{tabular}
}
\caption{
Best-of-$N$ sampling on \texttt{Robomimic Lift} with 2000 rollouts, varying the exploration policy and the test-time policy.
}
\vspace{-0.5em}
\label{table:ablation}
\end{wraptable}

We consider mixing the role of the \bc and \pbc policy on \texttt{Robomimic Lift} in this way, and provide our results in \Cref{table:ablation}. We find that using \pbc as the test-time policy is critical to achieving effective performance, but that this performance is achievable whether we use \bc or \pbc for the exploration policy.
This suggests that the utility of \pbc is primarily in its ability to provide a wider range of actions that can be sampled from the pretrained policy at test time, enabling RL finetuning approaches to easily select the best action.

Next we consider the qualitative behavior of the \pbc pretrained policy compared to the \bc and \nbc policies. We illustrate this on Libero task ``\texttt{Open the top drawer of the cabinet and put the bowl in it}'' and display a heatmap of the visitations for each policy in \Cref{fig:pbc_qualitative} (please see \Cref{sec:additional_ablations} for visualizations on additional tasks).  
We see that \pbc exhibits the widest distribution of states around the bowl, ensuring that it covers the behaviors necessary to reliably pick up the bowl, the most challenging aspect of this task. In contrast, \bc and \nbc exhibit less diversity around the bowl, which, if they do not cover all behaviors necessary to effectively pick up the bowl, may make them more difficult to finetune. At the same time, while exhibiting diversity around the states relevant to the task, \pbc focuses its behavior only on these relevant behaviors, ensuring the pretrained policy still performs effectively. In contrast, \nbc also interacts with the plate, which is irrelevant to this task and would therefore hurt its pretrained performance.

\begin{figure*}[t]
\centering

\noindent\makebox[\textwidth][c]{%
\begin{minipage}{0.48\textwidth}
    \centering

    \begin{subfigure}[t]{0.32\linewidth}
        \centering
        \includegraphics[width=\linewidth]{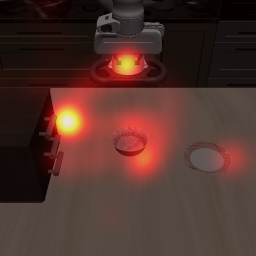}
        \caption{\bc}
        \label{fig:sub1}
    \end{subfigure}
    \hfill
    \begin{subfigure}[t]{0.32\linewidth}
        \centering
        \includegraphics[width=\linewidth]{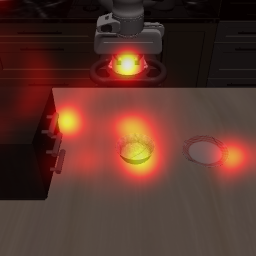}
        \caption{\nbc}
        \label{fig:sub2}
    \end{subfigure}
    \hfill
    \begin{subfigure}[t]{0.32\linewidth}
        \centering
        \includegraphics[width=\linewidth]{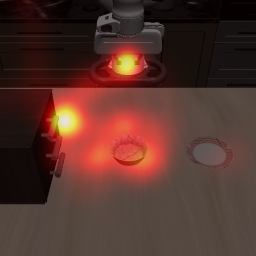}
        \caption{\pbc}
        \label{fig:sub3}
    \end{subfigure}

    \caption{Qualitative analysis of \pbc on Libero task ``\texttt{Open the top drawer of the cabinet and put the bowl in it}''.}
    \label{fig:pbc_qualitative}
\end{minipage}
\hfill

\begin{minipage}{0.24\textwidth}
    \centering
    \includegraphics[width=\linewidth]{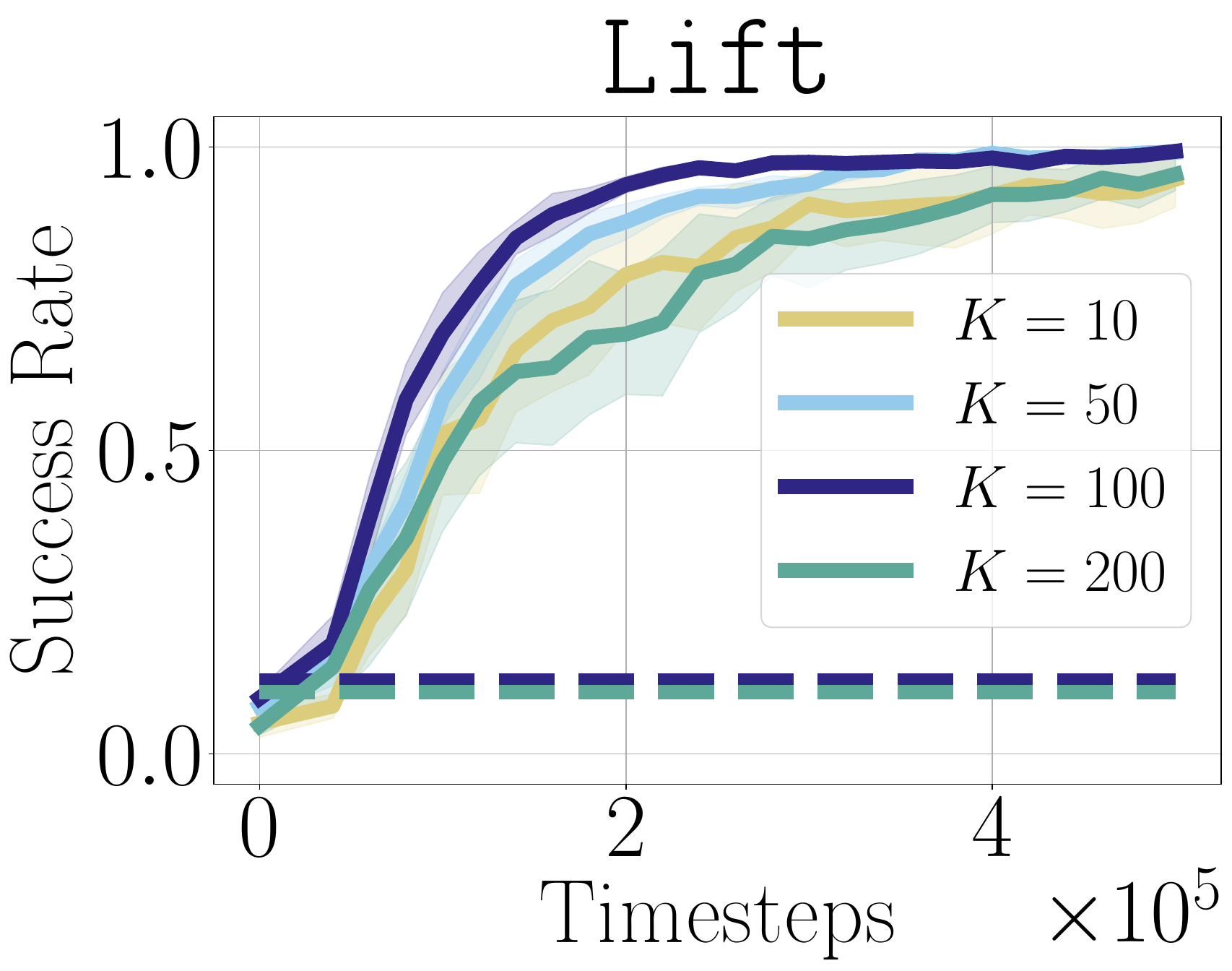}
    \caption{Sensitivity of \pbc with \dsrl finetuning to ensemble size. Dashed lines denote pretrained policy performance.}
    \label{fig:pbc_nen}
\end{minipage}
\hfill

\begin{minipage}{0.24\textwidth}
    \centering
    \includegraphics[width=0.84\linewidth]{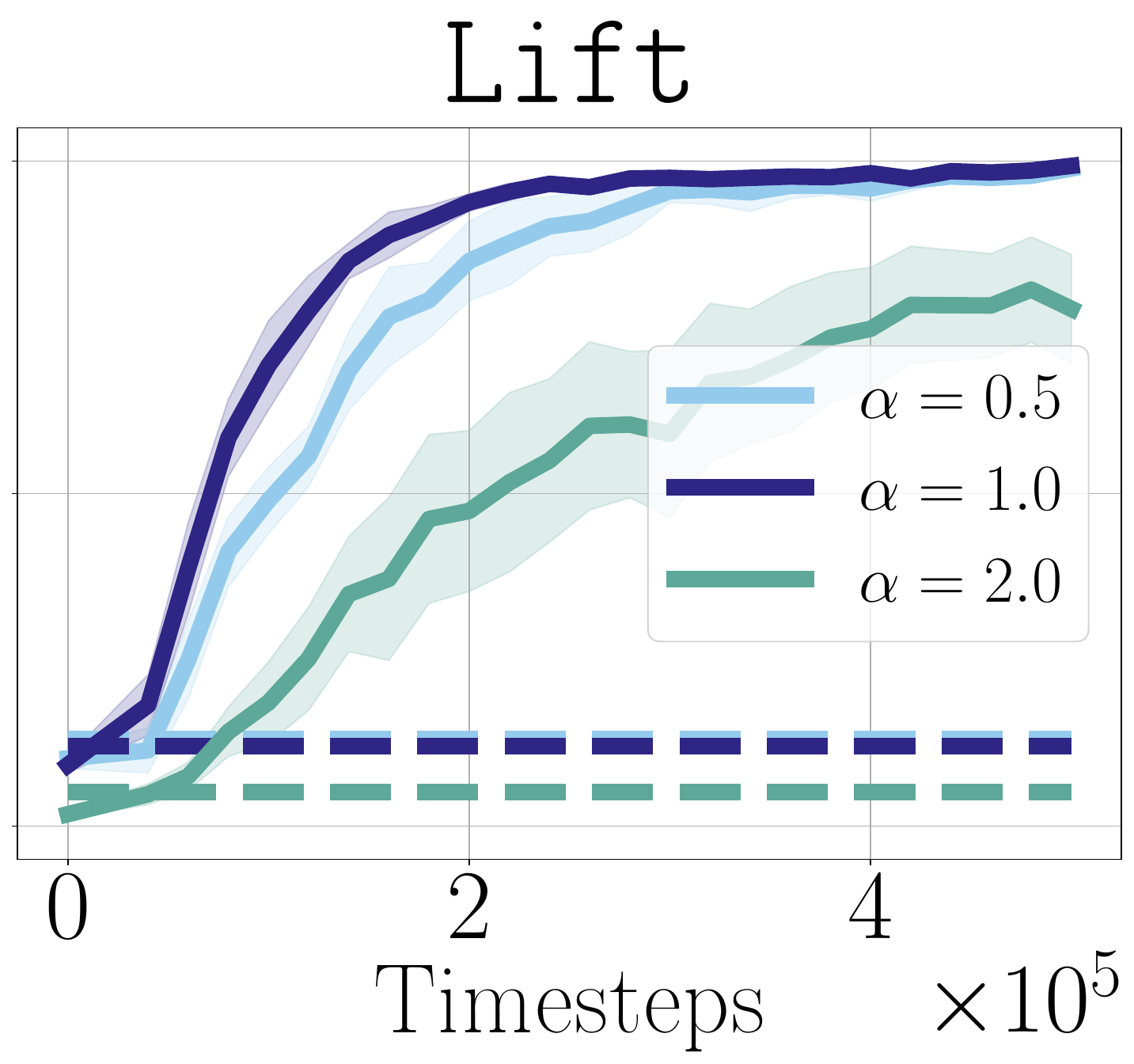}
    \caption{Sensitivity of \pbc with \dsrl finetuning to posterior weight. Dashed lines denote pretrained policy performance.}
    \label{fig:pbc_nw}
\end{minipage}
}
\vspace{-1em}
\end{figure*}

Finally, we consider the sensitivity of \pbc to two key hyperparameters: the size of the ensemble ($K$) and the weight of the posterior variance ($\alpha$ in \Cref{alg:posterior_bc}). We illustrate the performance of \pbc on \texttt{Robomimic Lift} varying these parameters in \Cref{fig:pbc_nen,fig:pbc_nw}. In \Cref{fig:pbc_nen} we see that \pbc performs best with a moderately sized ensemble ($K = 100$), but is not particularly sensitive to ensemble size as long as it is not too small or too large. In \Cref{fig:pbc_nw} we see that setting $\alpha$ too large can hurt the performance of \pbc, but that otherwise the performance of \pbc is relatively stable with respect to $\alpha$. We note as well that setting $\alpha$ too large  not only hurts the performance of the RL finetuning, but also causes the pretrained performance to drop. This is to be expected---even with the carefully tuned noise \pbc adds to the policy in pretraining, if the weight of this noise is too large the perturbations it induces will cause performance to drop below that of the \bc policy.
In general, throughout our experiments, we found that $\alpha = 1$ typically performs well.\loose

\vspace{-0.5em}
\section{Conclusion}
\vspace{-0.5em}
In this work, we have proposed a novel approach to pretraining policies from demonstrations that ensures the pretrained performance is no worse than that of the \bc policy, while expanding the action distribution to enable more effective RL finetuning. We have shown that this approach does indeed lead to improved RL finetuning performance in practice, scaling to real-world robotic settings. We believe this work motivates a variety of interesting questions for future work.
\begin{itemize}[leftmargin=*]
\item Our demonstrator action coverage condition introduced in \Cref{sec:act_coverage} is a \emph{necessary} condition, in some cases, for RL finetuning to reach the performance of the demonstrator policy, as \Cref{prop:bc_fails} shows. In general, however, demonstrator action coverage  does not give a guarantee about the sample complexity of the downstream RL finetuning. Can we derive a non-trivial \emph{sufficient} condition that ensures efficient RL finetuning without the aid of exploration approaches typically absent in practice (such as optimism), and how can we pretrain policies to ensure they meet such a sufficient condition? 
\item We have focused on pretraining only with supervised learning. While this is the most scalable approach, and the most commonly used approach in practice, is this a limiting factor in obtaining an effective initialization for online RL finetuning, and could we pretrain using other approaches as well (for example, offline RL)?
\item While we have primarily considered applications to robotic control, our approach could also be applied in language domains. Does pretraining (or SFT finetuning) of language models with our approach lead to improved performance in downstream RL finetuning?
\end{itemize}

\newpage
\section*{Acknowledgments}
This research was partly supported by RAI, ONR N00014-25-1-2060, and NSF IIS-2150826. The work of CF was partially supported by an NSF CAREER award.

\bibliography{iclr2026_conference}
\bibliographystyle{iclr2026_conference}

\newpage
\appendix

\newcommand{\pibetai}{\pi^{\beta,i}}
\newcommand{\TV}{\mathrm{TV}}
\newcommand{\Prob}{\mathbb{P}}
\newcommand{\KLb}{\mathrm{KL}}
\newcommand{\Dh}{D_{\mathrm{H}}}
\newcommand{\pist}{\pi^\star}
\newcommand{\stil}{\widetilde{s}}
\newcommand{\iT}{i_T}
\newcommand{\rtil}{\widetilde{r}}
\newcommand{\pitil}{\widetilde{\pi}}

\edit{\section{Additional Related Work}\label{sec:additional_related}

\textbf{Other approaches for pretraining from demonstrations.}
While our primary focus is on behavioral cloning (as noted, the workhorse of most modern applications) other approaches to pretraining from demonstrations exist. \bc is only one possible instantiation of \emph{imitation learning}; other approaches to imitation learning include inverse RL \citep{ng2000algorithms,abbeel2004apprenticeship,ziebart2008maximum}, methods that aim to learn a policy matching the state distribution of the demonstrator, such as adversarial imitation learning \citep{ho2016generative,kostrikov2018discriminator,fu2017learning,kostrikov2019imitation,ni2021f,garg2021iq,xu2022discriminator,li2023imitation,yue2024ollie}, and robust imitation learning \citep{chae2022robust,desai2020imitation,tangkaratt2020robust,wang2021robust,giammarino2025visually}.
The majority of these works, however, either assume access to additional data sources (e.g. suboptimal trajectories), or require online environment access and are therefore not truly offline pretraining approaches, which is the focus of this work. Furthermore, none of these works explicitly consider the role of pretraining in enabling efficient RL finetuning.\loose

Meta-learning directly aims learn an initialization that can be quickly adapted to a new task. While instantiations of meta-learning for 
imitation learning exist \citep{duan2017one,finn2017one,james2018task,dasari2021transformers,gao2023transferring}, our setting differs fundamentally from the meta-imitation learning setting. Meta-imitation learning assumes access to demonstration data from \emph{more than one task}, and attempts to learn an initialization that will allow for quickly adapting to demonstrations from a \emph{new} task. In contrast, our goal is to obtain an approach able to learn on a \emph{single} task
(though we also consider the multi-task setting), and we aim to find an initialization that allows for improvement on the \emph{same} task, while preserving pretrained performance on this task. Furthermore, rather than learning from new \emph{demonstrations}, as meta-imitation learning does, we aim to learn from (potentially suboptimal) data collected online and that is labeled with rewards.\loose

\paragraph{Reinforcement learning-based pretraining.}
In the RL literature, two lines of work bear some resemblance to ours as well. The \emph{offline-to-online RL} setting aims to train policies with RL on offline datasets that can then be improved with further online interaction \citep{lee2022offline,ghosh2022offline,kumar2022pre,zhang2023policy,uchendu2023jump,zheng2023adaptive,ball2023efficient,nakamoto2023cal}, and the \emph{meta-RL} setting aims to meta-learn a policy on some set of tasks which can then be quickly adapted to a new task \citep{wang2016learning,duan2016rl,finn2017model,finn2018probabilistic}. 
While similar to our work in that these works also aim to learn behaviors that can be efficiently improved online, the settings differ significantly in that the offline- or meta-pretraining typically requires reward labels (rather than unlabeled demonstrations) and are performed with RL (rather than BC)---in contrast, we study how BC-like pretraining (as noted, the workhorse of most modern applications) can enable efficient online adaptation.
}

\section{Proofs}

\subsection{BC Policy Fails to Cover Demonstrator Actions}

\begin{proof}[Proof of \Cref{prop:bc_fails}]
Let $\cM^1$ and $\cM^2$ denote multi-armed bandits with 3 arms and reward functions $r^1$ and $r^2$:
\begin{align*}
& r^1(a_1) = 0, r^1(a_2) = 1, r^1(a_3) = 0 \\
& r^2(a_1) = 0, r^2(a_2) = 0, r^2(a_3) = 1.
\end{align*}
Let $\pibeta(a_1) = 1 - 4 \epsilon$, $\pibeta(a_2) = 2\epsilon$, $\pibeta(a_3) = 2\epsilon$.

By construction of $\pihatbeta$, if $T(a_2) = 0$ then we will have $\pihatbeta(a_2) = 0$, and if $T(a_3) = 0$ we will have $\pihatbeta(a_3) = 0$.
By the definition of both $\cM^1$ and $\cM^2$, we have
\begin{align*}
    \Prob^{\cM^i}[T(a_2) = 0, T(a_3) = 0] = (1 - 4 \epsilon)^T.
\end{align*}
As we have assumed that $T \le \frac{1}{20\epsilon}$ and $\epsilon \in (0, 1/8]$, some calculation shows that we can lower bound this as $1/2$. Note that for both $\cM^1$ and $\cM^2$, we have $\cJ(\pibeta) =  2\epsilon$, while for policies $\pihatbeta$ that only play $a_1$, we have $\cJ(\pihatbeta) = 0$. This proves the first part of the result.

For the second part, note that the optimal policy on $\cM^1$ plays only $a_2$ and has expected reward of 1, while the optimal policy on $\cM^2$ plays only $a_2$ and has expected reward of 1. Let $\pihat$ denote an estimate of the optimal policy and $\Exp^{\cM^i, \pihatbeta}[\cdot]$ the expectation induced by playing the policy $\pihatbeta$ from the first part on instance $\cM^i$. Then:
\begin{align*}
    \min_{\pihat} \max_{i \in \{1, 2 \}} \Exp^{\cM^i, \pihatbeta}[\max_\pi \cJ^{\cM^i}(\pi) - \cJ^{\cM^i}(\pihat)] & = \min_{\pihat} \max_{i \in \{1, 2 \}} \Exp^{\cM^i, \pihatbeta}[1 - \pihat(a_{1 + i})].
\end{align*}
Note that $1 - \pihat(a_{2}) = \pihat(a_1) + \pihat(a_3) \ge \pihat(a_3)$. Thus we can lower bound the above as
\begin{align*}
    & \ge \min_{\pihat} \max \{ \Exp^{\cM^1, \pihatbeta}[\pihat(a_3)], \Exp^{\cM^2, \pihatbeta}[1 - \pihat(a_3)] \} \\
    & \ge \min_{\pihat} \frac{1}{2} \left ( \Exp^{\cM^1, \pihatbeta}[\pihat(a_3)] + \Exp^{\cM^2, \pihatbeta}[1 - \pihat(a_3)] \right )\\
    & \ge \frac{1}{2} - \frac{1}{2} \min_{\pihat} \left |  \Exp^{\cM^1, \pihatbeta}[\pihat(a_3)] - \Exp^{\cM^2, \pihatbeta}[\pihat(a_3)]\right |.
\end{align*}
We can bound
\begin{align*}
\left   |  \Exp^{\cM^1, \pihatbeta}[\pihat(a_3)] - \Exp^{\cM^2, \pihatbeta}[\pihat(a_3)]\right |  \le \TV(\Prob^{\cM^1,\pihatbeta}, \Prob^{\cM^2,\pihatbeta}).
\end{align*}
Since $\cM^1$ and $\cM^2$ only differ on $a_2$ and $a_3$, and since $\pihatbeta(a_2) = \pihatbeta(a_3) = 0$, we have $\TV(\Prob^{\cM^1,\pihatbeta}, \Prob^{\cM^2,\pihatbeta})= 0$. Thus, we conclude that
\begin{align*}
    \min_{\pihat} \max_{i \in \{1, 2 \}} \Exp^{\cM^i, \pihatbeta}[\max_\pi \cJ^{\cM^i}(\pi) - \cJ^{\cM^i}(\pihat)] \ge \frac{1}{2}.
\end{align*}
This proves the second part of the result.

\end{proof}

\subsection{Uniform Noise Fails}

\begin{proof}[Proof of \Cref{prop:unif_fails}]
\textbf{Construction.}
Let $\cM$ be the MDP with state space $\{ \stil_1, \ldots, \stil_k, s_1, s_2 \}$, actions $\{ a_1, a_2 \}$, horizon $H \ge 2$ with initial state distribution:
\begin{align*}
P_0(s_1) = 1/2, \quad P_0(\stil_1) = 2^{-2} + 2^{-k},  \quad P_0(\stil_i) = 2^{-i-1}, i \ge 2,
\end{align*}
transition function, for all $h \in [H]$:
\begin{align*}
& P_h(\stil_i \mid \stil_i, a) = 1, \forall a \in \cA, \quad P_h(s_1 \mid s_1, a_1) = 1, \\
& P_h(s_2 \mid s_1, a_2) = 1, \quad P_h(s_2 \mid s_2, a) =1, \forall a \in \cA,
\end{align*}
and reward that is 0 everywhere except
\begin{align*}
r_1(\stil_i, a_1) = r_H(s_1,a_1) = 1, \quad r_1(\stil_i, a_2) = 1 - 2\Delta,
\end{align*}
for some $\Delta > 0$ to be specified.
We consider $\pibeta$ defined as
\begin{align*}
\pibeta_h(a_1 \mid \stil_i) = \pibeta_h(a_2 \mid \stil_i) = \frac{1}{2}, \quad \pibeta_h(a_1 \mid s_1) = 1.
\end{align*}
Let $\epsilon := \frac{H^2 S \log T}{T} + \xi$, and set $\Delta \leftarrow 2\epsilon$.

\paragraph{Upper bound on $\alpha$.}
Note that $\cJ(\pibeta) = 1 - \frac{1}{2} \Delta$, and that the value of the optimal policy $\pist$ is $\cJ(\pist) = \max_\pi \cJ(\pi) = 1$. Let $\pitilunif$ denote the policy that, on all $\stil_i$ plays $\pist$, and on other states plays $\pist$ with probability $1-\alpha$, and otherwise plays $\unif(\cA)$. Note then that, regardless of the value of $\pihatbeta$, we have that $\cJ(\pitilunif) \ge \cJ(\piunif)$. Thus,
\begin{align*}
\cJ(\pibeta) - \Exp[\cJ(\piunif)] \ge \cJ(\pibeta) - \cJ(\pitilunif)
\end{align*}
If we are in $s_1$ at $h=2$, the only way we can receive any reward on the episode is if we take action $a_1$ for the last $H-1$ steps, and we then receive a reward of $1$.
Under $\pitilunif$, we take $a_1$ at each step with probability $1- \alpha + \alpha/A$, so our probability of getting a reward of $1$ is $(1 - \alpha + \alpha/A)^{H-1}$. Note that in contrast $\pibeta$ will always play $a_1$ and receive a reward of 1 in this situation.
If we are in $\stil_i$ at $h=2$ for any $i$, then $\pibeta$ will incur a loss of $\Delta$ more than $\pitilunif$. 
Thus, we can lower bound
\begin{align*}
\cJ(\pibeta) - \cJ(\pitilunif) \ge -\frac{1}{2}\Delta + \frac{1}{2} \cdot ( 1 - (1 - \alpha + \alpha/A)^{H-1})
\end{align*}
By assumption we have that $\frac{1}{2}\Delta = \epsilon$. Thus, if we want $\cJ(\pibeta) - \Exp[\cJ(\piunif)] \le \epsilon$, we need
\begin{align*}
\frac{1}{2} \cdot ( 1 - (1 - \alpha + \alpha/A)^{H-1}) \le 2 \epsilon.
\end{align*}
Rearranging this, we have
\begin{align*}
1 - 4 \epsilon \le (1 - \alpha + \alpha/A)^{H-1} \iff \frac{1}{H-1} \log \left ( 1 - 4\epsilon \right ) \le \log( 1 - \alpha + \alpha/A).
\end{align*}
From the Taylor decomposition of $\log(1-x)$, we see that $\log(1-\alpha + \alpha/A) \le -(1-1/A) \alpha$. Furthermore, we can lower bound
\begin{align*}
 \log(1 - 4\epsilon ) \ge -8 \epsilon
\end{align*}
as long as $\epsilon \le 1/2$. Altogether, then, we have
\begin{align*}
\frac{-8\epsilon}{H-1} \le -(1-1/A) \alpha \implies \alpha \le \frac{8 \epsilon}{(H-1)(1-1/A)} \implies \alpha \le 32 \epsilon
\end{align*}
where the last inequality follows since $H \ge 2, A = 2$.

\paragraph{Upper bound on $\gamma$.}
Let $\iT := \argmax_i \{ 2^{-i-1} \mid 2^{-i -1 } \le 1/T \}$, so that $1/2T \le P_0(\stil_{\iT}) \le 1/T$, and note that such an $\stil_{\iT}$ exists by construction. Let $\cE$ be the event $\cE := \{ T_1(\stil_{\iT}) = T_1(\stil_{\iT}, a_2) = 1 \}$. 
We have
\begin{align*}
\Prob[\cE] & = \Prob[T_1(\stil_{\iT}, a_2) = 1 \mid T_1(\stil_{\iT}) = 1] \Prob[T_1(\stil_{\iT}) = 1] \\
& = \frac{1}{2} \cdot T P_0(\stil_{\iT}) (1 - P_0(\stil_{\iT}) )^{T-1} \\
& = \frac{1}{2} \cdot T \cdot \frac{1}{2T} \cdot (1 - \frac{1}{T} )^{T-1} \\
& \ge \frac{1}{4 e}.
\end{align*}
Note that on the event $\cE$, we have $\pihatbeta_1(a_1 \mid \stil_{\iT}) = 0$, but $\pibeta_1(a_1 \mid \stil_{\iT}) = 1/2$. 
Thus,
\begin{align*}
\piunif_1(a_1 \mid \stil_{\iT}) = \alpha/A \le 32 \epsilon/A = 64 \epsilon/A \cdot \pibeta_1(a_1 \mid \stil_{\iT})
\end{align*}
where we have used the bound on $\alpha$ shown above. Thus, on $\cE$, we will only have that $\piunif$ achieves demonstrator action coverage for $\gamma \le 64 \epsilon/A$. Since $\cE$ occurs with probability at least $1/4e$, it follows that if we want to guarantee $\piunif$ achieves demonstrator action coverage with probability at least $1-\delta$ for $\delta < 1/4e$, we must have $\gamma \le 64 \epsilon / A$.

Note as well that, since $\pihatbeta_1(a_2 \mid \stil_{\iT}) = 1$, any policy in the support of $\pihatbeta$ will be suboptimal by a factor of at least $P_0(\stil_{\iT}) \cdot 2\Delta \ge \Delta/T$.
\end{proof}

\subsection{Analysis of Posterior Demonstrator Policy}\label{sec:post_analysis}
Throughout this section we denote
\begin{align*}
\pitil_h(a \mid s) := \begin{cases} (1-\alpha) \cdot \frac{T_h(s,a)}{T_h(s)} + \alpha \cdot \frac{T_h(s,a) + \lambda/A}{T_h(s) + \lambda} & T_h(s) > 0 \\
\unif(\cA) & T_h(s) = 0
\end{cases}
\end{align*}
for some $\alpha \in [0,1]$.

We also denote $w_h^\pi(s,a) := \Prob^\pi[s_h = s, a_h = a]$. $Q_h^\pi(s,a) := \Exp^\pi[\sum_{h' \ge h} r_{h'}(s_{h'}, a_{h'}) \mid s_h = s, a_h = a]$ denotes the standard $Q$-function. $\cJ(\pi; r)$ denotes the expected return of policy $\pi$ for reward $r$.

\begin{lemma}\label{lem:post_sampler_guarantee}
As long as $\delta \le 0.9$ and $\lambda \ge A$, we have
\begin{align*}
\Prob \left [ \pitil_h(a \mid s) \ge \alpha \cdot \min \left \{ \frac{\pibeta_h(a \mid s)}{64\log SH/\delta}, \frac{1}{2\lambda} \right \} , \forall a \in \cA, s \in \cS, h \in [H] \right ] \ge 1-\delta.
\end{align*}
\end{lemma}
\begin{proof}
Consider some $(s,h)$. By Bernstein's inequality, if $T_h(s) > 0$, we have that with probability at least $1-\delta$,
\begin{align}\label{eq:pipost_cov_lb1}
\frac{T_h(s,a)}{T_h(s)} \ge \pibeta_h(a \mid s) - \sqrt{\frac{2\pibeta_h(a \mid s) \log 1/\delta}{T_h(s)}} - \frac{2\log 1/\delta}{3T_h(s)} .
\end{align}
From some algebra, we see that as long as $T_h(s) \ge \frac{32\log 1/\delta}{\pibeta_h(a \mid s)}$, we have that $\frac{T_h(s,a)}{T_h(s)} \ge \frac{1}{2} \pibeta_h(a \mid s)$. 
By the definition of $\pitil$, under the good event of \eqref{eq:pipost_cov_lb1} we can then lower bound
\begin{align*}
\pitil_h(a \mid s)  & \ge \begin{cases}
\frac{\alpha}{1 + \lambda/T_h(s)} \cdot \frac{1}{2} \pibeta_h(a \mid s) & T_h(s) \ge \frac{32\log 1/\delta}{\pibeta_h(a \mid s)} \\
\frac{\alpha \lambda/A}{T_h(s) + A} & \text{o.w.}
\end{cases} \\
& \ge \begin{cases}
\frac{\alpha \cdot 32 \log 1/\delta}{32 \log 1/\delta  + \lambda \cdot \pibeta_h(a \mid s)} \cdot \frac{1}{2} \pibeta_h(a \mid s) & N_h(s) \ge \frac{32\log 1/\delta}{\pibeta_h(a \mid s)} \\
\frac{\alpha \lambda/A \cdot \pibeta_h(a \mid s)}{32\log 1/\delta + \lambda \cdot \pibeta_h(a \mid s)} & \text{o.w.}
\end{cases} \\
& \overset{(a)}{\ge} \frac{\alpha \cdot \pibeta_h(a \mid s)}{32\log 1/\delta + \lambda \cdot \pibeta_h(a \mid s)} \\
& \ge \alpha \cdot \min \left \{ \frac{\pibeta_h(a \mid s)}{64\log 1/\delta}, \frac{1}{2\lambda} \right \}
\end{align*}
where $(a)$ follows as long as $\delta \le 0.9$ and $\lambda \ge A$.
In the case when $T_h(s) = 0$ we have $\pitil_h(a \mid s) = 1/A \ge 1/\lambda$, so this lower bound still holds. 
Taking a union bound over arms proves the result.
\end{proof}

\begin{lemma}\label{lem:post_subopt}
As long as $\lambda \ge 4\log(HT)$, we have
\begin{align*}
\Exp[\cJ(\pihatbeta) - \cJ(\pitil)] \lesssim (1 + \alpha H) \cdot \frac{H^2 S \log T}{T} + \alpha \cdot \frac{H^2 S \lambda}{T}.
\end{align*}
\end{lemma}
\begin{proof}
By the Performance-Difference Lemma we have:
\begin{align}
\cJ(\pihatbeta) - \cJ(\pitil) & =  \sum_{h=1}^H \sum_{s \in \cS} w_h^{\pihatbeta}(s) \cdot \left ( \Exp_{a \sim \pihatbeta_h(s)}[Q_h^{\pitil}(s, a)] - \Exp_{a \sim \pitil_h(s)}[Q_h^{\pitil}(s, a)] \right )  \nonumber \\
& \le  \sum_{h=1}^H \sum_{s \in \cS} w_h^{\pihatbeta}(s) \cdot \left | \Exp_{a \sim \pihatbeta_h(s)}[Q_h^{\pitil}(s, a)] - \Exp_{a \sim \pitil_h(s)}[Q_h^{\pitil}(s, a)] \right | . \label{eq:post_regret_decomp1}
\end{align}
For $(s,h)$ with $N_h(s) > 0$, we have
\begin{align*}
 \left | \Exp_{a \sim \pihatbeta_h(s)}[Q_h^{\pitil}(s, a)] - \Exp_{a \sim \pitil_h(s)}[Q_h^{\pitil}(s, a)] \right | & \le \sum_{a \in \cA} H\cdot | \pihatbeta_h(a \mid s) - \pitil_h(a \mid s) |,
\end{align*}
where we have used that $Q_h^{\pipost}(s, a) \in [0,H]$. Then, using the definition of $\pihatbeta$ and $\pitil$ we can bound this as
\begin{align*}
& \le \sum_{a \in \cA} \alpha H \cdot \left | \frac{T_h(s,a)}{T_h(s)} - \frac{T_h(s,a) + \lambda/A}{T_h(s) + \lambda} \right | \\
& = \sum_{a \in \cA} \frac{\alpha \lambda H}{A} \cdot \left | \frac{AT_h(s,a) - T_h(s)}{T_h(s) (T_h(s) + \lambda)} \right | \\
& \le \sum_{a \in \cA}  \frac{\alpha \lambda H}{A} \cdot \frac{A T_h(s,a) + T_h(s)}{T_h(s)(T_h(s) + \lambda)} \\
& = \frac{2 \alpha \lambda H}{T_h(s) + \lambda}.
\end{align*}
Since $\Exp_{a \sim \pihatbeta_h(s)}[Q_h^{\pitil}(s, a)] - \Exp_{a \sim \pitil_h(s)}[Q_h^{\pitil}(s, a)] = 0$ by construction when $T_h(s) = 0$, we then have 
\begin{align*}
\eqref{eq:post_regret_decomp1} \le  \sum_{h=1}^H \sum_{s \in \cS} w_h^{\pihatbeta}(s) \cdot \frac{2 \alpha \lambda H}{T_h(s) + \lambda}.
\end{align*}
Let $\cE$ denote the good event from \Cref{lem:count_concentrate} with $\delta = \frac{S}{T}$. Then as long as $\lambda \ge 4 \log (HT)$ we can bound the above as
\begin{align*}
& \le \sum_{h=1}^H \sum_{s \in \cS} w_h^{\pihatbeta}(s) \cdot \frac{2 \alpha \lambda H}{T_h(s) + \lambda} \bbI \{ \cE \} + 2H^2 \cdot \bbI \{ \cE^c \} \\
& \le \sum_{h=1}^H \sum_{s \in \cS} w_h^{\pihatbeta}(s) \cdot \frac{4\alpha \lambda H}{w_h^{\pibeta}(s) \cdot T + \lambda}  + 2H^2 \cdot \bbI \{ \cE^c \}.
\end{align*}
Let $\rtil$ denote the reward function:
\begin{align*}
\rtil_h(s,a) := \frac{ \lambda}{w_h^{\pibeta}(s) \cdot T + \lambda}
\end{align*}
and note that $\rtil \in [0,1]$, and
\begin{align*}
\sum_{h=1}^H \sum_{s \in \cS} w_h^{\pihatbeta}(s) \cdot \frac{4\alpha \lambda H}{w_h^{\pibeta}(s) \cdot T + \lambda} = 4\alpha H \cdot \cJ(\pihatbeta; \rtil).
\end{align*}
By Theorem 4.4 of \cite{rajaraman2020toward}, we have\footnote{Note that Theorem 4.4 of \cite{rajaraman2020toward} shows an inequality in the opposite direction of what we show here: they bound $\cJ(\pibeta;\rtil) - \Exp[\cJ(\pihatbeta; \rtil)]$ instead of $\Exp[\cJ(\pihatbeta; \rtil)] - \cJ(\pibeta;\rtil)$. However, we see that the only place in their proof where their argument relied on this ordering is in Lemma A.8. We show in \Cref{lem:reverse_nived} that a reverse version of their Lemma A.8 holds, allowing us to instead bound $\Exp[\cJ(\pihatbeta; \rtil)] - \cJ(\pibeta;\rtil)$.}
\begin{align*}
\Exp[\cJ(\pihatbeta; \rtil)] & \lesssim \cJ(\pibeta; \rtil) + \frac{H^2 S \log T}{T} \\
& = \sum_{h=1}^H \sum_{s \in \cS} w_h^{\pibeta}(s) \cdot \frac{ \lambda}{w_h^{\pibeta}(s) \cdot T + \lambda} + \frac{H^2 S \log T}{T}\\
& \le \frac{HS \lambda}{T} +  \frac{H^2 S \log T}{T}.
\end{align*}
Noting that $\Exp[2H^2 \cdot \bbI \{ \cE^c \}] \le 2H^2 \delta \le \frac{2H^2 S}{T} $ completes the proof. 
\end{proof}

\begin{lemma}\label{lem:count_concentrate}
With probability at least $1-\delta$, for all $(s,h)$, we have
\begin{align*}
T_h(s) + \lambda \ge \frac{1}{2} w_h^{\pibeta}(s) \cdot T + \frac{1}{2} \lambda
\end{align*}
as long as $\lambda \ge 4 \log \frac{SH}{\delta}$.
\end{lemma}
\begin{proof}
Consider some $(s,h)$ and note that $\Exp[T_h(s) / T] = w_h^{\pibeta}(s)$. By Bernstein's inequality, we have with probability $1-\delta/SH$:
\begin{align*}
T_h(s) \ge w_h^{\pibeta}(s) \cdot T -  \sqrt{2 w_h^{\pibeta}(s) \cdot T \cdot \log \frac{SH}{\delta}} - \frac{2}{3} \log \frac{SH}{\delta}.
\end{align*}
We would then like to show that
\begin{align*}
& w_h^{\pibeta}(s) \cdot T -  \sqrt{2 w_h^{\pibeta}(s) \cdot T \cdot \log \frac{SH}{\delta}} - \frac{2}{3} \log \frac{SH}{\delta} + \lambda \ge \frac{1}{2} (w_h^{\pibeta}(s) \cdot T + \lambda) \\
& \iff \frac{1}{2} w_h^{\pibeta}(s) \cdot T + \frac{1}{2} \lambda \ge  \sqrt{2 w_h^{\pibeta}(s) \cdot T \cdot \log \frac{SH}{\delta}} + \frac{2}{3} \log \frac{SH}{\delta}
\end{align*}
As we have assumed $\lambda \ge 4 \log \frac{SH}{\delta}$, it suffices to show
\begin{align*}
\frac{1}{2} w_h^{\pibeta}(s) \cdot T +  \log \frac{SH}{\delta} \ge  \sqrt{2 w_h^{\pibeta}(s) \cdot T \cdot \log \frac{SH}{\delta}} .
\end{align*}
However, this is true by the AM-GM inequality. A union bound proves the result.
\end{proof}

\newcommand{\Prb}{\mathrm{Pr}}
\newcommand{\pifirst}{\pi^{\mathrm{first}}}
\newcommand{\piorcfirst}{\pi^{\mathrm{orc-first}}}

\begin{lemma}[Reversed version of Lemma A.8 of \cite{rajaraman2020toward}]\label{lem:reverse_nived}
Adopting the notation from \cite{rajaraman2020toward}, we have
\begin{align*}
\Exp[\Prb_{\pifirst}[\cE]] \le \frac{SH\log N}{N}
\end{align*}
for $\cE^c$ the event that within a trajectory, the policy only visits states for which $T_h(s) > 0$. 
\end{lemma}
\begin{proof}
Let $\cE_{s,h}$ denote the event that the state $s$ is visited at step $h$ and $T_h(s) = 0$, and $\cE_h := \cup_{s \in \cS} \cE_{s,h}$. Then, by simple set inclusions, we have:
\begin{align*}
\cE & = \bigcup_{h \in [H]} \bigcup_{s \in \cS} \cE_{s,h} = \bigcup_{h \in [H]} \bigcup_{s \in \cS} \bigg ( \cE_{s,h} \cap \bigcap_{h' < h} \cE_{h'}^c \bigg ).
\end{align*} 
By a union bound it follows that
\begin{align*}
\Exp[\Prb_{\pifirst}[\cE]] & \le \sum_{h \in [H]} \sum_{s \in \cS} \Exp[\Prb_{\pifirst}[\cE_{s,h} \cap \bigcap_{h' < h} \cE_{h'}^c ]] .
\end{align*}
Now note that
\begin{align*}
\Prb_{\pifirst}[\cE_{s,h} \cap \bigcap_{h' < h} \cE_{h'}^c ] & = \Prb_{\pifirst}[\cE_{s,h} \mid \bigcap_{h' < h} \cE_{h'}^c ] \Prb_{\pifirst}[ \bigcap_{h' < h} \cE_{h'}^c ] \\
& = \Prb_{\pifirst}[\cE_{s,h} \mid \bigcap_{h' < h} \cE_{h'}^c ] \Prb_{\pifirst}[ \cE_{h-1}^c \mid \bigcap_{h' < h-1} \cE_{h'}^c ] \Prb_{\pifirst}[ \bigcap_{h' < h-1} \cE_{h'}^c ] \\
& \vdots \\
& =  \Prb_{\pifirst}[\cE_{s,h} \mid \bigcap_{h' < h} \cE_{h'}^c ]  \cdot \prod_{h' < h} \Prb_{\pifirst}[ \cE^c_{h'} \mid \bigcap_{h'' < h'} \cE^c_{h''}] .
\end{align*}
If the event $\bigcap_{h' < h} \cE_{h'}^c$ holds, then up to step $h$ no states are encountered for which $T_{h'}(s) = 0$. Thus, on such states, $\pifirst$ and $\piorcfirst$ will behave identically. It follows that $\Exp[\Prb_{\pifirst}[\cE_{s,h} \mid \bigcap_{h' < h} \cE_{h'}^c ]] = \Exp[\Prb_{\piorcfirst}[\cE_{s,h} \mid \bigcap_{h' < h} \cE_{h'}^c ]]$. By a similar argument, we have $\Prb_{\piorcfirst}[ \cE^c_{h'} \mid \bigcap_{h'' < h'} \cE^c_{h''}] = \Prb_{\pifirst}[ \cE^c_{h'} \mid \bigcap_{h'' < h'} \cE^c_{h''}]$ for each $h' < h$.
Thus, 
\begin{align*}
\Prb_{\pifirst}[\cE_{s,h} \cap \bigcap_{h' < h} \cE_{h'}^c ] = \Prb_{\piorcfirst}[\cE_{s,h} \cap \bigcap_{h' < h} \cE_{h'}^c ].
\end{align*}
It follows that
\begin{align*}
\Exp[\Prb_{\pifirst}[\cE]] & \le \sum_{h \in [H]} \sum_{s \in \cS} \Exp[\Prb_{\piorcfirst}[\cE_{s,h} \cap \bigcap_{h' < h} \cE_{h'}^c ]] \le \sum_{h \in [H]} \sum_{s \in \cS} \Exp[\Prb_{\piorcfirst}[\cE_{s,h} ]] .
\end{align*}
From here the proof follows identically to the proof of Lemma A.8 of \cite{rajaraman2020toward}.
\end{proof}

\begin{proof}[Proof of \Cref{thm:main}]
Set $\lambda = \max \{ A, 4 \log (HT) \}$ and $\alpha = \frac{1}{\max \{ A, H, \log(HT) \}}$.
We have
\begin{align*}
\cJ(\pibeta) - \Exp[\cJ(\pihatbeta)] + \Exp[\cJ(\pihatbeta)] - \Exp[\cJ(\pitil)] \lesssim \frac{H^2 S \log T}{T} + (1 + \alpha H) \cdot \frac{H^2 S \log T}{T} + \alpha \cdot \frac{H^2 S \lambda}{T}
\end{align*}
where we bound $\cJ(\pibeta) - \Exp[\cJ(\pihatbeta)]$ by Theorem 4.4 of \cite{rajaraman2020toward}, and $\Exp[\cJ(\pihatbeta)] - \Exp[\cJ(\pitil)]$ by \Cref{lem:post_subopt} since $\lambda \ge 4\log(HT)$. By our choice of $\alpha = \frac{1}{\max \{ A, H, \log(HT) \}}$, we can bound all of this as
\begin{align*}
\lesssim \frac{H^2 S \log T}{T}.
\end{align*}
This proves the suboptimality guarantee. To show that $\pitil$ achieves demonstrator action coverage, we apply \Cref{lem:post_sampler_guarantee} using our values of $\lambda$ and $\alpha$.
\end{proof}

\subsection{Optimality of Posterior Demonstrator Policy}

Let $\cM$ denote a multi-armed bandit with $A > 1$ actions where $r(a_1) = 1$ and $r(a_i) = 0$ for $i > 1$. Let $\pibetai$ denote the policy defined as
\begin{align*}
    \pibetai(a) = \begin{cases}
        1 - \alpha & a =1 \\
        \alpha & a = i \\
        0 & \text{o.w.}
    \end{cases}
\end{align*}
for $i  > 1$ and $\alpha$ some value we will set, and $\pi^{\beta,1}(1)= 1$.
We let $\cM^i = (\cM, \pibetai)$ the instance-demonstrator pair, $\Exp^i[\cdot]$ the expectation on this instance, $\Prob^{i}$ the distribution on this instance, and $\Prob^{i,T} = \otimes_{t=1}^T \Prob^i$.
\begin{lemma}\label{lem:lb_prob_diff}
    Consider the instance constructed above. Then we have that, for $j \neq i$:
    \begin{align*}
        \Prob^i[\pihat(i) \ge \gamma \cdot \alpha] \le 2 \cdot \Prob^j[\pihat(i) \ge \gamma \cdot \alpha] + T \cdot \alpha.
    \end{align*}
\end{lemma}
\begin{proof}
    This follows from Lemma A.11 of \cite{foster2021statistical}, which immediately gives that:
    \begin{align*}
        \Prob^i[\{ \pihat(i) \ge \gamma \cdot \alpha] \le 2 \cdot \Prob^j[\pihat(i) \ge \gamma \cdot \alpha] + \Dh^2(\Prob^{i,T}, \Prob^{j,T}),
    \end{align*}
    where $\Dh(\cdot, \cdot)$ denotes the Hellinger distance.
    Since the squared Hellinger distance is subadditive we have
    \begin{align*}
        \Dh^2(\Prob^{i,T}, \Prob^{j,T}) \le T \cdot \Dh^2(\Prob^{i}, \Prob^{j}).
    \end{align*}
    By elementary calculations we see that $\Dh^2(\Prob^{i}, \Prob^{j}) = \alpha$, which proves the result.
\end{proof}

\begin{theorem}[Full version of \Cref{thm:main_lb}]
    Let $\pihat$ achieve demonstrator action coverage with some parameter $\gamma$ for each $\cM^i, i \in [A]$, and some $\delta \in (0,1/4]$, and assume that
    \begin{align*}
        \cJ(\pi^{\beta,i}) - \Exp^i[\cJ(\pihat)] \le \xi, \quad \forall i \ge 1
    \end{align*}
    for some $\xi > 0$. Then if $T \le \frac{1}{4\alpha}$, it must be the case that
    \begin{align*}
        \gamma \le \frac{\xi}{2A\alpha}.
    \end{align*}
    In particular, setting $\xi = c \cdot \frac{\log T}{T}$ and if $\alpha = \frac{1}{2T}$, we have
    \begin{align*}
        \gamma \le c \cdot \frac{\log T}{A}.
    \end{align*}
\end{theorem}
\begin{proof}
Our goal is to find the maximum value of $\gamma$ such that our constraint on the optimality of $\pihat$ is met, for each $\cM^i$. In particular, this can be upper bounded as
\begin{align}\label{eq:lb_opt1}
    \max_{\pihat, \gamma} \gamma \quad \text{s.t.} \quad  \Prob^i[\{ \pihat(a) \ge \gamma \cdot \pibeta(a), \forall a \in \cA \}] \ge 1 - \delta, \ \cJ(\pi^{\beta,i}) - \Exp^i [ \cJ(\pihat)] \le \xi, \ \forall i \ge 1.
\end{align}
Note that for $\cM^i, i \ge 1$, the event $\{ \pihat(a) \ge \gamma \cdot \pi^{\beta,i}(a), \forall a \in \cA \}$ is a subset of the event $\{ \pihat(i) \ge \gamma \cdot \alpha \}$. This allows us to bound \eqref{eq:lb_opt1} as
\begin{align}\label{eq:lb_opt2}
    \max_{\pihat, \gamma} \gamma \quad \text{s.t.} \quad  \Prob^i[\pihat(i) \ge \gamma \cdot \alpha] \ge 1 - \delta, \ \cJ(\pi^{\beta,i}) - \Exp^i [ \cJ(\pihat)] \le \xi, \ \forall i \ge 1.
\end{align}
By \Cref{lem:lb_prob_diff}, we have that for each $i > 1$, 
\begin{align*}
    \Prob^i[\pihat(i) \ge \gamma \cdot \alpha] \le 2 \cdot \Prob^1[\pihat(i) \ge \gamma \cdot \alpha] + T \cdot \alpha.
\end{align*}
Furthermore, on $\cM^1$ we have $\cJ(\pi^{\beta,1}) - \Exp^{1}[\cJ(\pihat)] = \Exp^{1}[\sum_{i > 1} \pihat(i)]$.
Given this, we can upper bound \eqref{eq:lb_opt2} as
\begin{align}\label{eq:lb_opt3}
    \max_{\pihat, \gamma} \gamma \quad \text{s.t.} \quad  \Prob^1[\pihat(i) \ge \gamma \cdot \alpha] \ge \frac{1}{2} \cdot (1 - \delta - T \cdot \alpha), \forall i > 1, \ \Exp^{1}[\sum_{i > 1} \pihat(i)] \le \xi.
\end{align}
By Markov's inequality, we have
\begin{align*}
    \Prob^1[\pihat(i) \ge \gamma \cdot \alpha] \le \frac{\Exp^{1}[\pihat(i)]}{\gamma \cdot \alpha}.
\end{align*}
Furthermore, since we have assumed $\delta \le 1/4$ and $T \le \frac{1}{4\alpha}$, we have $\frac{1}{2} \cdot (1 - \delta - T \cdot \alpha) \ge \frac{1}{4}$. 
We can therefore bound \eqref{eq:lb_opt3} as 
\begin{align}\label{eq:lb_opt4}
    \max_{\pihat, \gamma} \gamma \quad \text{s.t.} \quad  \Exp^1[\pihat(i)] \ge \frac{1}{4} \cdot \gamma \alpha, \forall i > 1, \ \Exp^{1}[\sum_{i > 1} \pihat(i)] \le \xi.
\end{align}
However, we see then that we immediately have
\begin{align*}
    \gamma \le \frac{\xi}{4 (A - 1) \alpha}.
\end{align*}
This proves the result as long as $A > 1$.
\end{proof}

\section{Posterior Demonstrator Policy for Gaussian Demonstrator}
Let $P(\cdot \mid \mu)$ denote the distribution $\cN(\mu, \Sigma)$, where we assume $\mu$ is unknown and $\Sigma$ is known. Assume that we have samples $\frakD = \{ x_1, \ldots, x_T \} \sim P(\cdot \mid \must)$.
Let $\Qprior = \cN(0, \Sigprior)$ denote the prior on $\mu$.
Throughout this section we let $=^d$ denote equality in distribution.

\begin{lemma}\label{lem:gauss_posterior}
    Under $\Qprior$, we have that the posterior $\Qpost$ on $\mu$ is:
    \begin{align*}
        \Qpost(\cdot \mid \frakD) = \cN \left (\Sigpost\Sigma^{-1} \cdot \sum_{t=1}^T x_t, \Sigpost \right ),
    \end{align*}
    for $\Sigpost^{-1} = \Sigprior^{-1} + T \cdot \Sigma^{-1}$.
\end{lemma}
\begin{proof}
    Dropping terms that do not depend on $\mu$, we have
    \begin{align*}
        \Qpost(\mu \mid \frakD) & = \frac{P(\frakD \mid \mu) \Qprior(\mu)}{P(\frakD)} \\
        & \propto \exp \left ( - \frac{1}{2} \sum_{t=1}^T (x_t - \mu)^\top \Sigma^{-1} (x_t - \mu) \right ) \cdot \exp \left ( - \frac{1}{2} \mu^\top \Sigprior \mu \right ) \\
        & \propto \exp \left ( -\frac{1}{2} T \mu^\top \Sigma^{-1} \mu - \frac{1}{2} \mu^\top \Qprior^{-1} \mu + \mu^\top \Sigma^{-1} \cdot \sum_{t=1}^T x_t \right ) \\
        & = \exp \left ( -\frac{1}{2} (\mu - \Sigpost \mupost)^\top \Sigpost^{-1} (\mu - \Sigpost \mupost) + \frac{1}{2} \mupost^\top \Sigpost \mupost \right )
    \end{align*}
for $\Sigpost^{-1} = \Sigprior^{-1} + T \cdot \Sigma^{-1}$, and $\mupost = \Sigma^{-1} \cdot \sum_{t=1}^T x_t$. 
\end{proof}

\begin{lemma}[General version of \Cref{prop:policy_post_opt}]\label{lem:opt_post_sample}
    Let
    \begin{align*}
        \muhat = \argmin_{\mu} \sum_{t=1}^T ( \mu - \xtil_t )^\top \Sigma^{-1} (\mu - \xtil_t) + (\mu - \mutil)^\top \Sigprior^{-1} (\mu - \mutil), 
    \end{align*}
    for $\xtil_t = x_t + w_t$, $w_t \sim \cN(0, \Sigma)$, and $\mutil \sim \Qprior$.
    Then $\muhat =^d \Qpost(\cdot \mid \frakD)$.
\end{lemma}
\begin{proof}
    By computing the gradient of the objective, setting it equal to 0, and solving for $\mu$, we see that
    \begin{align*}
        \muhat & = (\Sigprior^{-1} + T \Sigma^{-1})^{-1} \cdot \left ( \Sigma^{-1} \cdot \sum_{t=1}^T \xtil_t + \Sigprior^{-1} \mutil \right ) \\
        & = (\Sigprior^{-1} + T \Sigma^{-1})^{-1} \cdot  \Sigma^{-1} \cdot \sum_{t=1}^T x_t +  (\Sigprior^{-1} + T \Sigma^{-1})^{-1} \cdot \left ( \Sigma^{-1} \cdot \sum_{t=1}^T w_t + \Sigprior^{-1} \mutil \right ).
    \end{align*}
    Note that the first term in the above is deterministic conditioned on $\frakD$, and the second term is mean 0 and has covariance $(\Sigprior^{-1} + T \Sigma^{-1})^{-1}$. We see then that the mean and covariance of $\muhat$ match the mean the covariance of $\Qpost(\cdot \mid \frakD)$ given in \Cref{lem:gauss_posterior}, which proves the result. 
\end{proof}

\section{Additional Experimental Details}\label{sec:app_exp}

For all experiments we instantiate \pbc directly as suggested in \Cref{alg:posterior_variance} and \Cref{alg:posterior_bc}. We describe additional details on this instantiation next.

In all experiments, we parameterize $f_\ell$ in \Cref{alg:posterior_variance} as an MLP (perhaps on top of a ResNet or other feature encoder, as described below). For the \texttt{Robomimic} experiments we let $f_\ell$ parameterize a Gaussian distribution and seek to model the actions in the dataset with a Gaussian, for other settings we simply have $f_\ell$ predict the actions directly (i.e. predicting a deterministic estimate of the actions rather than a distribution).
Note that simply training $f_\ell$ to predict the actions directly, rather than setting $f_\ell$ to a generative model that seeks to model the entire action distribution, is consistent with \Cref{prop:policy_post_opt}---we aim to estimate a \emph{sample} from the posterior distribution, for which it suffices to just fit a deterministic quantity, rather than fitting the entire \emph{distribution} as generative modeling typically aims to do.
Furthermore, fitting a simple predictor on the actions directly usually requires fewer training iterations than fitting, for example, a diffusion model to the entire distribution, so this also reduces the computation required to fit the ensemble.

We found in practice that using bootstrap sampling to generate the datasets $\frakD_\ell$ in \Cref{alg:posterior_variance} performs better than adding noise to the dataset as \Cref{prop:policy_post_opt} suggests. We use both trajectory-level or state-action-level bootstrapping. For trajectory-level bootstrapping we generate $\frakD_\ell$ as in \Cref{alg:traj_bootstrap}.
\begin{algorithm}[h]
\begin{algorithmic}[1]
\State \textbf{input:} demonstration dataset $\frakD$
\State $\frakD_\ell \leftarrow \emptyset$
\For{$t = 1,2,\ldots, $ number of trajectories in $\frakD$}
\State Sample trajectory $\tau \sim \unif(\frakD)$
\State $\frakD_\ell \leftarrow \frakD_\ell \cup \{ \tau \}$
\EndFor 
\State \textbf{return} $\frakD_\ell$
\end{algorithmic}
\caption{Trajectory-Level Bootstrap Sampling}
\label{alg:traj_bootstrap}
\end{algorithm}
For state-action-level bootstrapping we generate $\frakD_\ell$ as in \Cref{alg:sa_bootstrap}.
\begin{algorithm}[h]
\begin{algorithmic}[1]
\State \textbf{input:} demonstration dataset $\frakD$
\State $\frakD_\ell \leftarrow \emptyset$
\For{$t = 1,2,\ldots, |\frakD|$}
\State Sample state-action pair $(s,a) \sim \unif(\frakD)$
\State $\frakD_\ell \leftarrow \frakD_\ell \cup \{ (s,a) \}$
\EndFor 
\State \textbf{return} $\frakD_\ell$
\end{algorithmic}
\caption{Trajectory-Level Bootstrap Sampling}
\label{alg:sa_bootstrap}
\end{algorithm}
In all experiments we parameterize our final policy with a diffusion model. Given this, \Cref{alg:posterior_bc} is trained on the standard diffusion loss.
Further details on each experiment are given below.

To leave room for RL improvement (i.e. to ensure performance is not saturated by the pretrained policy) we limit the number of demos per task in the pretraining dataset, for both the \texttt{Robomimic} and \texttt{Libero} experiments (see below for the precise number of trajectories used in pretraining).

\subsection{Robomimic Experiments}\label{sec:robomimic_exp_details}

We instantiate $\pihat^\theta$ with a diffusion policy that uses an MLP architecture.
For $f_\ell$, we train an MLP to simply predict the action directly in $\frakD_i$ (i.e. we do not use a diffusion model for $f_\ell$), but use the same architecture and dimensions for $f_\ell$ as the diffusion policies. \edit{We used trajectory-level bootstrapped sampling (\Cref{alg:traj_bootstrap}) to compute the ensemble.} \edit{In all cases we pretrain on the Multi-Human \texttt{Robomimic} datasets, and in cases where we use less than the full dataset, we randomly select trajectories from the dataset to train on, using the same trajectories for each approach.}

For each RL finetuning method, we sweep over the same hyperparameters for each pretrained policy method (i.e. BC, \nbc, \pbc), and include results for the best one. For \nbc, we swept over values of $\sigma$ and included results for the best-performing one. 
For all experiments results are averaged over 5 seeds (we pretrain 5 policies for each approach and run RL finetuning on each of them once, for a total of 5 RL finetuning runs per pretraining method, finetuning method, and task).
For each evaluation, we roll out the policy 200 times.
For \dppo we utilize the default hyperparameters as stated in \cite{ren2024diffusion}, and utilize DDPM sampling.  
\edit{For \textsc{ValueDICE}, we use the officially published codebase, and the default hyperparameters provided there.} We found that the \iql training on the data produced by \textsc{ValueDICE} could be somewhat unstable, and so to improve stability, for \texttt{Lift}, added LayerNorm to the \iql critic.
\edit{For \dsrl, we utilize a -1/0 success reward, and otherwise utilize a 0/1 success reward, using Robomimic's built-in success detector to determine the reward.}
We provide hyperparameters for the individual experiments below.

\begin{table}[H]
\caption{
\footnotesize
\textbf{Common \dsrl hyperparameters for all experiments.}
}
\vspace{5pt}
\label{tab:dsrl_online_hyperparams}
\begin{center}
\scalebox{0.9}
{
\begin{tabular}{ll}
    \toprule
    \textbf{Hyperparameter} & \textbf{Value} \\
    \midrule
    Learning rate & $0.0003$ \\
    Batch size & $256$ \\
    Activation & Tanh \\
    Target entropy & $0$ \\
    Target update rate ($\tau$) & $0.005$ \\
    Number of actor and critic layers & $3$ \\
    Number of critics & $2$ \\
    Number of environments & $4$ \\
    \bottomrule
\end{tabular}
}
\end{center}
\end{table}

\begin{table}[H]
\caption{
\footnotesize
\textbf{\dsrl hyperparameters for \texttt{Robomimic} experiments.}
}
\vspace{5pt}
\begin{center}
\scalebox{0.9}
{
\begin{tabular}{lllll}
    \toprule
    \textbf{Hyperparameter} & \texttt{Lift} & \texttt{Can} & \texttt{Square}   \\
    \midrule
Hidden size & $2048$  & $2048$ & $2048$  \\

Gradient steps per update
& \begin{tabular}[t]{@{}l@{}}
$20$ (\pbc, \bc), \\
$10$ (\nbc)
\end{tabular}
& \begin{tabular}[t]{@{}l@{}}
$20$ (\pbc, \bc), \\
$10$ (\nbc)
\end{tabular}
& \begin{tabular}[t]{@{}l@{}}
$20$ (\pbc, \bc), \\
$10$ (\nbc)
\end{tabular}
\\

Noise critic update steps & $20$ & $10$ & $10$  \\
Discount factor & $0.99$ & $0.99$ & $0.999$   \\
Action magnitude  & $1.5$ & $1.5$ & $1.5$   \\
Initial steps & $24000$ & $24000$ & $32000$  \\
    \bottomrule
\end{tabular}
}
\end{center}
\end{table}

\begin{table}[H]
\caption{
\edit{
\footnotesize
\textbf{Hyperparameters for pretrained policies for \texttt{Robomimic} \dsrl experiments.}}
}
\vspace{5pt}
\begin{center}
\scalebox{0.9}
{
\begin{tabular}{lllll}
    \toprule
    \textbf{Hyperparameter} & \texttt{Lift} & \texttt{Can} & \texttt{Square}   \\
    \midrule
    Dataset size (number trajectories) & $5$ & $10$ & $30$ \\
    Action chunk size & $4$ & $4$ & $4$   \\
train denoising steps & $100$ & $100$ & $100$  \\
 inference denoising steps & $8$ & $8$ & $8$ \\
 Hidden size & $512$ & $1024$ & $1024$ \\
 Hidden layers & $3$ & $3$ & $3$ \\
 Training epochs & $3000$ & $3000$ & $3000$ \\
 Ensemble size (\pbc) & $100$ & $100$ & $100$ \\
 Ensemble training epochs (\pbc) & $10000$ & $6000$ & $3000$ \\
 Posterior noise weight $\alpha$ (\pbc) & $1$ & $0.5$ & $1$ \\
 Uniform noise $\sigma$ (\nbc) & $0.1$ & $0.1$ & $0.05$ \\
    \bottomrule
\end{tabular}
}
\end{center}
\end{table}

\begin{table}[H]
\caption{
\footnotesize
\textbf{Best-of-$N$ hyperparameters for \texttt{Robomimic} experiments.}
}
\vspace{5pt}
\begin{center}
\scalebox{0.9}{
\begin{tabular}{l p{2.5cm} p{2.5cm} p{2.5cm}}
    \toprule
    \textbf{Hyperparameter} & \texttt{Lift} & \texttt{Can} & \texttt{Square} \\
    \midrule

Total gradient steps &
$2000000$ &
$2000000$ &
$2000000$ \\

\iql $\tau$ (1000 rollouts) &
0.5 (\bc, \pbc), 0.7 (\nbc), 0.9 (\dice) &
0.7  &
0.7 \\

\iql $\tau$ (2000 rollouts) &
0.5 (\bc), 0.7 (\nbc,\pbc), 0.9 (\dice) &
0.7 &
0.7  \\

Discount factor &
$0.999$ &
$0.999$ &
$0.999$ \\

    \bottomrule
\end{tabular}}
\end{center}
\end{table}

\begin{table}[H]
\caption{
\footnotesize
\textbf{Hyperparameters for pretrained policies for \texttt{Robomimic} Best-of-$N$ experiments.}
}
\vspace{5pt}
\begin{center}
\scalebox{0.9}
{
\begin{tabular}{lllll}
    \toprule
    \textbf{Hyperparameter} & \texttt{Lift} & \texttt{Can} & \texttt{Square}   \\
    \midrule
    Dataset size (number trajectories) & $20$ & $\edit{300}$ & $\edit{300}$ \\
    Action chunk size & $1$ & $\edit{1}$ & $\edit{1}$   \\
Train denoising steps & $100$ & $100$ & $100$  \\
 Hidden size & $512$ & $1024$ & $1024$ \\
 Hidden layers & $3$ & $3$ & $3$ \\
 Training epochs & $3000$ & $3000$ & $3000$ \\
 Ensemble size (\pbc) & 
 \begin{tabular}[t]{@{}l@{}}
 $100$ (1000 rollouts), \\
 $10$ (2000 rollouts)
\end{tabular}
 & $10$ & $10$ \\
 Ensemble training epochs (\pbc) & $3000$ & $500$ & $500$ \\
 Posterior noise weight $\alpha$ (\pbc) &
  \begin{tabular}[t]{@{}l@{}}
 $1$ (1000 rollouts), \\
 $2$ (2000 rollouts)
\end{tabular}
 & $1$ &   \begin{tabular}[t]{@{}l@{}}
 $1$ (1000 rollouts), \\
 $2$ (2000 rollouts)
\end{tabular} \\
 Uniform noise $\sigma$ (\nbc) & $0.1$ & $0.025$ & $0.025$ \\
    \bottomrule
\end{tabular}
}
\end{center}
\end{table}

\begin{table}[H]
\caption{
\edit{
\footnotesize
\textbf{Hyperparameters for pretrained policies for \texttt{Robomimic} \dppo experiments.}}
}
\vspace{5pt}
\begin{center}
\scalebox{0.9}
{
\begin{tabular}{lllll}
    \toprule
    \textbf{Hyperparameter} & \texttt{Lift} & \texttt{Can} & \texttt{Square}   \\
    \midrule
    Dataset size (number trajectories) & $5$ & $10$ & $30$ \\
    Action chunk size & $4$ & $4$ & $4$   \\
train denoising steps & $100$ & $100$ & $100$  \\
 Hidden size & $512$ & $1024$ & $1024$ \\
 Hidden layers & $3$ & $3$ & $3$ \\
 Training epochs & $3000$ & $3000$ & $3000$ \\
 Ensemble size (\pbc) & $100$ & $100$ & $10$ \\
 Ensemble training epochs (\pbc) & $3000$ & $6000$ & $3000$ \\
 Posterior noise weight $\alpha$ (\pbc) & $0.5$ & $0.25$ & $1$ \\
 Uniform noise $\sigma$ (\nbc) & $0.1$ & $0.05$ & $0.05$ \\
    \bottomrule
\end{tabular}
}
\end{center}
\end{table}

\subsection{Libero Experiments}

For Libero, we utilize the transformer architecture from \cite{dasari2024ingredients} for $\pihat^\theta$. 
For \pbc we use state-action bootstrap sampling (\Cref{alg:traj_bootstrap}) to generate $\frakD_\ell$.
For $f_\ell$, we utilize the same ResNet and tokenizer as $\pihat^\theta$, but simply utilize a 3-layer MLP head on top of it---trained to predict the actions directly---rather than a full diffusion transformer. 
For the Best-of-$N$ experiments, \pbc utilizes a diagonal posterior covariance estimate (that is, instead of computing the full covariance matrix as prescribed by \Cref{alg:posterior_variance}, we compute the covariance dimension-wise, and construct a diagonal covariance matrix from this), while for the \dsrl runs it is trained with the full matrix posterior covariance estimate.
We train on \texttt{Libero 90} data from the first 3 scenes of \texttt{Libero 90}---\texttt{KITCHEN SCENE 1, KITCHEN SCENE 2}, and \texttt{KITCHEN SCENE 3}---and use 25 trajectories from each task in each scene. For task conditioning, we conditioning $\pihat^\theta$ on the BERT language embedding \citep{devlin2019bert} of the corresponding text given for that task in the Libero dataset.

For each RL finetuning method, we sweep over the same hyperparameters for each pretrained policy method (i.e. BC, \nbc, \pbc), and include results for the best one. We utilize the \textsc{Dsrl-Sac} variant of \dsrl from \cite{wagenmaker2025steering}. For \nbc, we swept over values of $\sigma$ and included results for the best-performing one.
The \dsrl experiments are averaged over 3 different pretraining runs per method, and one \dsrl run per pretrained run. The Best-of-$N$ experiments are averaged over 2 different pretraining runs per method.
For each evaluation, we roll out the policy 100 times. \edit{In all cases, we utilize a -1/0 success reward, using Libero's built-in success detector to determine the reward.}

We provide hyperparameters for the individual experiments below.

\begin{table}[H]
\caption{
\footnotesize
\textbf{\dsrl hyperparameters for all Libero experiments.}
}
\vspace{5pt}
\label{tab:dsrl_online_hyperparams}
\begin{center}
\scalebox{0.9}
{
\begin{tabular}{ll}
    \toprule
    \textbf{Hyperparameter} & \textbf{Value} \\
    \midrule
    Learning rate & $0.0003$ \\
    Batch size & $256$ \\
    Activation & Tanh \\
    Target entropy & $0$ \\
    Target update rate ($\tau$) & $0.005$ \\
    Number of actor and critic layers & $3$ \\
    Layer size & $1024$ \\
    Number of critics & $2$ \\
    Number of environments & $1$ \\
    Gradient steps per update & $20$ \\
    Discount factor & $0.99$ \\
    Action magnitude & $1.5$ \\
    Initial episode rollouts & $20$ \\
    \bottomrule
\end{tabular}
}
\end{center}
\end{table}

\begin{table}[H]
\caption{
\footnotesize
\textbf{Best-of-$N$ hyperparameters for all Libero experiments.}
}
\vspace{5pt}
\label{tab:bon_online_hyperparams}
\begin{center}
\scalebox{0.9}
{
\begin{tabular}{ll}
    \toprule
    \textbf{Hyperparameter} & \textbf{Value} \\
    \midrule
    \iql learning rate & $0.0003$ \\
    \iql batch size & $256$ \\
    \iql $\beta$ & $3$ \\
    Activation & Tanh \\
    Target update rate  & $0.005$ \\
    $Q$ and $V$ number of layers & $2$ \\
    $Q$ and $V$ layer size & $256$ \\
    Number of critics & $2$ \\
    $N$ (Best-of-$N$ samples) & $32$ \\
    \iql gradient steps & $50000$ \\
    \iql $\tau$ & $0.9$ \\
    Discount factor & $0.99$ \\
    \bottomrule
\end{tabular}
}
\end{center}
\end{table}

\begin{table}[H]
\caption{
\footnotesize
\textbf{Hyperparameters for DiT diffusion policy in Libero experiments.}
}
\label{tab:dit_params}
\vspace{5pt}
\begin{center}
\scalebox{0.9}
{
\begin{tabular}{ll}
    \toprule
    \textbf{Hyperparameter} & Value \\
    \midrule
Batch size &  $150$ \\
Learning rate & $0.0003$ \\
Training steps & $50000$ \\
LR scheduler & cosine \\
Warmup steps & $2000$ \\
\hline
Action chunk size & $4$ \\
Train denoising steps & $100$ \\
Inference denoising steps & $8$ \\
Image encoder & ResNet-18 \\
Hidden size & $256$ \\
Number of Heads & $8$ \\
Number of Layers & $4$ \\
Feedforward dimension & $512$ \\
Token dimension & $256$ \\
\hline
Ensemble size (\pbc) & $5$ \\
Ensemble training steps (\pbc) & $25000$ \\
Ensemble layer size & $512$ \\
Ensemble number of layers & $3$ \\
Posterior noise weight (\pbc) & $2$ (\dsrl run), $4$ (Best-of-$N$ run) \\
Uniform noice $\sigma$ (\nbc) & $0.05$ \\
    \bottomrule
\end{tabular}
}
\end{center}
\end{table}

\subsection{WidowX Experiments}

\begin{figure}[H]
  \centering

 \raisebox{0.5\baselineskip}[0pt][0pt]{
 \begin{minipage}[t]{0.45\textwidth}
    \centering
    \includegraphics[width=0.75\linewidth]{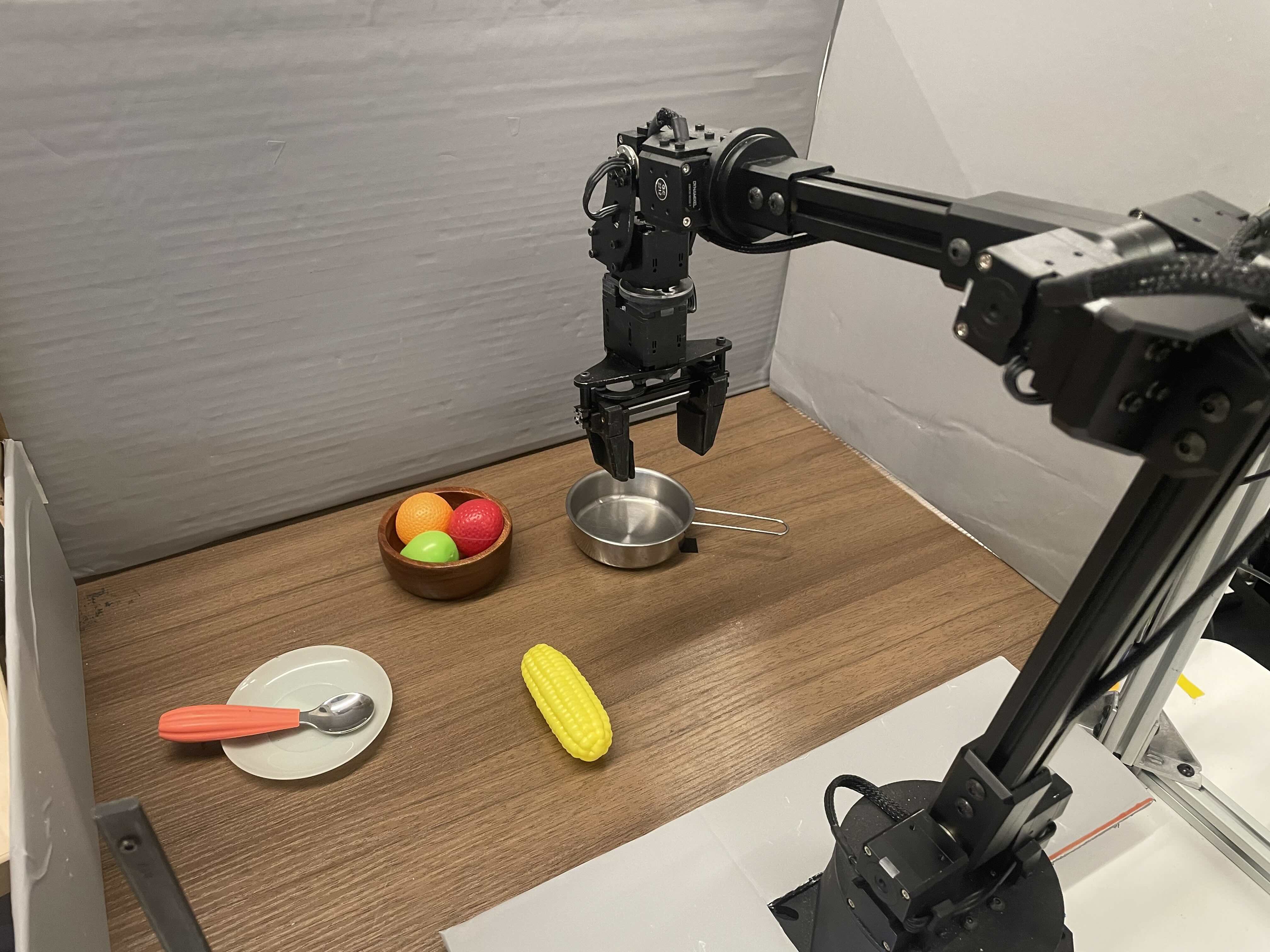}
    \caption{Setup for WidowX ``\texttt{Put corn in pot}'' task.}
  \end{minipage}}
  \hfill
 \begin{minipage}[t]{0.45\textwidth}
    \centering
    \includegraphics[width=0.75\linewidth]{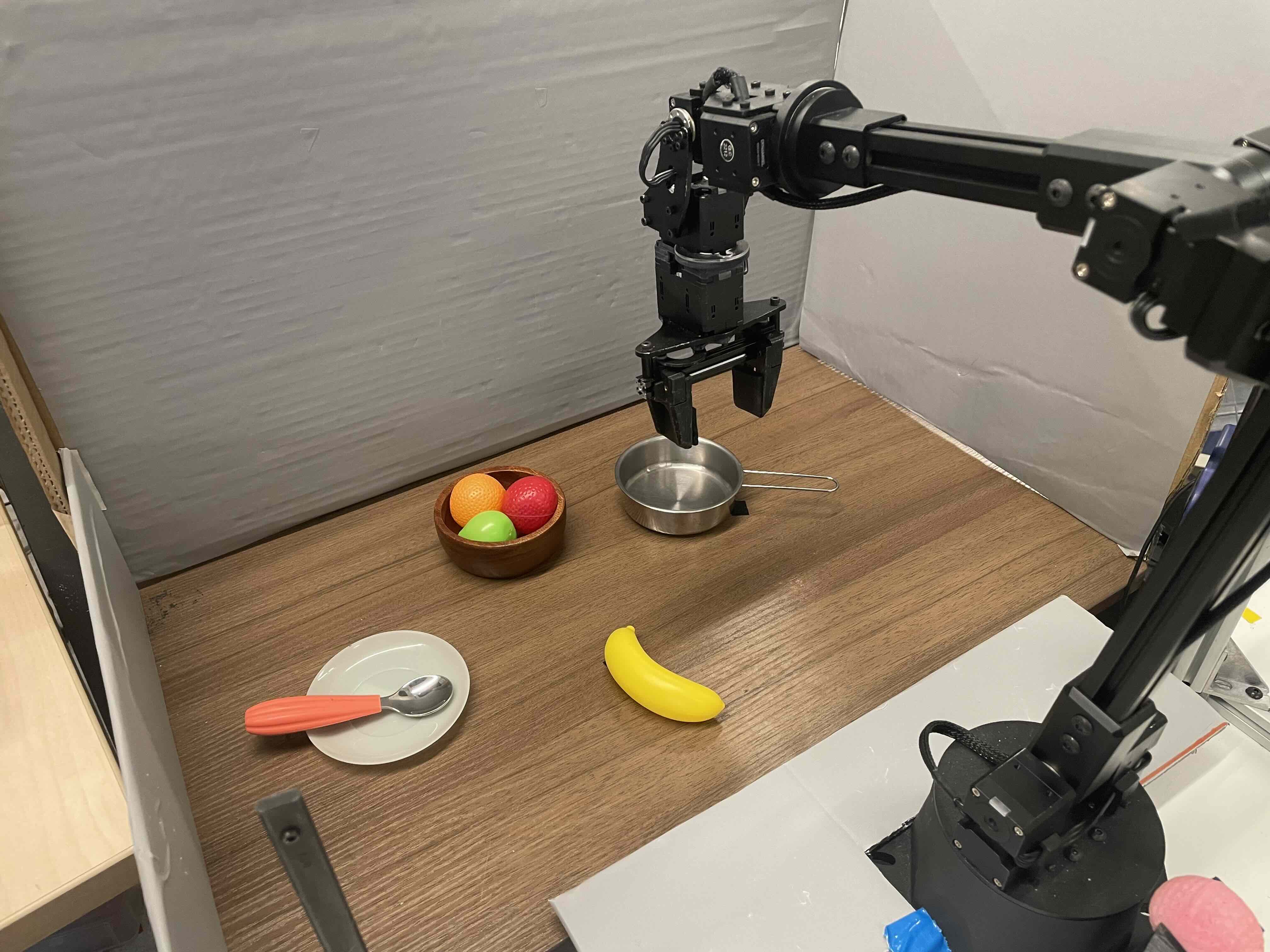}
    \caption{Setup for WidowX ``\texttt{Pick up banana}'' task.}
    \label{fig:widowx_banana}
  \end{minipage}
\end{figure}

For the WidowX experiment, we collect 10 demonstrations on the ``\texttt{Put corn in pot}'' task. For the diffusion policy, we utilize a U-Net architecture with ResNet image encoder. For the \pbc ensemble predictors, we utilize a ResNet image encoder with MLP regression head, trained to directly predict the action in the dataset.
For both \bc and \pbc, we pretrain the policy on the 10 demonstrations, then roll out the pretrained policy 100 times on each task, manually resetting the scene each time and classifying each trajectory as success or failure. We utilize a 0/1 reward (every step is given a reward of 0 unless it succeeds, when it is given a reward of 1). We then train an \iql $Q$-function on the rollout data and, at test time, roll out the pretrained policy, sampling $N$ actions at each step, and choosing the action with the maximum $Q$-value. 
For \iql, we utilize an MLP-based architecture, and to process the images, we utilize image features from a ResNet encoder pretrained on the Bridge v2 dataset \citep{walke2023bridgedata}. For both \bc and \pbc, we try different values of $N$ and different number of \iql training steps, and report the results for the best-performing values for each approach.
All hyperparameters for the diffusion policy are given in \Cref{tab:dp_widowx}, and for \iql in \Cref{tab:bon_widowx}.

\begin{table}[H]
\caption{
\edit{
\footnotesize
\textbf{Hyperparameters for pretrained policies for WidowX experiments.}}
}
\label{tab:dp_widowx}
\vspace{5pt}
\begin{center}
\scalebox{0.9}
{
\begin{tabular}{ll}
    \toprule
    \textbf{Hyperparameter} & Both WidowX tasks    \\
    \midrule
    Action chunk size & $1$   \\
Train denoising steps & $100$   \\
 Inference denoising steps & $16$  \\
 Image encoder & ResNet-18 \\
U-Net channel size & $[256, 512, 1024]$ \\
  U-Net kernel size & $5$ \\
   Training epochs & $800$  \\
 Ensemble predictor hidden size & $512$  \\
 Ensemble predictor hidden layers & $3$  \\
 Ensemble size (\pbc) & $10$  \\
 Ensemble training epochs (\pbc) & $300$  \\
 Posterior noise weight $\alpha$ (\pbc) & $1$  \\
    \bottomrule
\end{tabular}
}
\end{center}
\end{table}

\begin{table}[H]
\caption{
\footnotesize
\textbf{Best-of-$N$ hyperparameters for WidowX experiments.}
}
\vspace{5pt}
\label{tab:bon_widowx}
\begin{center}
\scalebox{0.9}
{
\begin{tabular}{lccccl}
    \toprule
    \textbf{Hyperparameter} & \texttt{Put corn in pot} & \texttt{Pick up banana} \\
    \midrule
    \iql learning rate & $0.0003$ & $0.0003$ \\
    \iql batch size & $256$ & $256$  \\
    \iql $\beta$ & $3$ & $3$ \\
    Activation & Tanh & Tanh \\
    Target update rate  & $0.005$ & $0.005$ \\
    $Q$ and $V$ number of layers & $2$ & $2$ \\
    $Q$ and $V$ layer size & $256$ & $256$ \\
    Number of critics & $2$ & $2$  \\
    $N$ (Best-of-$N$ samples) & $4$ & $16$ \\
    \iql gradient steps & $400000$ (\bc), $700000$ (\pbc) &  $100000$ \\
    \iql $\tau$ & $0.7$ &  $0.7$ \\
    Discount factor & $0.97$  & $0.97$ \\
    \bottomrule
\end{tabular}
}
\end{center}
\end{table}

\subsection{Additional Ablations}\label{sec:additional_ablations}
For all ablation experiments, other than the hyperparameter we vary, we utilize the hyperparameters given in \Cref{sec:robomimic_exp_details}.
In \Cref{fig:traj_size_ablation} we provide an additional ablation on the dataset size for \texttt{Robomimic Square}, and in \Cref{fig:libero_additional_qualititative} provide additional qualitative results on \texttt{Libero}.

\begin{figure}[H]
  \centering

 \begin{minipage}[t]{0.95\textwidth}
    \centering
    \includegraphics[width=.33\linewidth]{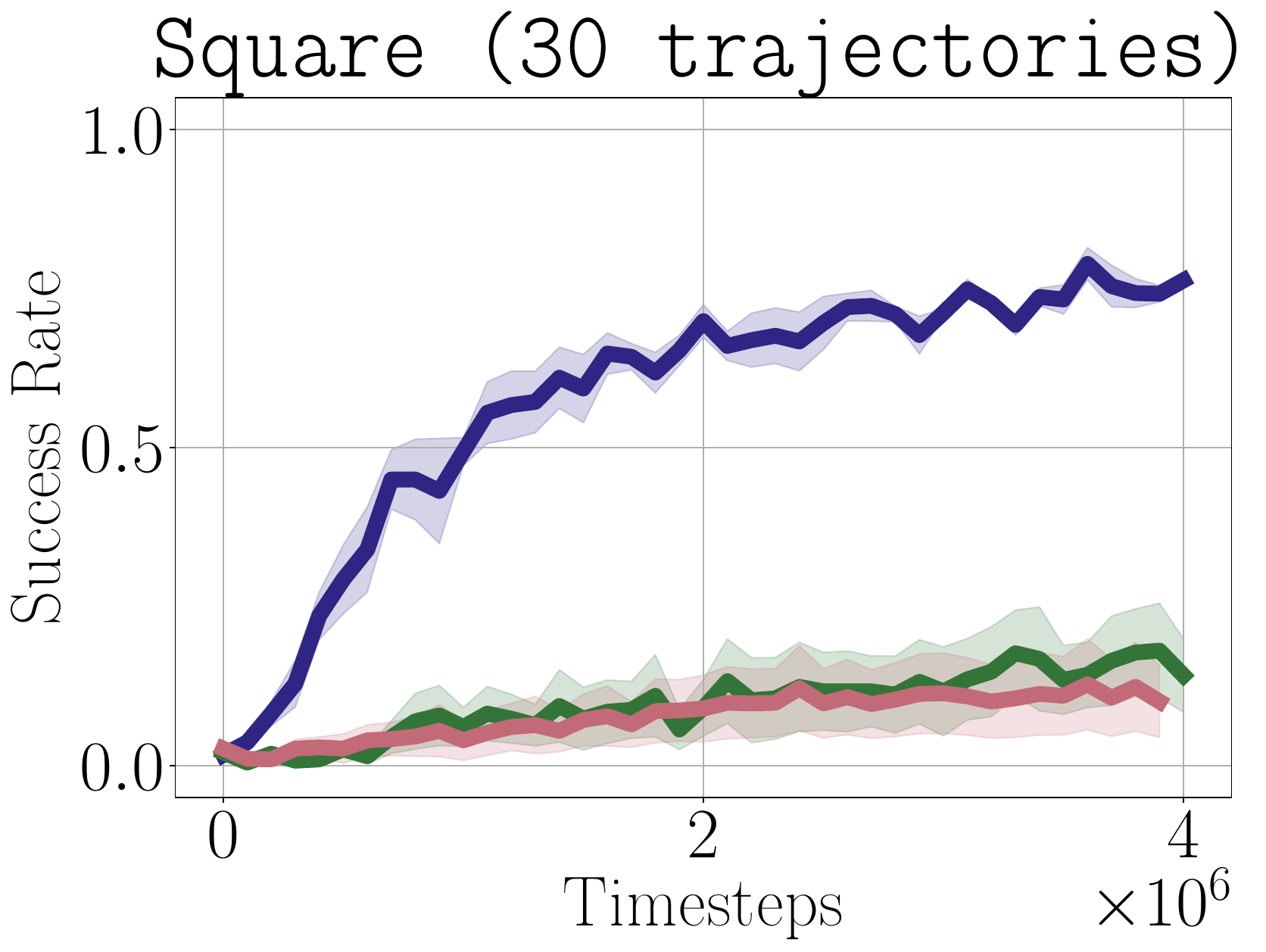}\hfill
    \includegraphics[width=.33\linewidth]{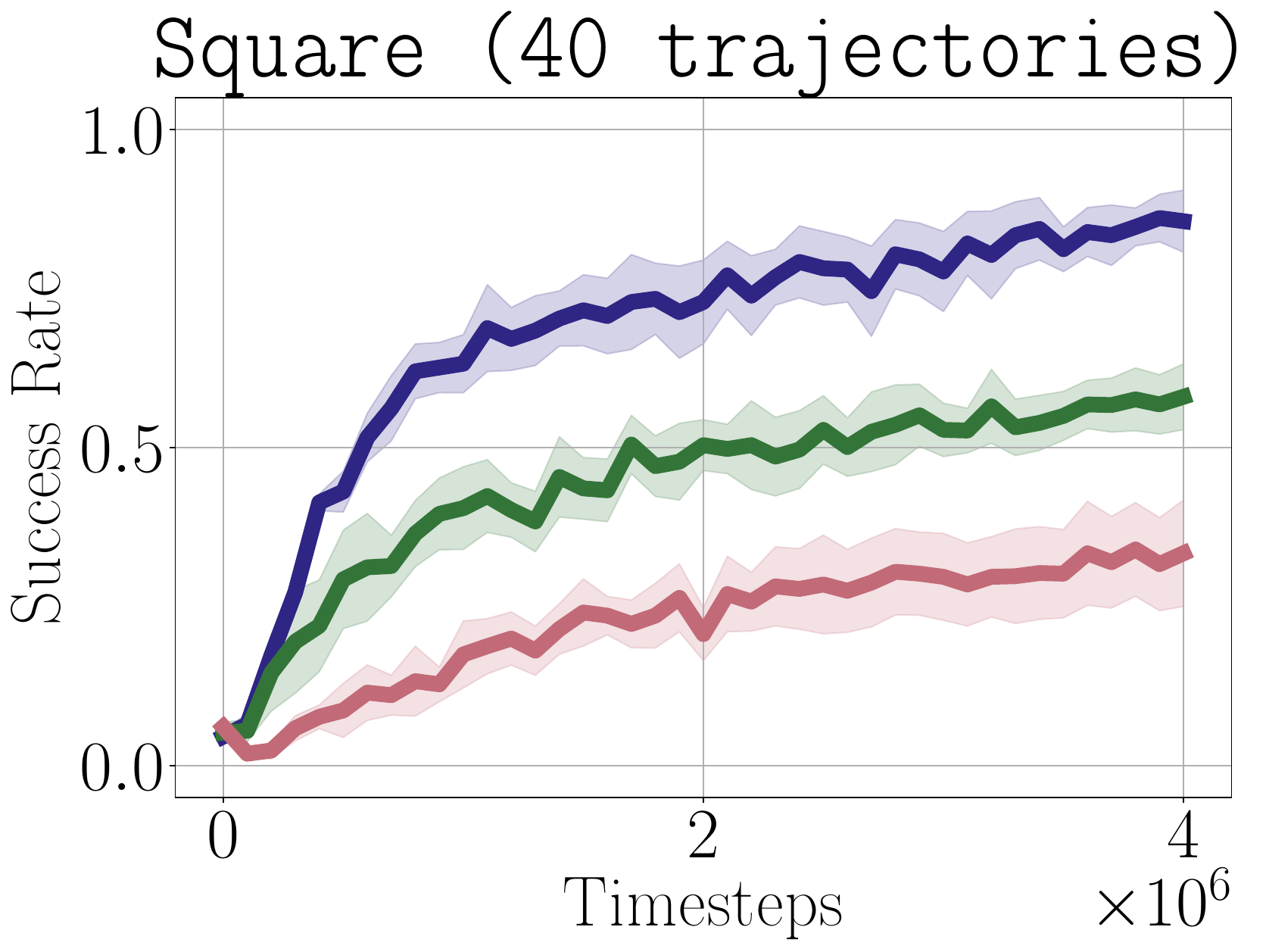}\hfill
    \includegraphics[width=.33\linewidth]{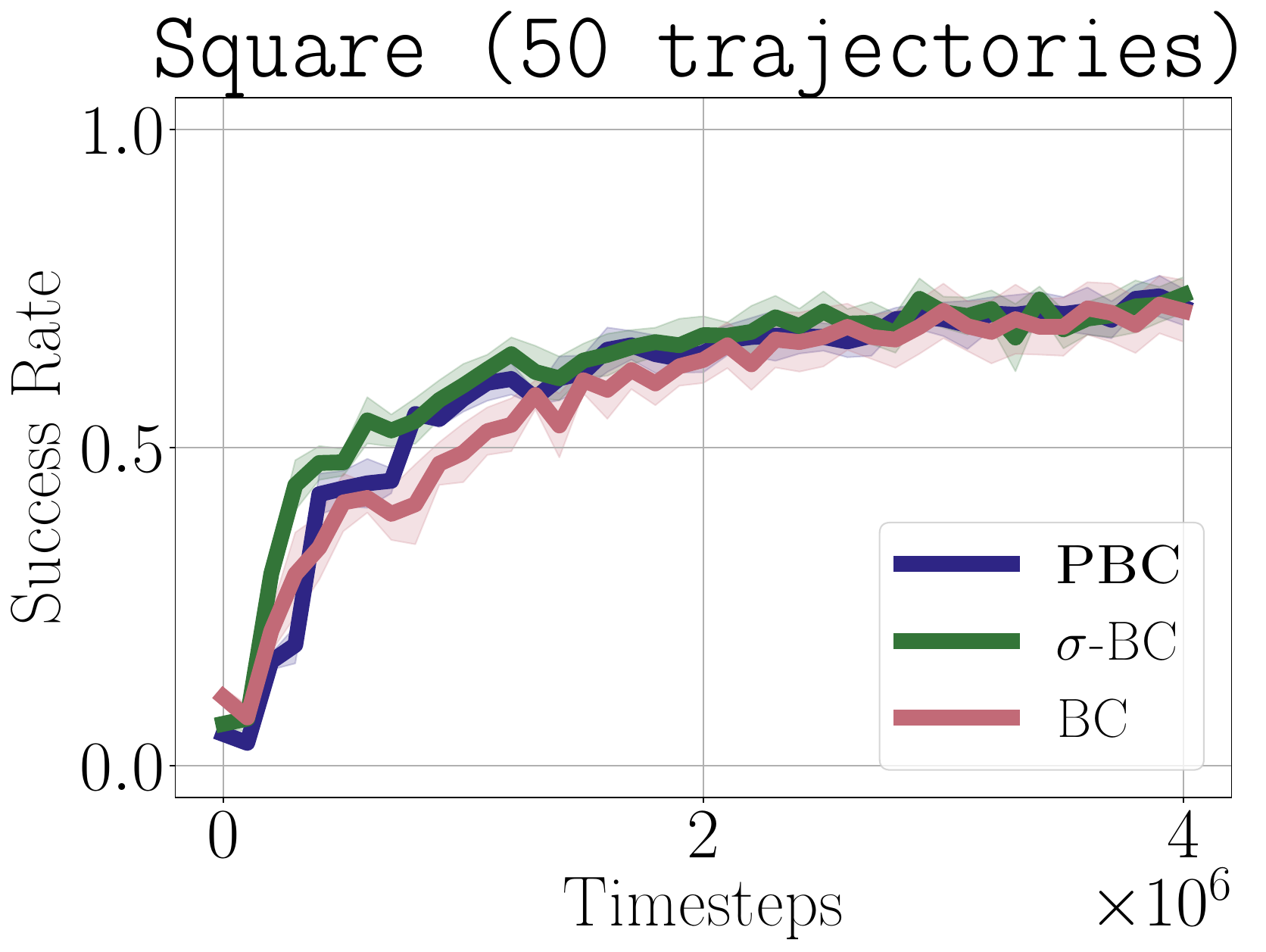}
    \vspace{-0.5em}
    \caption{\edit{Comparison of \dsrl finetuning performance combined with different BC pretraining approaches on \texttt{Robomimic Square}, varying the number of trajectories in the dataset the policies are pretrained on. As can be seen, the finetuning performance of policies pretrained with \pbc is largely unaffected by the size of the pretraining dataset, while BC and \nbc are both very sensitive to dataset size. For large enough datasets (50 trajectories), BC and \nbc perform as well as \pbc. This is to be expected---if we train on enough data, our uncertainty will be low, so \pbc will essentially reduce to BC. These results illustrate that \pbc gracefully interpolates between settings where BC overfits to small amounts of data, hurting its finetuning performance, and settings where BC is sufficient for effective finetuning.}}
      \label{fig:traj_size_ablation}
  \end{minipage}
\end{figure}

\begin{figure}[H]
\includegraphics[width=0.99\textwidth]{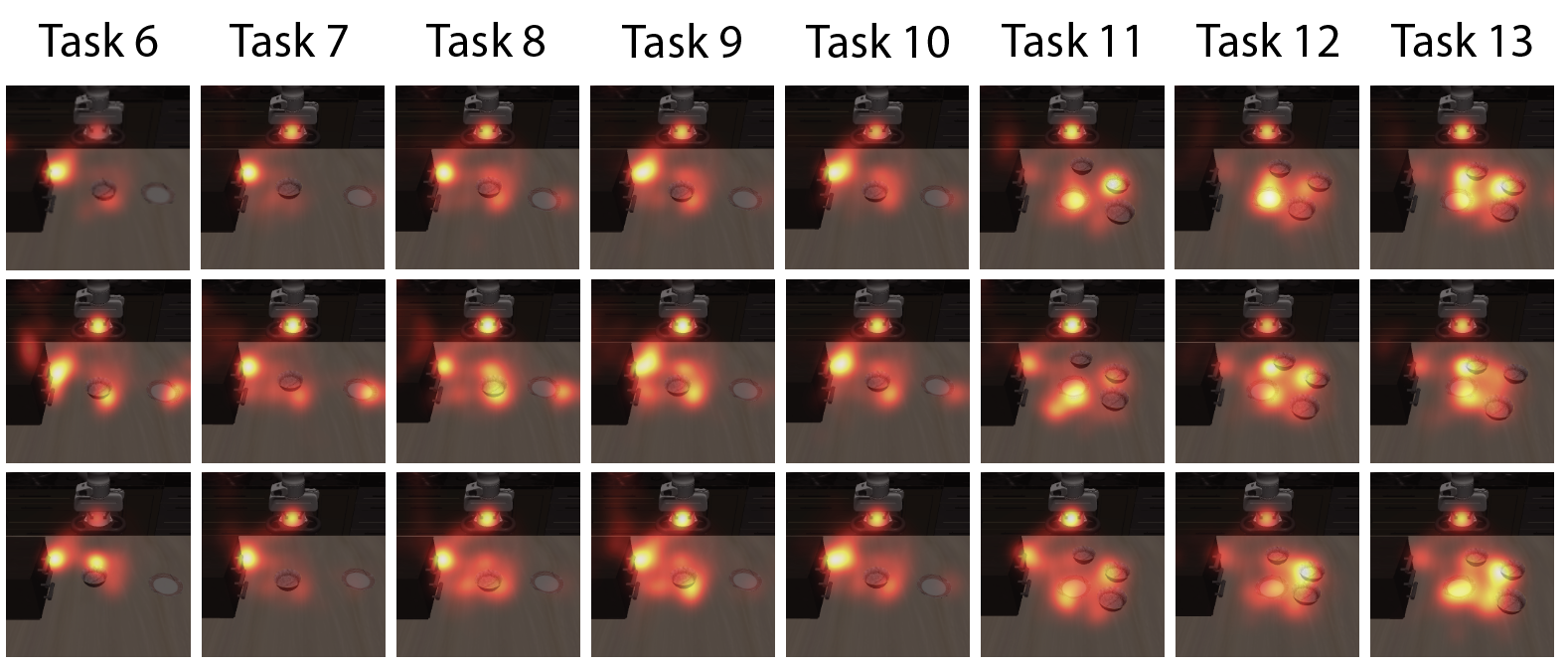}
\includegraphics[width=0.99\textwidth]{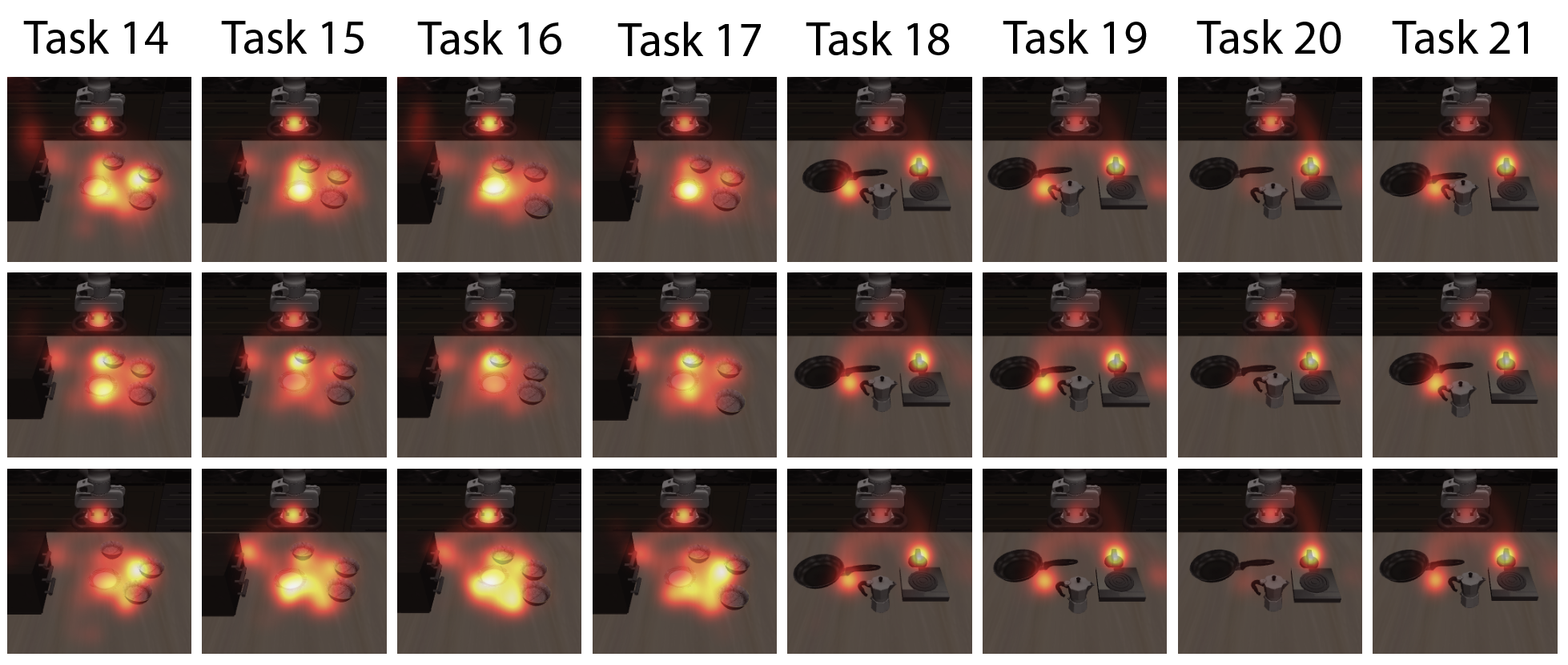}
\caption{Additional density heatmaps of pretrained policies on tasks 6-21 from \texttt{Libero 90}. See \Cref{tab:libero_tasks} for task commands.
}
\label{fig:libero_additional_qualititative}
\end{figure}

\begin{table}[H]
\caption{
\edit{
\footnotesize
\textbf{Task descriptions for Libero tasks in \texttt{Kitchen Scene 1-3}.}}
}
\label{tab:libero_tasks}
\vspace{5pt}
\begin{center}
\scalebox{0.9}
{
\begin{tabular}{ll}
    \toprule
    \textbf{Task ID} & Task description    \\
    \midrule
    Task 6 & \texttt{Open the bottom drawer of the cabinet} \\
    Task 7 & \texttt{Open the top drawer of the cabinet} \\                              
Task 8 & \texttt{Open the top drawer of the cabinet and put the bowl in it} \\         
Task 9 & \texttt{Put the black bowl on the plate} \\                                                           
Task 10 & \texttt{Put the black bowl on top of the cabinet} \\                        
Task 11 & \texttt{Open the top drawer of the cabinet} \\                             
Task 12 & \texttt{Put the black bowl at the back on the plate} \\                     
Task 13 & \texttt{Put the black bowl at the front on the plate} \\                     
Task 14 & \texttt{Put the middle black bowl on the plate} \\                 
Task 15 & \texttt{Put the middle black bowl on top of the cabinet} \\       
Task 16 & \texttt{Stack the black bowl at the front on the black bowl in the middle} \\
Task 17 & \texttt{Stack the middle black bowl on the back black bowl} \\         
Task 18 & \texttt{Put the frying pan on the stove} \\                              
Task 19 & \texttt{Put the moka pot on the stove} \\                                   
Task 20 & \texttt{Turn on the stove} \\                                                                          
Task 21 & \texttt{Turn on the stove and put the frying pan on it} \\
    \bottomrule
\end{tabular}
}
\end{center}
\end{table}

\end{document}